\documentclass{article}

\usepackage[preprint]{neurips_2026}

\usepackage[utf8]{inputenc}
\usepackage[T1]{fontenc}
\usepackage{verbatim}
\usepackage{hyperref}
\usepackage{url}
\usepackage{nicefrac}
\usepackage{bbm}
\usepackage{algorithm}
\usepackage{algpseudocode}
\usepackage{yck}
\usepackage{titletoc}

\graphicspath{{fig/}}

\newcommand{\PP}{\mathbb{P}}

\newcommand{\DT}{\mathcal{D}_{\mathrm{Trust}}}
\newcommand{\TT}{\mathcal{T}_{\mathrm{trans}}}
\newcommand{\lenc}{\mathcal{L}_{\mathrm{enc}}}

\usepackage{pifont}
\newcommand{\cmark}{\ding{51}}
\newcommand{\xmark}{\ding{55}}

\usepackage{xcolor}
\usepackage{xspace}

\definecolor{methodorange}{RGB}{180,85,20}

\definecolor{highlightpurplebg}{RGB}{238,230,250}
\definecolor{highlightgreenbg}{RGB}{226,245,232}

\newcommand{\PUAuditPlus}{%
  \textcolor{methodorange}{\textsc{AURA}}\xspace
}

\newcommand{\hlgreen}[1]{%
  \begingroup
  \setlength{\fboxsep}{1.2pt}%
  \colorbox{highlightgreenbg}{#1}%
  \endgroup
}

\newcommand{\gcmark}{\textcolor{green!60!black}{\cmark}}
\newcommand{\rxmark}{\textcolor{red!70!black}{\xmark}}

\newtheorem{theorem}{Theorem}
\newtheorem{lemma}{Lemma}

\newtheorem{remark}{Remark}
\newtheorem{assumption}{Assumption}

\title{AURA: Adaptive Uncertainty-aware Refinement for LLM-as-a-Judge Auditing}

\author{%
  Zilong Zhang\\
  Department of Mathematics and Statistics \\
  Georgia State University\\
  \And
  Yi-Ting Hung \\
  Department of Mathematics and Statistics \\
  Georgia State University\\
  \And
  Weiyi He \\
  Department of Probability and Statistics \\
  Department of Computer Science and Engineering\\
  Michigan State University\\
  \AND
  Junxi Zhang \\
  Department of Mathematics and Statistics \\
  Concordia University
  \And
  Lei Ding \\
  Department of Statistics \\
  University of Manitoba \\
  \texttt{lei.ding@umanitoba.ca}
  \And
  Chi-Kuang Yeh \\
  Department of Mathematics and Statistics \\
  Georgia State University \\
  \texttt{cyeh@gsu.edu}
}

\begin{document}

\maketitle

\begin{abstract}
Large language models (LLMs) are increasingly used as judges for open-ended generation, as large-scale human evaluation is often expensive and difficult to scale, yet their preferences remain imperfect proxies for human judgment. Existing auditing pipelines often assume that a reliable subset of examples or clean supervision signals are available beforehand, for example from human annotation, heuristic filtering, or the outputs of strong judges. In LLM evaluation, this assumption is fragile: the initial split may inherit judge bias, while human verification is typically too scarce to define stable groups at scale. We propose \PUAuditPlus, an adaptive uncertainty--aware refinement framework for auditing pairwise LLM--as--a--judge decisions under selected human verification. \PUAuditPlus iteratively learns a human-consistency signal, propagates reliable evidence, and prioritizes uncertain comparisons for human review. The key idea is to treat trust in a judge as a latent quantity that is progressively refined as evidence accumulates. We provide a compact formulation, a stable refinement procedure, and a comprehensive evaluation on both synthetic and real pairwise LLM-answer data.
\end{abstract}

\section{Introduction}
Evaluating language-model answers is no longer only a question of whether a benchmark has the right prompts. It is increasingly about whether the evaluator can reliably distinguish between multiple plausible answers that differ in factual accuracy, completeness, reasoning quality, safety, or usefulness~\citep{liu2023geval,kim2023prometheus}. Human evaluation remains the most direct measure of these differences, but it is expensive, slow, and difficult to scale as models, prompts, and evaluation criteria evolve. As a result, this has made \emph{LLM-as-a-judge} evaluation a practical component of modern model development, where a strong model grades or compares outputs from other models \citep{zheng2023judging,dubois2023alpacafarm}.

The practical value of LLM judges comes with a statistical problem. A judge's preference is not the same object as a human preference. LLM judges can be sensitive to response order, verbosity, formatting, self-preference, and other superficial cues \citep{zheng2023judging,wang2024fair,dubois2024length, zeng2023llmbar,thakur2024judging}. These issues are especially consequential in pairwise evaluation, where a single incorrect comparison can reverse the winner and loser of an example. The central challenge is therefore not merely to imitate judge output, but to audit when a judge's decision is likely to agree with human judgment and when it should be corrected or verified.

A natural starting point is to treat judge auditing as a weakly supervised learning problem. In this view, a small set of human verified comparisons provides trustworthy evidence, while the remaining examples form a large uncertainty pool. Prior positive and unlabeled learning methods offer useful tools for such settings~\citep{elkan2008learning,duplessis2014analysis,kiryo2017nonnegative,bekker2020survey}. In a PU view, examples for which the LLM judge agrees with human preference form a positive group, while the remaining unverified examples form an unlabelled mixture. Standard PU methods, however, usually assume that the positive examples are already known and that the unlabelled pool is fixed. This assumption is often too strong for LLM evaluation. The initial groups may be derived from the same noisy judge under audit, from weak heuristics, or from a small and biased set of human labels. If these groups are treated as fixed, the auditing procedure may simply preserve and amplify the original bias.

This paper asks a different question: can the method learn which judge decisions are human--consistent while simultaneously identifying where uncertainty remains and which examples warrant additional verification? To address this problem, we propose \PUAuditPlus, an adaptive uncertainty--aware refinement framework for pairwise LLM--answer auditing under limited human supervision. For a small subset, human labels indicate whether LLM preferences are human-consistent, while the remaining examples form an uncertain pool. \PUAuditPlus maintains a soft responsibility for each example, representing the probability that the LLM preference is correct, and iteratively refines these estimates by training a human-consistency scorer, incorporating local and anchor-based evidence, and propagating reliable signals through a sparse transport-based refinement. The key idea is to treat reliable and uncertain groups as latent quantities that are progressively updated rather than fixed in advance, distinguishing our approach from standard PU methods and judge-only calibration. The transport component propagates reliable evidence from anchor examples to nearby uncertain comparisons in the learned geometry, while the verification step focuses human effort on informative and uncertain cases~\citep{zhu2002labelprop,zhou2003localglobal,chapel2020partial}. Together, these components form a unified refinement loop that improves the estimation of judge--human agreement under limited supervision.

Our contributions are as follows
\begin{itemize}
    \item We formulate LLM--as--a--judge auditing as adaptive human--consistency refinement under limited human verification, where trustworthy and uncertain groups are progressively updated, and informative comparisons are selected for human verification.
    \item We organize the method around a compact set of interpretable state variables, including soft human--consistency responsibility, anchor confidence, and carried evidence inflow.
    \item We provide a stability analysis for the alternating update procedure and an evaluation protocol covering synthetic recovery, real pairwise LLM evaluation, robustness to noisy initialization, and annotation efficiency.
\end{itemize}

\section{Background and related work}
\label{sec:related}

LLM-as-a-judge evaluation has become a common approach for assessing open-ended generations, especially when task-specific reference answers are unavailable. MT-Bench and Chatbot Arena popularized pairwise LLM and human preference evaluation as a scalable framework for comparing chat assistants \citep{zheng2023judging,chiang2024chatbot}. Subsequent work has proposed stronger evaluator models, prompting schemes, and benchmark protocols \citep{liu2023geval,wang2023pandalm,kim2023prometheus,zhu2023judgelm,dubois2024length}. At the same time, LLM judges can exhibit position bias, verbosity bias, prompt sensitivity, score instability, and weaker alignment with humans on difficult examples \citep{wang2024fair,zeng2023llmbar,thakur2024judging,tan2024judgebench,park2024offsetbias}. Rather than introducing another standalone judge, \PUAuditPlus asks when the output of an existing judge should be trusted.

Our work is also related to preference modeling and reward-model evaluation. Pairwise comparison has a long statistical history, including the Bradley--Terry model \citep{bradley1952rank}, and modern alignment pipelines use preference data to train reward models or optimize policies through RLHF and DPO \citep{christiano2017deep,ouyang2022training,bai2022training,rafailov2023dpo}. RewardBench evaluates reward models using prompt--chosen--rejected triples across chat, reasoning, and safety settings \citep{lambert2025rewardbench}. These works primarily study how to train or benchmark models using preference data. \PUAuditPlus instead studies a complementary auditing problem: given an LLM judge's pairwise preference, estimate whether it agrees with human judgment and decide whether to trust, correct, or verify it.

Methodologically, \PUAuditPlus connects to PU learning, weak supervision, graph-based semi-supervised learning, transport, and active verification. Standard PU learning assumes labeled positives and an unlabeled mixture \citep{elkan2008learning,bekker2020survey}, whereas our target label is LLM--human agreement rather than answer quality, and human verification can provide both positive and negative anchors. We therefore treat the problem as adaptive PU learning, where the positive and unlabeled groups are refined as evidence accumulates. The propagation step is related to graph-based label smoothing and transport methods \citep{zhu2002labelprop,zhou2003localglobal,cuturi2013sinkhorn,peyre2019computational,chapel2020partial}, but \PUAuditPlus uses a conservative transport-based operator for sparse, budgeted evidence propagation rather than full distribution alignment. The verification policy is related to active learning, selective prediction, and calibration \citep{settles2009active,geifman2017selective,guo2017calibration}, since it selects examples for human review when the current refinement state is uncertain or influential.

\begin{figure}[t]
\centering
\includegraphics[width=0.7\linewidth,angle=-90]{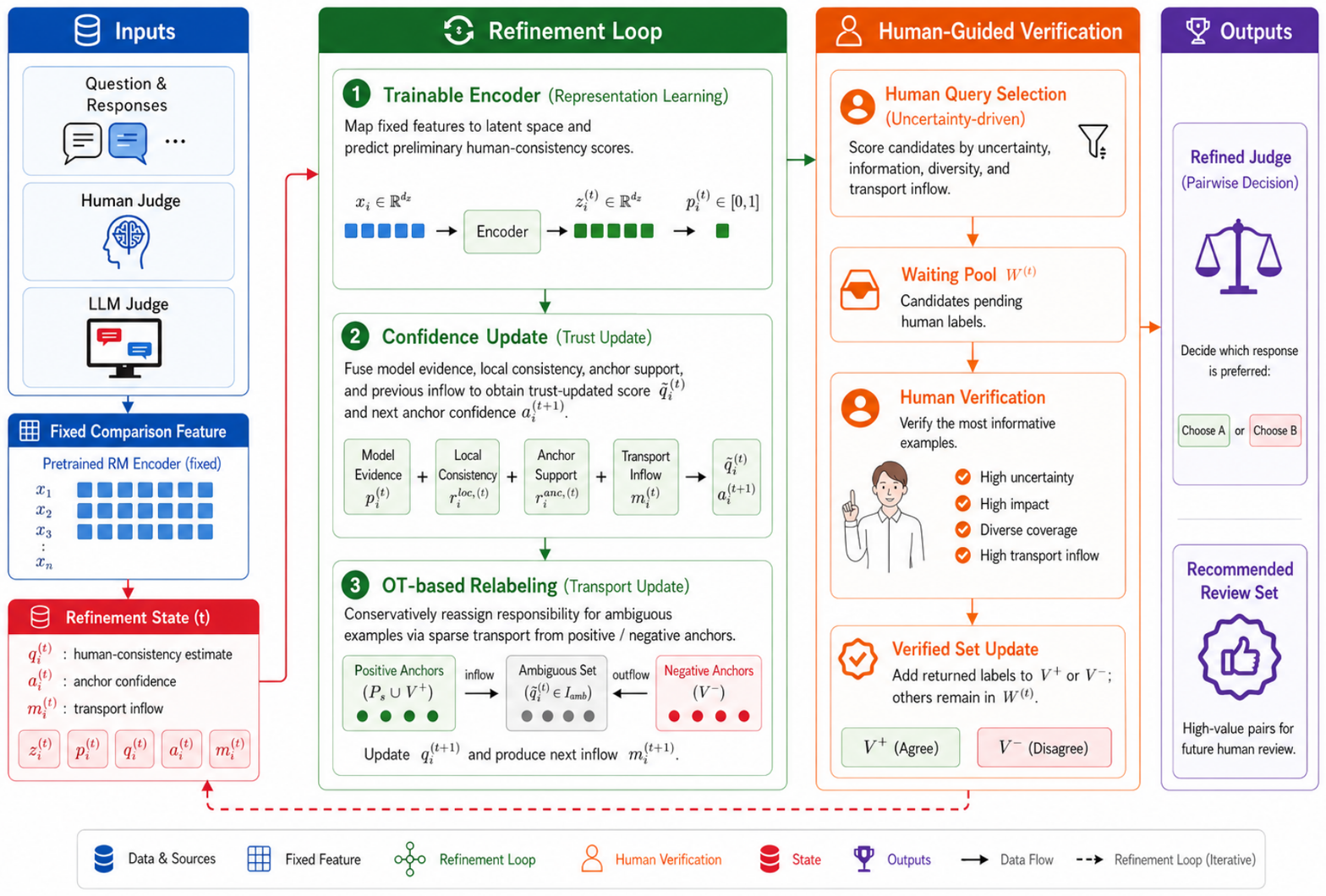}
\caption{Overview of \PUAuditPlus framework. At iteration \(t\), the current state \((q^{(t)},a^{(t)},m^{(t)})\) is used to train the encoder, which produces the latent representations \(z^{(t)}\) and preliminary scores \(p^{(t)}\). The trust-update module uses them to compute the consistency estimate \(\tilde q^{(t)}\) and the next anchor confidence \(a^{(t+1)}\). The transport step then refines \(\tilde q^{(t)}\) into the next-round consistency estimate \(q^{(t+1)}\) and produces the next inflow state \(m^{(t+1)}\). Finally, the verification step selects examples for human review; returned labels are added to \(\mathcal{V}^{+}\) or \(\mathcal{V}^{-}\), while pending examples are kept in \(\mathcal{W}^{(t+1)}\).} 
\label{fig:workflow}
\end{figure}

\section{Problem formulation and notation}
\label{sec:problem_formulation}

Let $\D=\{(Q_i,R_i^1,R_i^2,J_i^{\mathrm{H}},J_i^{\mathrm{L}})\}_{i=1}^n$ be a collection of samples, where \(Q_i\) is the question or prompt, \(R_i^1\) and \(R_i^2\) are two candidate responses, \(J_i^{\mathrm{L}}\) is the initial LLM judge, and \(J_i^{\mathrm{H}}\) is the human judge. An example of the data structure is shown in Figure~\ref{fig:data_example} in Appendix~\ref{app:notation}. The LLM judge assigns a strict preference \(J_i^{\mathrm{L}}\in\{1,2\}\), where \(J_i^{\mathrm{L}}=k\) indicates that \(R_i^k\) is preferred. Human labels, when available, are denoted by \(J_i^{\mathrm{H}}\in\{1,2\}\). We exclude tie cases. 
For each comparison, the LLM judge induces an initial winner--loser assignment:
\[
R_i^{w}=R_i^{J_i^{\mathrm{L}}},
\qquad
R_i^{\ell}=R_i^{3-J_i^{\mathrm{L}}}.
\]
Here, \(R_i^{w}\) and \(R_i^{\ell}\) are interpreted as the LLM-induced preliminary winner and loser, rather than ground-truth human preferences; they will be refined using the available human labels.
For each example, we define the human-consistency label
\(
y_i
=
\mathbf{1}\{J_i^{\mathrm{L}}=J_i^{\mathrm{H}}\},
\)
which indicates whether the LLM preference agrees with the human preference. This label is observed only for human-verified examples and remains latent for unverified examples. Thus, \(y_i=1\) indicates that the LLM preference agrees with the human preference, while \(y_i=0\) indicates that the LLM preference should be corrected. The corresponding agreement-positive and agreement-negative sets are
\[
\mathcal{P}^{\star}=\{i:y_i=1\},
\qquad
\mathcal{N}^{\star}=\{i:y_i=0\}.
\]
Our goal is to estimate
\(\PP(y_i=1 \mid Q_i,R_i^1,R_i^2,J_i^{\mathrm L})\)
for unverified examples, and use this estimate to refine the LLM-induced preference signal under a limited human-review budget.

To describe this refinement process, we organize examples according to their verification status. Let \(\mathcal{P}_s\) denote a trusted positive seed set available before refinement. In our experiments, \(\mathcal{P}_s\) is initialized from a small human-verified positive seed set and used only as an initial anchor set, rather than as the fixed clean positive group assumed by standard PU methods. Let \(\mathcal{U}\) denote the refinement pool, where the algorithm performs progressive label refinement. \(\mathcal{V}^{+}\) and \(\mathcal{V}^{-}\) denote verified examples that are human-consistent and human-inconsistent, respectively. Unlike the fixed seed anchor set \(\mathcal{P}_s\), \(\mathcal{V}^{+}\) and \(\mathcal{V}^{-}\) record verified labels incorporated during refinement and may grow over iterations. These verified sets provide hard constraints during refinement. We also maintain a waiting pool \(\mathcal{W}^{(t)}\), which contains examples that have been selected for human review but whose human labels have not yet been incorporated. Throughout this paper, $\|\cdot\|$ denotes the Euclidean norm unless otherwise specified. Detailed implementation constants are listed in Appendix~\ref{app:notation}.

\section{Methodology}
\label{sec:method}

We now present \PUAuditPlus, a progressive refinement framework for converting noisy LLM-judge preferences into preference signals that align with human judgments. Starting from preliminary preferences produced by an LLM judge and a small set of human-verified examples, \PUAuditPlus iteratively learns a task-specific representation, updates human-consistency estimates, propagates trust through a conservative transport step, and selectively queries additional human labels. Figure~\ref{fig:workflow} illustrates the overall workflow, and Algorithm~\ref{alg:progressive_pu_refinement} in Appendix summarizes the full procedure.

\subsection{Progressive representation and trust refinement}
\label{sec:progressive_refinement}

For each pairwise comparison, we first construct a fixed comparison feature
\[
x_i
=
E_{\mathrm{RM}}(Q_i,R_i^w)
-
E_{\mathrm{RM}}(Q_i,R_i^\ell)
\in\mathbb{R}^{d_x},
\]
where \(E_{\mathrm{RM}}\) is a pretrained reward-model encoder used only for feature extraction and kept frozen throughout refinement. At iteration \(t\), a trainable encoder maps \(x_i\) into a task-specific latent representation \(z_i^{(t)}\in\mathbb{R}^{d_z}\), and a prediction head produces a preliminary human-consistency estimate $p_i^{(t)}=\sigma(\ell_i^{(t)})$.
This score is not treated as the final label; instead, it provides model evidence for the subsequent trust and transport updates.

The encoder is trained from the current refinement state \((q^{(t)},a^{(t)},m^{(t)})\). The goal is to fit the available human-verified labels, use reliable unverified examples as additional soft supervision, and keep the learned latent space stable across refinement rounds. We use the objective
\[
\mathcal{L}_{\mathrm{enc}}^{(t)}
=
\mathcal{L}_{\mathrm{ver}}^{(t)}
+
\lambda_{\mathrm{soft}}\mathcal{L}_{\mathrm{soft}}^{(t)}
+
\lambda_{\mathrm{geo}}\mathcal{R}_{\mathrm{geo}}^{(t)}
+
\lambda_{\mathrm{anchor}}\mathcal{R}_{\mathrm{anchor}}^{(t)}.
\]
Here \(\mathcal{L}_{\mathrm{ver}}^{(t)}\) is the supervised loss on human-verified examples, and \(\mathcal{L}_{\mathrm{soft}}^{(t)}\) trains on high-confidence unverified examples using their current estimation \(q_i^{(t)}\) as soft labels. The two regularization terms stabilize the representation: \(\mathcal{R}_{\mathrm{geo}}^{(t)}\) enforces local smoothness, while \(\mathcal{R}_{\mathrm{anchor}}^{(t)}\) discourages high-confidence examples from drifting across rounds. Full network parameterization and loss definitions are provided in Appendix~\ref{app:encoder_details}.

\subsection{Trust update, conservative transport, and verification}
\label{sec:trust_transport_verification}

After encoder training, \PUAuditPlus refines the preliminary score \(p_i^{(t)}\) using model, local, anchor, and transport evidence. We combine these signals through a trust logit
\[
u_i^{(t)}
=
\lambda_p
\log\frac{p_i^{(t)}}{1-p_i^{(t)}}
+
\lambda_{\mathrm{loc}}
\left(2r_i^{\mathrm{loc},(t)}-1\right)
+
\lambda_{\mathrm{anc}}
\left(2r_i^{\mathrm{anc},(t)}-1\right)
+
\lambda_m
\left(2m_i^{(t)}-1\right)
+
\beta_0,
\]
where \(r_i^{\mathrm{loc},(t)}\) measures neighborhood-level agreement in the learned latent space, \(r_i^{\mathrm{anc},(t)}\) measures support from trusted anchors, and \(m_i^{(t)}\) carries positive inflow evidence from the previous round. The trust logit is transformed into a trust-updated estimate \(\tilde q_i^{(t)}\). The same evidence is also used to update the next anchor confidence \(a_i^{(t+1)}\). Detailed trust-update rules are given in Appendix~\ref{app:trust_update_details}.

We then apply a conservative transport update to uncertain examples. Let $\mathcal{M}^{(t)}=\{j\in\mathcal{U}: \tilde q_j^{(t)}\in I_{\mathrm{amb}}\}$ be the uncertain set. For each \(j\in\mathcal{M}^{(t)}\), the transport module computes positive and negative inflow scores \(m_j^{+,(t)}\) and \(m_j^{-,(t)}\) in the learned latent space. Rather than solving a full optimal transport problem, we use a sparse budgeted propagation rule:
\[
q_{j}^{+,(t)}
=
\tilde q_j^{(t)}
+
\eta_+m_j^{+,(t)}
\left(1-\tilde q_j^{(t)}\right),
\qquad
q_j^{(t+1)}
=
q_{j}^{+,(t)}
-
\eta_-m_j^{-,(t)}q_{j}^{+,(t)}.
\]
Examples outside \(\mathcal{M}^{(t)}\) keep their trust-updated estimates, \(q_j^{(t+1)}=\tilde q_j^{(t)}\). For the next round, we store \(m_j^{(t+1)}=m_j^{+,(t)}\), while the negative inflow is used only in the current update. Details of the inflow computation are provided in Appendix~\ref{app:transport_details}.

Finally, \PUAuditPlus selects unresolved examples for human verification. Returned labels are added to \(\mathcal{V}^{+}\) or \(\mathcal{V}^{-}\), while pending examples are kept in \(\mathcal{W}^{(t)}\) to avoid repeated queries. The loop stops when the refined estimates, hard assignments, transport inflow, and verified set size become stable. Verification and stopping details are provided in Appendix~\ref{app:verification_stopping_details}. 

Overall, each refinement round follows
\[
(q^{(t)},a^{(t)},m^{(t)})
\rightarrow
(z^{(t)},p^{(t)})
\rightarrow
(\tilde q^{(t)},a^{(t+1)})
\rightarrow
(q^{(t+1)},a^{(t+1)},m^{(t+1)}),
\]
with selective human verification updating \(\mathcal{V}^{+}\), \(\mathcal{V}^{-}\), and the waiting set \(\mathcal{W}^{(t)}\).

\section{Theoretical results}
In this section, we develop a theoretical analysis of \PUAuditPlus that proceeds from algorithmic foundations to statistical guarantees. We begin by establishing that the joint optimization converges, and then sharpen this guarantee to a linear convergence rate via a contraction argument.

\subsection{Convergence guarantee for outer loop iteration}\label{sec:convergence-stability}

The outer loop in Section~\ref{sec:method} alternates three algorithmic blocks: encoder training, trust update, and transport reassignment, with selective verification acting as a bounded perturbation that modifies the verified sets and anchor pool. We construct a surrogate Lyapunov function that aggregates progress across these three blocks and show that it decreases sufficiently along the iterates, up to summable error induced by verification. We begin by stating the regularity assumptions, which are standard in analyses of nonconvex optimization and alternating minimization methods~\citep{bertsekas1999nonlinear}.

\begin{assumption}[Encoder loss regularity]\label{ass:loss-reg}
The encoder loss 
$\mathcal{L}^{(t)}_{\mathrm{enc}}(\Theta)
=\mathcal{L}^{(t)}_{\mathrm{ver}}
+\lambda_{\mathrm{soft}}\mathcal{L}^{(t)}_{\mathrm{soft}}
+\lambda_{\mathrm{geo}}\mathcal{L}^{(t)}_{\mathrm{geo}}
+\lambda_{\mathrm{anchor}}\mathcal{L}^{(t)}_{\mathrm{anchor}}$
satisfies: 
(i) it is continuously differentiable in $\Theta$ with $L_\Theta$-Lipschitz gradient; 
(ii) it is nonnegative, i.e., $\mathcal{L}^{(t)}_{\mathrm{enc}}(\Theta)\geq0$; 
(iii) the encoder update produces an approximate minimizer with bounded suboptimality, 
$\mathcal{L}^{(t)}_{\mathrm{enc}}(\Theta^{(t)}) \le \min_\Theta \mathcal{L}^{(t)}_{\mathrm{enc}}(\Theta) + \delta_t$, 
where $\delta_t \ge 0$ and $\sum_t \delta_t < \infty$; 
(iv) the parameter space $\Theta$ is compact.
\end{assumption}

\begin{assumption}[Bounded trust-logit components]
\label{ass:trust-bdd}
The trust-logit weights $\lambda_p, \lambda_\mathrm{loc}, \lambda_\mathrm{anc}, \lambda_m \geq 0$
and the offset $\beta_0 \in \R$ are finite. The classifier output satisfies $p_i^{(t)} \in [\epsilon_p, 1-\epsilon_p]$ for some
$\epsilon_p \in (0, 1/2)$, enforced by the algorithm's clipping operations, and the local and anchor evidence scores satisfy $r_i^{\mathrm{loc},(t)}, r_i^{\mathrm{anc},(t)} \in [0,1]$.
\end{assumption}

\begin{assumption}[Summable transport budget]
\label{ass:budget}
The transport budget $B_t$ used in the transport step in Appendix~\ref{app:transport_details} satisfies
$\sum_{t=0}^\infty B_t^2 < \infty$.
\end{assumption}

\begin{assumption}\label{ass:trust and transport control}
Define
$H^{(t)}:=\rho_q\|q^{(t)}-q^{(t-1)}\|+\rho_m\|m^{(t)}-m^{(t-1)}\|$ for some constants $\rho_q,~\rho_m>0$. Assume there exist constants $\eta>0$, $C_B\geq 0$ and a summable sequence $\{\kappa_t\}_{t\geq 0}$, such that for each iteration $t$, $(1+\eta)H^{(t+1)}\leq H^{(t)}+C_BB_t^2+\kappa_t$.
\end{assumption}

The trust-update operator $\DT$ defined in Appendix~\ref{app:trust_update_details}
maps the encoder outputs $(p^{(t)}, z^{(t)}$ and previous inflow $ m^{(t)})$ to the trust-updated responsibilities $\bar{q}^{(t)}$. We now show that, on indices not subject
to hard constraints, the operator is Lipschitz continuous in its inputs, which will be used to control the cross terms in the proof of Lemma~\ref{lemma:suff-decrease}.


\begin{lemma}[Lipschitz continuity of trust logit]\label{lemma:trust:logit}
     Under restriction of $p^{(t)}_i\in[\epsilon_p, 1-\epsilon_p]$ for some $\epsilon_p\in(0, 1/2)$ to avoid divergence of $\log (p/(1-p))$ as $p\to0$ or $1$, the true logit $u_i^{(t)}$ defined in Appendix~\ref{app:trust_update_details} satisfies 
    \begin{equation}
        |u_i^{(t+1)}-u_i^{(t)}|\leq \frac{\lambda_p}{\epsilon_p(1-\epsilon_p)}|\Delta p^{(t+1)}|+2\lambda_\text{loc}|\Delta r^{\text{loc},(t+1)}|+2\lambda_\text{anc}|\Delta r^{\text{anc},(t+1)}|+2\lambda_m|\Delta m^{(t+1)}|,    \end{equation}
    where $\Delta p^{(t+1)}=p_i^{(t+1)}-p_i^{(t)},~\Delta r^{\text{loc},(t+1)}=r_i^{\text{loc},(t+1)}-r_i^{\text{loc},(t)},~\Delta r^{\text{anc},(t+1)}=r_i^{\text{anc},(t+1)}-r_i^{\text{anc},(t)},$ and $\Delta m^{(t+1)}=m_i^{(t+1)}-m_i^{(t)}$.
\end{lemma}

\begin{theorem}[Lipschitz continuity of $\mathcal{D}_{\mathrm{trust}}$ on unconstrained indices]\label{thm:trust-continuity} 
Let \(\mathcal{U}:=\{1,\ldots,n\}\setminus(\mathcal{P}_s\cup\mathcal{V}^{+}\cup\mathcal{V}^{-})\) denotes indices not subject to hard constraints. Under Assumption~\ref{ass:trust-bdd}, the trust update operator $\DT$ restricted to $\mathcal{U}$, is Lipschitz in its inputs $(p, r^{\text{loc}}, b_{\text{anc}}, m)$
with constant
    \[
    L_{\DT}
    :=
    \frac{1}{4}
    \max\left\{
    \frac{\lambda_p}{\epsilon_p(1-\epsilon_p)}, \;
    2\lambda_{\text{loc}}, \;
    2\lambda_{\text{anc}}, \;
    2\lambda_m
    \right\}.
    \]
    The Lipschitz constant in the inflow $m$ is
    $L_{\DT, m} = \lambda_m / 2$ when $(p, r^{\text{loc}}, r^{\text{anc}})$ are held fixed.
\end{theorem}
\begin{remark}
    The Lipschitz constants given here directly enter the proof of Lemma~\ref{lemma:suff-decrease}: they bound how much the trust--updated responsibilities $\bar{q}$ can moves per outer-loop iteration. Theorem~\ref{thm:trust-continuity} also clarifies how the hyperparameters
$(\lambda_p, \lambda_{\text{loc}}, \lambda_{\text{anc}}, \lambda_m)$ affect numerical stability: larger weights yield larger Lipschitz constants and thus weaker contraction.
\end{remark}

\subsection{Global convergence and conservation properties}\label{sec:convergence-property}
\begin{lemma}[Minimization with sufficient decrease]\label{lemma:suff-decrease}
Under Assumption~\ref{ass:loss-reg}--\ref{ass:budget}, the encoder update satisfies $\mathcal{L}^{(t)}_{\mathrm{enc}}(\Theta^{(t+1)})\leq\mathcal{L}^{(t)}_{\mathrm{enc}}(\Theta^{(t)})-c_\Theta\|\nabla_\Theta\mathcal{L}^{(t)}_{\mathrm{enc}}(\Theta^{(t)})\|^2$ for some constant $c_\Theta =\eta(1-L_\Theta\eta/2)>0$. With $\eta=1/L_\Theta$, the inequality holds with $c_\Theta=1/(2L_\Theta)$.
\end{lemma}

We now define the surrogate Lyapunov function used in the convergence analysis.

\begin{lemma}[Block--wise sufficient decrease]\label{lem:lyapunov}
    Define the surrogate Lyapunov function $V^{(t)}:=\mathcal{L}^{(t)}_{\mathrm{enc}}(\Theta^{(t)})+\rho_q\|q^{(t)}-q^{(t-1)}\|+\rho_m\|m^{(t)}-m^{(t-1)}\|$
    for some constants $\rho_q,~\rho_m>0$. Under Assumptions~\ref{ass:loss-reg}--\ref{ass:trust and transport control} and
Theorem~\ref{thm:trust-continuity}, there exist constants $c_\Theta, c_q, c_m > 0$ such that   \begin{equation}\label{eq:lyap-dec}
        V^{(t+1)}\leq V^{(t)}- c_\Theta \|\nabla_\Theta \mathcal{L}^{(t)}_{\mathrm{enc}}(\Theta^{(t-1)})\|^2
- c_q \|q^{(t+1)} - q^{(t)}\|
- c_m \|m^{(t+1)} - m^{(t)}\|
+ \mathcal{E}^{(t)}
    \end{equation}
    where the drift error $\mathcal{E}^{(t)}\geq0$ satisfies $\sum_{t=0}^\infty\mathcal{E}^{(t)}<\infty$.
\end{lemma}

The convergence theorem follows from the surrogate Lyapunov decreases.

\begin{theorem}[Convergence of the Iteration]\label{thm:convergence}
Under Assumptions~\ref{ass:loss-reg}--\ref{ass:trust and transport control}, the sequence 
$\{(\Theta^{(t)}, q^{(t)}, a^{(t)}, m^{(t)})\}_{t \geq 0}$ produced by the outer loop satisfies:
(a) $V^{(t)} \to V^\star$ for some $V^\star \geq 0$; 
(b) $\|\nabla_\Theta \mathcal{L}^{(t)}_{\mathrm{enc}}(\Theta^{(t-1)})\| \to 0$, 
$\|q^{(t+1)} - q^{(t)}\| \to 0$, and $\|m^{(t+1)} - m^{(t)}\| \to 0$; 
(c) every limit point $(\Theta^\star, q^\star, a^\star, m^\star)$ satisfies $\nabla_\Theta \mathcal{L}_{\mathrm{enc}}^\star(\Theta^\star) = 0,\quad
q^\star = \mathcal{T}_{\mathrm{trans}}\big(\mathcal{D}_{\mathrm{trust}}(p^\star, z^\star, m^\star)\big),\quad
m^\star = m^{+,\star},$
where $p^\star, z^\star$ are encoder outputs at $\Theta^\star$, and $\mathcal{T}_{\mathrm{trans}}$ denotes the transport update map in Appendix~\ref{app:transport_details}.
\end{theorem}



This theorem shows that the outer loop is numerically well-behaved: the encoder approaches a local minimizer of its loss, while the trust probabilities and transport inflows stabilize. Crucially, the limit point is mutually self-consistent, meaning that no block can improve unilaterally given the others, which formalizes \PUAuditPlus. In addition, the numerical stability is further supported by the mass conservation property of the transport operator: the total transported mass remains equal to the budget $B_t$ across iterations, ensuring that evidence propagation acts as a bounded and controlled perturbation to the refinement state. Formal statements of mass conservation (Lemma~\ref{lem:mass-cons}) and its role in bias reduction (Theorem~\ref{thm:bias-reduction}) are provided in Appendix~\ref{app:mass-conservation}.

\section{Experiments}
\label{sec:experiments}

We evaluate \PUAuditPlus in three parts. First, we use simulations to test whether the method can recover latent human-consistency structure under noisy judge signals and selective verification. Second, we evaluate real LLM-as-a-judge data across multiple judge models and question types. Third, we provide ablation studies to examine the contribution of each component. The evaluation metrics are provided in Appendix~\ref{sec:exp_metrics}.

\begin{table*}[t]
\centering
\caption{Simulation results averaged over five independent runs. Each run contains 640 examples: 160 source positive examples and 480 target/unlabeled examples, among which 168 are hidden positives and 312 are negatives. Adjusted accuracy, accuracy difference, and flip rate are reported as mean (max--min range) across runs. Bold values indicate positive improvements.}
\small
\resizebox{\textwidth}{!}{%
\begin{tabular}{cc|cccc}
\toprule
\midrule
\multicolumn{6}{c}{\textbf{Simulation Results for CD/DD Noise and SCAR/SAR Verification}}\\
\midrule
\midrule
\textbf{Noise}
& \textbf{Verification}
& \textbf{Orig. Acc.}
& \textbf{Adj. Acc.}
& \textbf{Diff.}
& \textbf{Flip Rate} \\
\midrule
CD & SCAR
& 73.75\%
& 85.41\% (1.33\%)
& \hlgreen{\textbf{+11.66\%}} (1.33\%)
& 39.53\% (1.09\%) \\

CD & SAR
& 73.75\%
& 84.31\% (1.22\%)
& \hlgreen{\textbf{+10.56\%}} (1.22\%)
& 41.56\% (1.51\%) \\

DD & SCAR
& 73.75\%
& 84.19\% (1.23\%)
& \hlgreen{\textbf{+10.44\%}} (1.23\%)
& 38.81\% (0.95\%) \\

DD & SAR
& 73.75\%
& 84.88\% (1.19\%)
& \hlgreen{\textbf{+11.13\%}} (1.19\%)
& 38.19\% (0.66\%) \\

\bottomrule
\end{tabular}}
\label{tab:simulation_cd_dd_scar_sar}
\end{table*}

\begin{figure}[t]
    \centering
    \includegraphics[width=1\linewidth]{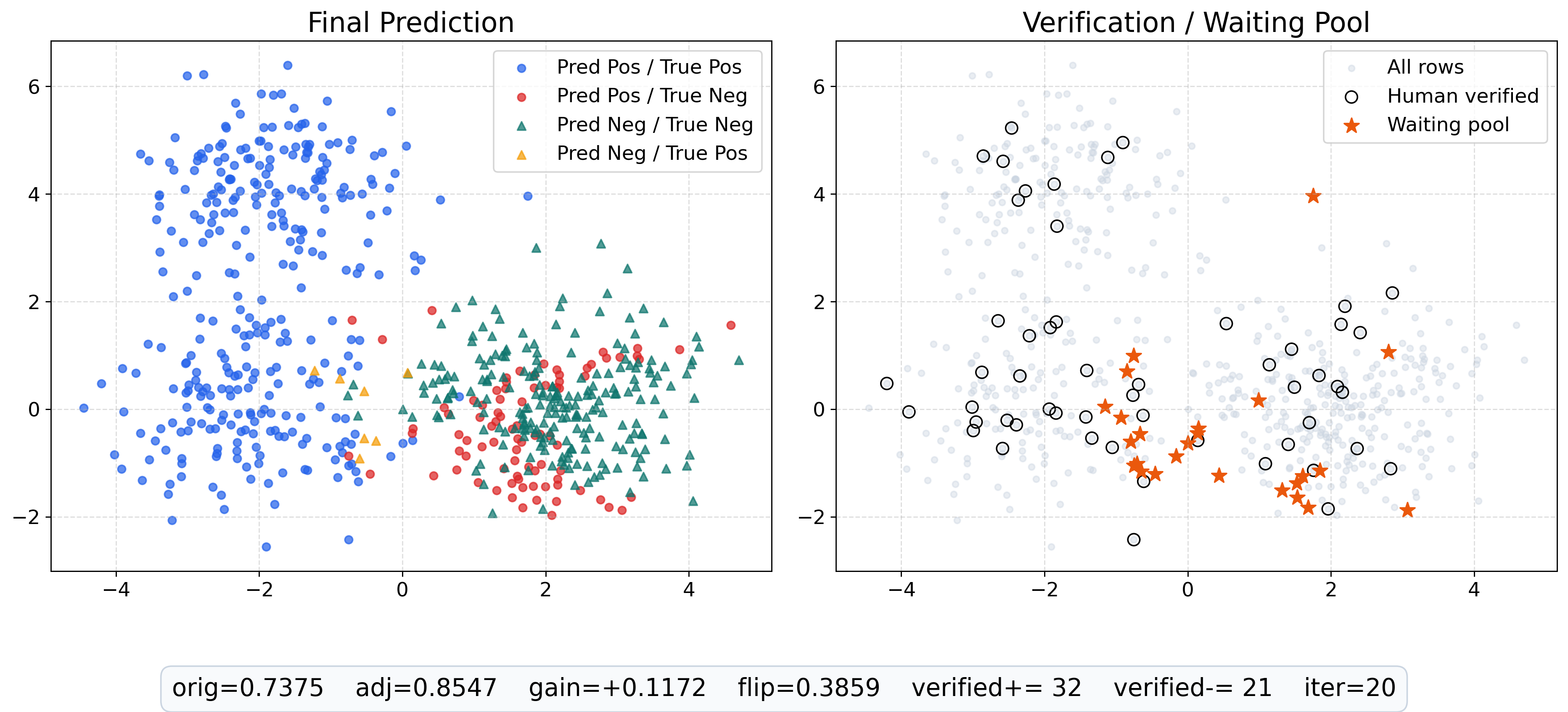}
    \caption{
    DD+SAR simulation visualization. 
    \textbf{Left}: verified and waiting-pool examples under selective verification. 
    \textbf{Right}: final hard prediction after applying \PUAuditPlus.
    }
    \label{fig:simulation_dd_sar_main}
\end{figure}

\subsection{Simulation}
\label{sec:simulation}

We first evaluate \PUAuditPlus in a generative simulation where two latent groups encode whether the initial judge signal is human-consistent. The simulator draws oracle latent coordinates for these groups, maps them to noisy observed comparison features, and corrupts the initial consistency signals, so \PUAuditPlus must recover useful latent structure using only noisy features and corrupted supervision. \PUAuditPlus observes only these noisy features and corrupted signals. We consider four settings: class-dependent (CD) or distribution-dependent (DD) corruption, crossed with selected-completely-at-random (SCAR) or selected-at-random (SAR) verification. CD corruption depends only on the latent consistency group, while DD corruption also depends on the feature geometry. SCAR verification samples verified examples conditionally at random, while SAR verification introduces feature-dependent selection bias. Details of the simulation design are provided in Appendix~\ref{app:simulation_details}.

Table~\ref{tab:simulation_cd_dd_scar_sar} shows that \PUAuditPlus improves the initial signal across all four noise and verification settings. The original accuracy is 73.75\%, while the adjusted accuracy reaches 84.19\%--85.41\%. The largest improvement appears in the CD+SCAR setting, where both corruption and verification are the most regular, and \PUAuditPlus raises accuracy from 73.75\% to 85.41\%. The method also remains effective in the harder DD+SAR setting, where both judge errors and human verification depend on the feature geometry. This setting is the closest synthetic analogue to selective human verification in real LLM-as-a-judge data, yet \PUAuditPlus still improves accuracy to 84.88\%. Figure~\ref{fig:simulation_dd_sar_main} visualizes this DD+SAR case. Although the verified set is incomplete and feature-biased, \PUAuditPlus recovers a clearer separation between human-consistent and human-inconsistent examples. The sensitivity analyses further support this interpretation: \PUAuditPlus benefits from clearer latent group separation, while weak separation in high-dimensional feature spaces makes local geometry and anchor evidence less reliable and therefore reduces the gain. Detailed results are provided in Appendix~\ref{app:simulation_geometry}.

\subsection{Real-Data evaluation}
\label{sec:real_data}

We next evaluate \PUAuditPlus on real LLM-as-a-judge data constructed from MT-Bench~\citep{zheng2023judging} and Chatbot Arena~\citep{chiang2024chatbot}. We use three question types, coding, math reasoning, and structured factual questions, and include one mixed question type MT-Bench setting. We evaluate five judge models: \textsc{GPT-5.4}, \textsc{GPT-5.4-mini}~\citep{openai_gpt54_blog}, \textsc{Gemini-2.5-Flash}~\citep{gemini25flash}, \textsc{Qwen2.5-7B-Instruct}~\citep{qwen25_technical_report}, and \textsc{Mistral-7B-Instruct-v0.3}~\citep{jiang2023mistral7b}. The full setup is described in Appendix~\ref{app:real_data_details}.

\begin{table}[t]
\centering
\caption{Coding results for GPT-5.4 and Qwen with 20\% verified baseline data. Accuracy and gain are reported as mean (max--min range). Human Verif. Req. denotes the number of human-verified examples required by each method. Baselines use 20\% verified data, while \PUAuditPlus uses a substantially smaller adaptive verification budget.}
\small
\resizebox{\linewidth}{!}{%
\begin{tabular}{lccccc}
\toprule
\midrule
\multicolumn{6}{c}{\textbf{Coding Results with 20\% Verified Baseline Data}}\\
\midrule
\textbf{Metric}
& \textbf{Logistic Reg.}
& \textbf{MLP}
& \textbf{Random Forest}
& \textbf{Label Prop.}
& \textbf{\PUAuditPlus} \\
\midrule
\multicolumn{6}{c}{\textbf{Model: GPT-5.4 (n = 2949, Orig. Acc = 63.92\%)}}\\
\midrule
Adj. Acc.
& 59.54\% (3.56)
& 60.83\% (2.80)
& 63.23\% (1.87)
& 60.50\% (9.92)
& \hlgreen{\textbf{67.20\% (2.00)}} \\
Gain
& -4.38\% (3.56)
& -3.09\% (2.80)
& -0.69\% (1.87)
& -3.42\% (9.92)
& \hlgreen{\textbf{+3.28\% (2.00)}} \\
Human Verif. Req.
& 590
& 590
& 590
& 590
& \hlgreen{\textbf{100.6}} \\
\midrule
\multicolumn{6}{c}{\textbf{Model: Qwen (n = 8884, Orig. Acc = 53.73\%)}}\\
\midrule
Adj. Acc.
& 58.47\% (1.59)
& 57.44\% (1.83)
& 59.20\% (0.96)
& 53.77\% (2.11)
& \hlgreen{\textbf{77.74\% (9.78)}} \\
Gain
& +4.75\% (1.59)
& +3.72\% (1.83)
& +5.47\% (0.96)
& +0.04\% (2.11)
& \hlgreen{\textbf{+24.02\% (9.78)}} \\
Human Verif. Req.
& 1777
& 1777
& 1777
& 1777
& \hlgreen{\textbf{278.8}} \\
\bottomrule
\end{tabular}}
\label{tab:gpt54_qwen_coding_20_verified_main}
\end{table}

Tables~\ref{tab:gpt54_qwen_coding_20_verified_main} shows that \PUAuditPlus consistently improves the original LLM-judge signal across datasets, judge models, and question types. The gains are consistent across judge families and evaluation settings. The main advantage of \PUAuditPlus is label efficiency. We compare against baselines trained with 3\%, 20\%, and 80\% verified data, while \PUAuditPlus uses a final verification budget close to the 3\% setting. Despite this much smaller human-label budget, \PUAuditPlus is competitive with or better than the baselines in most configurations. Table~\ref{tab:gpt54_qwen_coding_20_verified_main} gives a representative example. On GPT-5.4 and Qwen Coding, \PUAuditPlus achieves better adjusted accuracy while using only about 101 and 279 verified examples, compared with 590 and 1,777 examples used by baselines. \PUAuditPlus also has modest additional computational cost. Additional results are in Appendix~\ref{app:additional_real_data_results}. 
Its dominant overhead is the transport-based inflow update. If \(n_u\) uncertain examples are compared with \(n_a\) anchor examples in a \(d_z\)-dimensional latent space, a direct similarity-matrix implementation costs \(O(n_u n_a d_z)\) time and \(O(n_u n_a)\) memory. 
The runtime and memory results in Appendix~\ref{app:runtime_memory} further support its practical scalability on real-data experiments.

\subsection{Ablation study}
\label{sec:ablation}

\begin{table}[t]
\centering
\caption{Full factorial ablation on the Coding setting with \textsc{Gemini-2.5-Flash}. The original judge accuracy is 63.55\%. We vary four components: encoder training, full trust-state update, transport update, and human query selection.}
\small
\setlength{\tabcolsep}{5pt}
\renewcommand{\arraystretch}{1.08}
\resizebox{\linewidth}{!}{%
\begin{tabular}{ccc|ccc|ccc}
\toprule
\textbf{Full Trust}
& \textbf{Transport}
& \textbf{Human Query}
& \multicolumn{3}{c|}{\textbf{Trained Encoder}}
& \multicolumn{3}{c}{\textbf{Untrained Encoder}} \\
\cmidrule(lr){4-6}\cmidrule(lr){7-9}
& &
& \textbf{Adj. Acc.} & \textbf{Gain} & \textbf{Flips}
& \textbf{Adj. Acc.} & \textbf{Gain} & \textbf{Flips} \\
\midrule
\gcmark & \gcmark & \gcmark
& \hlgreen{\textbf{66.88\%}} & \hlgreen{\textbf{+3.33\%}} & 164
& \hlgreen{64.72\%} & \hlgreen{+1.17\%} & 183 \\

\gcmark & \gcmark & \rxmark
& \hlgreen{66.35\%} & \hlgreen{+2.80\%} & 153
& \hlgreen{66.13\%} & \hlgreen{+2.59\%} & 213 \\

\rxmark & \gcmark & \gcmark
& \hlgreen{66.81\%} & \hlgreen{+3.26\%} & 168
& \hlgreen{64.50\%} & \hlgreen{+0.96\%} & 179 \\

\rxmark & \gcmark & \rxmark
& \hlgreen{66.67\%} & \hlgreen{+3.12\%} & 168
& \hlgreen{66.03\%} & \hlgreen{+2.48\%} & 220 \\

\gcmark & \rxmark & \gcmark
& 53.40\% & -10.14\% & 2250
& 40.78\% & -22.77\% & 2576 \\

\gcmark & \rxmark & \rxmark
& 53.16\% & -10.39\% & 2257
& 40.71\% & -22.84\% & 2574 \\

\rxmark & \rxmark & \gcmark
& 39.72\% & -23.83\% & 2720
& 41.67\% & -21.88\% & 2547 \\

\rxmark & \rxmark & \rxmark
& 39.72\% & -23.83\% & 2720
& 41.03\% & -22.52\% & 2581 \\
\bottomrule
\end{tabular}}
\label{tab:ablation_mtbench_module_combo}
\end{table}

We ablate the main components of \PUAuditPlus on the Coding setting with \textsc{Gemini-2.5-Flash}. We vary four components: whether the encoder is trained, whether the trust update uses the full confidence state or only the model probability \(p_i\), whether the transport update is enabled, and whether human queries are selected by the proposed query score or uniformly at random. Implementation details for these four components are provided in Appendices~\ref{app:encoder_details}, \ref{app:trust_update_details}, \ref{app:transport_details}, and~\ref{app:verification_stopping_details}, respectively.

Table~\ref{tab:ablation_mtbench_module_combo} reports a full factorial ablation of the four main components. The complete model achieves the best adjusted accuracy, improving the original judge accuracy from 63.55\% to 66.88\%. Replacing human querying with random selection reduces the gain from 3.33 to 2.80 percentage points, and using only \(p_i\) in the trust update also lowers performance, showing that both adaptive verification and the full trust state are useful. The transport update has the largest effect: disabling it causes all variants to flip more than two thousand assignments and sharply reduces accuracy, indicating that the transport component plays an important role in stabilizing the refinement state and propagating anchor evidence. Trained encoders also generally outperform untrained ones when transport is enabled. Overall, the ablation supports the coupled design of \PUAuditPlus, combining representation learning, trust-state refinement, transport-based stabilization, and adaptive verification.

\section{Discussion and future work}
\label{sec:conclu}
Overall, these findings show that \PUAuditPlus is a useful method to improve LLM-as-a-judge evaluation when human labels are limited. Rather than assuming that reliable positive and unlabelled groups are available before auditing, \PUAuditPlus learns these groups progressively from judge signals, learned representations, transport-based evidence, and selective verification. This design improves adjusted accuracy in both simulations and real pairwise LLM-answer evaluation while using substantially fewer human verified examples than standard supervised baselines. More broadly, our results suggest that weakly supervised judge auditing should treat trust in the judge as a dynamic quantity: it should be refined and selectively verified as evidence accumulates. This perspective provides a practical path toward more label-efficient and reliable evaluation pipelines for LLM outputs.

One promising direction is to further improve the representation used for refinement. In real LLM evaluation data, human preference labels can vary substantially, because annotators may weigh reasoning, factuality, completeness, style, and usefulness differently. This makes the boundary between human-consistent and human-inconsistent judge decisions less clear than in controlled synthetic settings, and the current representation may not always capture these fine-grained distinctions cleanly. This observation suggests an opportunity for stronger task-adaptive encoders, richer comparison features, and uncertainty-aware treatment of human labels. We expect these extensions to further improve the robustness of adaptive PU auditing, especially in settings where human judgments are nuanced rather than purely deterministic.

\bibliographystyle{plainnat}
\bibliography{references_puauditplus}

\appendix
\clearpage
\newpage

\startcontents[appendix]
\section*{Appendix}
\printcontents[appendix]{}{1}{\setcounter{tocdepth}{2}}
\clearpage

\newpage
\section{Notation and figures}
\label{app:notation}

\begin{table}[h]
\centering
\caption{Summary of main notation. Detailed quantities used only locally are defined near their corresponding equations.}
\small
\begin{tabular}{ll}
\toprule
Symbol & Meaning \\
\midrule
$\mathcal{D}$ & pairwise comparison dataset \\
$Q_i$ & prompt or question \\
$R_i^1,R_i^2$ & two candidate responses \\
$J_i^{\mathrm L}$ & preference given by the LLM judge \\
$J_i^{\mathrm H}$ & human preference, when observed \\
$R_i^w,R_i^\ell$ & winner and loser induced by $J_i^{\mathrm L}$ \\
$y_i$ & judge--human consistency label, $y_i=\mathbf{1}\{J_i^{\mathrm L}=J_i^{\mathrm H}\}$ \\
$\mathcal{P}^{\star},\mathcal{N}^{\star}$ & ideal human-consistent and human-inconsistent sets \\
$\mathcal{U}$ & refinement pool \\
$\mathcal{P}_s$ & trusted positive seed/reference set available before refinement \\
$\mathcal{V}^{+},\mathcal{V}^{-}$ & verified human-consistent and human-inconsistent examples \\
$\mathcal{W}^{(t)}$ & waiting pool of examples selected for review but not yet incorporated \\
\midrule
$E_{\mathrm{RM}}$ & frozen reward-model or text encoder \\
$x_i$ & fixed directional comparison feature \\
$z_i^{(t)}$ & learned latent representation at iteration $t$ \\
$\bar z_i^{(t)}$ & normalized latent representation \\
$p_i^{(t)}$ & preliminary human-consistency estimate from the encoder \\
$q_i^{(t)}$ & refined human-consistency estimate at iteration $t$ \\
$\tilde q_i^{(t)}$ & trust-updated consistency estimate before transport refinement \\
$a_i^{(t)}$ & anchor confidence at iteration $t$ \\
$m_i^{(t)}$ & carried positive transport inflow from the previous round \\
\midrule
$r_i^{\mathrm{loc},(t)}$ & local-consistency evidence score \\
$r_i^{\mathrm{anc},(t)}$ & anchor-support evidence score \\
$\mathcal{M}^{(t)}$ & uncertain examples receiving transport refinement \\
$\mathcal{B}_{+}^{(t)},\mathcal{B}_{-}^{(t)}$ & positive and negative anchor pools for transport \\
$m_j^{+,(t)},m_j^{-,(t)}$ & positive and negative transport inflows \\
$\mathcal{C}^{(t)}$ & candidate set for human verification \\
$H_i^{(t)}$ & final query score for human verification \\
\bottomrule
\end{tabular}
\label{tab:notation}
\end{table}

\begin{figure}[t]
\centering
\includegraphics[width=1\linewidth]{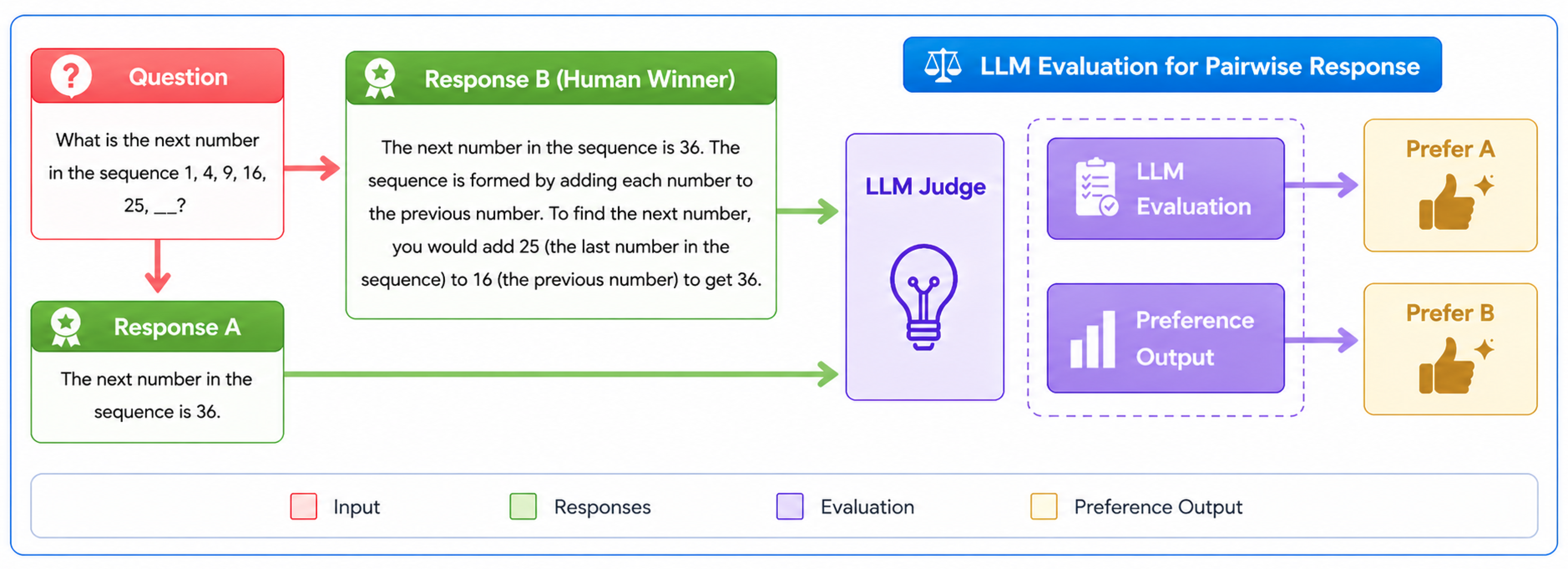}
\caption{Illustration of the pairwise LLM-as-a-judge setting. Given a question and two candidate responses, the LLM judge produces an initial preference between the responses, which may differ from the human preference.}
\label{fig:data_example}
\end{figure}

\newpage

\begin{algorithm}[t]
\caption{Progressive PU Refinement Workflow}
\label{alg:progressive_pu_refinement}
\begin{algorithmic}[1]
\Require Pairwise dataset \(\mathcal{D}\), reward encoder \(E_{\mathrm{RM}}\), human-verification oracle

\State \(X \leftarrow \mathrm{BuildDiffFeatures}(\mathcal{D},E_{\mathrm{RM}})\)
\State \((\mathcal{P}_s,\mathcal{V}^{+},\mathcal{V}^{-},\mathcal{W},q,a,m) \leftarrow \mathrm{InitializeState}(\mathcal{D})\)

\For{\(t=0,\dots,T-1\)}
    \State \((z^{(t)},p^{(t)}) \leftarrow \mathrm{TrainEncoder}(X,q,a,m,\mathcal{V}^{+},\mathcal{V}^{-})\)
    \State \((\tilde q^{(t)},a) \leftarrow \mathrm{TrustUpdate}(z^{(t)},p^{(t)},q,a,m,\mathcal{P}_s,\mathcal{V}^{+},\mathcal{V}^{-})\)
    \State \((q,m) \leftarrow \mathrm{TransportUpdate}(\tilde q^{(t)},a,z^{(t)},\mathcal{P}_s,\mathcal{V}^{+},\mathcal{V}^{-})\)
    \State \((\mathcal{V}^{+},\mathcal{V}^{-},\mathcal{W}) \leftarrow \mathrm{SelectForVerification}(q,m,z^{(t)},p^{(t)},\mathcal{W})\)
    \If{\(\mathrm{StopCriterion}(q,m,\mathcal{V}^{+},\mathcal{V}^{-})\)}
        \State \textbf{break}
    \EndIf
\EndFor

\State \Return refined responsibilities \(q\), refined positive set \(\widehat{\mathcal{P}}=\{i:q_i\ge 1/2\}\), and verified sets \((\mathcal{V}^{+},\mathcal{V}^{-})\)

\end{algorithmic}
\end{algorithm}

\section{Details of methodology}
\subsection{Trainable encoder and representation learning}
\label{app:encoder_details}

The refinement algorithm takes \(\{x_i\}_{i=1}^n\) as fixed inputs and trains a task-specific encoder on top of these comparison features. At outer iteration \(t\), the encoder maps \(x_i\) into a latent representation
\[
z_i^{(t)}
=
W_2^{(t)}
\operatorname{GELU}
\left(
W_1^{(t)}x_i+b_1^{(t)}
\right)
+
b_2^{(t)}\in\R^{d_z}.
\]
Here, \(\operatorname{GELU}\) denotes the Gaussian Error Linear Unit activation function, defined as
\[
\operatorname{GELU}(x)=x\Phi(x),
\]
where \(\Phi(\cdot)\) is the cumulative distribution function of the standard normal distribution. A linear prediction head then produces a logit
\[
\ell_i^{(t)}
=
(w^{(t)})^\top z_i^{(t)}
+
b^{(t)},\quad p_i^{(t)}
=
\sigma(\ell_i^{(t)})
=
\frac{1}{1+\exp(-\ell_i^{(t)})}.
\]
We use \(p_i^{(t)}\) as a preliminary estimate of whether the LLM-induced preference is human-consistent. For similarity-based operations, we use
\[
\bar z_i^{(t)}
=
\operatorname{Norm}_{\epsilon}(z_i^{(t)})
=
\frac{z_i^{(t)}}{\|z_i^{(t)}\|+\epsilon},
\]
where \(\epsilon>0\) is fixed for numerical stability.

The encoder and prediction head are trained with verified supervision, confidence-aware soft supervision, local geometric regularization, and anchor-preserving regularization:
\[
\mathcal{L}_{\mathrm{enc}}^{(t)}
=
\mathcal{L}_{\mathrm{ver}}^{(t)}
+
\lambda_{\mathrm{soft}}\mathcal{L}_{\mathrm{soft}}^{(t)}
+
\lambda_{\mathrm{geo}}\mathcal{R}_{\mathrm{geo}}^{(t)}
+
\lambda_{\mathrm{anchor}}\mathcal{R}_{\mathrm{anchor}}^{(t)}.
\]

\paragraph{Verified supervision.}
For verified examples \(i\in\mathcal{V}^{+}\cup\mathcal{V}^{-}\), define the target $\widetilde y_i=\mathbf{1}\{i\in\mathcal{V}^{+}\}$. We use
\[
\mathcal{L}_{\mathrm{ver}}^{(t)}
=
\frac{
\sum_{i\in\mathcal{V}^{+}\cup\mathcal{V}^{-}}
\hat\omega_{\mathrm{ver}}
\operatorname{BCE}(\widetilde y_i,\ell_i^{(t)})
}{
\max(1,|\mathcal{V}^{+}\cup\mathcal{V}^{-}|)
},
\]
where \(\operatorname{BCE}\) denotes binary cross-entropy between a binary target \(y\in\{0,1\}\) and a logit \(\ell\):

\[
\operatorname{BCE}(y,\ell)=-y\log\sigma(\ell)-(1-y)\log\bigl(1-\sigma(\ell)\bigr),\qquad\hat\omega_{\mathrm{ver}}=1+\omega_{\mathrm{ver}}a_v.
\]
Since verified examples have fixed anchor confidence \(a_v\), this verified-example weight is constant across verified examples and iterations.

\paragraph{Confidence-aware soft supervision.}
We also train on high-confidence unverified examples using their current responsibilities as soft labels. Define the confidence of the current responsibility as $c_i^{(t)}=\max\{q_i^{(t)},1-q_i^{(t)}\}$.
Among examples that are not verified and not source positives, denoted by \(\mathcal{U}_{\mathrm{soft}}^{(t)}\), we select the high-confidence set
\[
\mathcal{H}_{\mathrm{soft}}^{(t)}
=
\{i\in\mathcal{U}_{\mathrm{soft}}^{(t)}:c_i^{(t)}\ge \delta_{\mathrm{soft}}^{(t)}\}, \quad \delta_{\mathrm{soft}}^{(t)}
=
\operatorname{Quantile}_{\tau_{\mathrm{soft}}^{(t)}}
(\{c_j^{(t)}:j\in\mathcal{U}_{\mathrm{soft}}^{(t)}\})
.
\]

The loss is
\[
\mathcal{L}_{\mathrm{soft}}^{(t)}
=
\frac{
\sum_{i\in\mathcal{H}_{\mathrm{soft}^{(t)}}}
\omega_i^{\mathrm{soft},(t)}
\operatorname{BCE}(q_i^{(t)},\ell_i^{(t)})
}{
\sum_{i\in\mathcal{H}_{\mathrm{soft}^{(t)}}}
\omega_i^{\mathrm{soft},(t)}
+\epsilon
},\quad \omega_i^{\mathrm{soft},(t)}
=
\frac{1}{2}
+
\frac{1}{2}a_i^{(t)}.
\]

\paragraph{Local geometric regularization.}
To encourage locally similar examples to have consistent predictions and representations, we use a precomputed neighborhood system. Specifically, we construct a \(k\)-nearest-neighbor graph in the fixed comparison-feature space \(\{x_i\}_{i=1}^n\) using cosine similarity. Let \(\mathcal{N}(i)\) denote the \(k\) nearest neighbors of example \(i\), and define
\[
A_{ij}
=
\frac{\exp(\tau_{\mathrm{geo}}\langle \bar x_i,\bar x_j\rangle)}
{\sum_{r\in\mathcal{N}(i)}\exp(\tau_{\mathrm{geo}}\langle \bar x_i,\bar x_r\rangle)+\epsilon},
\qquad j\in\mathcal{N}(i),
\]
where \(\bar x_i=x_i/(\|x_i\|+\epsilon)\). We define
\[
\mathcal{R}_{\mathrm{geo}}^{(t)}
=
\frac{1}{n}
\sum_{i=1}^n
\sum_{j\in\mathcal{N}(i)}
A_{ij}
\left[
(p_i^{(t)}-p_j^{(t)})^2
+
\eta_z
\|z_i^{(t)}-z_j^{(t)}\|^2
\right].
\]
The first term encourages nearby examples to have similar human-consistency scores, while the second term encourages smoothness of the learned latent representation. The coefficient \(\eta_z\ge 0\) controls the strength of the latent-space smoothness penalty.

\paragraph{Anchor-preserving regularization.}
Since the encoder is re-trained across outer iterations, we penalize large movement of high-confidence examples in latent space:
\[
\mathcal{R}_{\mathrm{anchor}}^{(t)}
=
\frac{1}{n}
\sum_{i=1}^n
a_i^{(t)}
\|z_i^{(t)}-z_i^{(t-1)}\|^2.
\]
This term is omitted for \(t=0\). It stabilizes the learned representation around reliable anchors. Hence, the encoder parameters \(\Theta^{(t)}\) are updated by minimizing $\mathcal{L}_{\mathrm{enc}}^{(t)}(\Theta;q^{(t)},a^{(t)},z^{(t-1)})$.

\subsection{Trust update}
\label{app:trust_update_details}
After encoder training, we use the learned representation and model score to update the refinement state. The trust update combines four signals: the encoder score \(p_i^{(t)}\), local consistency \(r_i^{\mathrm{loc},(t)}\), anchor support \(r_i^{\mathrm{anc},(t)}\), and the previous transport inflow \(m_i^{(t)}\).

\paragraph{Local consistency evidence.}
Using the neighborhood weights \(A_{ij}\), we define
\[
r_i^{\mathrm{loc},(t)}
=
\Pi_{[0,1]}
\left(
1-
\sum_{j\in\mathcal{N}(i)}
A_{ij}
\left|
p_i^{(t)}-p_j^{(t)}
\right|
\right),
\]
where \(\Pi_{[0,1]}(x)=\min\{1,\max\{0,x\}\}\) projects a scalar onto the interval \([0,1]\). This score is large when example \(i\)'s encoder score is consistent with those of its local neighbors.

\paragraph{Anchor-support evidence.}
We compare each example with positive and negative anchors in the learned latent space. Let
\[
\mathcal{A}_{+}^{(t)}
=
\mathcal{P}_s\cup\mathcal{V}^{+},
\qquad
\mathcal{A}_{-}^{(t)}
=
\mathcal{V}^{-}.
\]
Define the normalized positive and negative prototypes as
\[
\mu_+^{(t)}
=
\operatorname{Norm}_{\epsilon}
\left(
\frac{1}{|\mathcal{A}_{+}^{(t)}|}
\sum_{j\in\mathcal{A}_{+}^{(t)}}
\bar z_j^{(t)}
\right),
\qquad
\mu_-^{(t)}
=
\operatorname{Norm}_{\epsilon}
\left(
\frac{1}{|\mathcal{A}_{-}^{(t)}|}
\sum_{j\in\mathcal{A}_{-}^{(t)}}
\bar z_j^{(t)}
\right),
\]
The anchor-support evidence is
\[
r_i^{\mathrm{anc},(t)}
=
\operatorname{Scale}
\left(
\left\langle \bar z_i^{(t)},\mu_+^{(t)}\right\rangle
-
\left\langle \bar z_i^{(t)},\mu_-^{(t)}\right\rangle
\right).
\]
Here \(\operatorname{Scale}(\cdot)\) denotes min--max scaling across all examples in the current round, i.e.,
\(\operatorname{Scale}(s_i)=(s_i-\min_k s_k)/(\max_k s_k-\min_k s_k+\epsilon)\)
for scores \(\{s_k\}_{k=1}^n\). A larger \(r_i^{\mathrm{anc},(t)}\) means that example \(i\) is closer to the positive-anchor direction than to the negative-anchor direction.

\paragraph{Trust-state update.}
The trust module combines encoder evidence, local evidence, anchor evidence, and transport evidence into a single trust logit:
\[
u_i^{(t)}
=
\lambda_p
\log
\frac{p_i^{(t)}}{1-p_i^{(t)}}
+
\lambda_{\mathrm{loc}}
\left(2r_i^{\mathrm{loc},(t)}-1\right)
+
\lambda_{\mathrm{anc}}
\left(2r_i^{\mathrm{anc},(t)}-1\right)
+
\lambda_m
\left(2m_i^{(t)}-1\right)
+
\beta_0,
\]
where \(\lambda_p,\lambda_{\mathrm{loc}},\lambda_{\mathrm{anc}},\lambda_m\) weight the four evidence sources, and \(\beta_0\) is an intercept term controlling the overall conservativeness of the update.

We then apply source and verified constraints:
\[
\begin{aligned}
\tilde q_i^{(t)}
&=
\begin{cases}
1, & i\in\mathcal{P}_s\cup\mathcal{V}^{+},\\
0, & i\in\mathcal{V}^{-},\\
\sigma(u_i^{(t)}), & \text{otherwise},
\end{cases}
&
a_i^{(t+1)}
&=
\begin{cases}
a_s, & i\in\mathcal{P}_s,\\
a_v, & i\in\mathcal{V}^{+}\cup\mathcal{V}^{-},\\
\tilde a_{i}^{(t+1)}, & \text{otherwise}.
\end{cases}
\end{aligned}
\]
where $\tilde a_{i}^{(t+1)}
=
\Pi_{[0,1]}
\left(
\gamma_q \tilde q_i^{(t)}
+
\gamma_{\mathrm{loc}}r_i^{\mathrm{loc},(t)}
\right).$ Here \(\gamma_q\) and \(\gamma_{\mathrm{loc}}\) control the contributions of each parts. The resulting pair \((\tilde q^{(t)},a^{(t+1)})\) is passed to the transport reassignment module, which produces the next responsibility \(q^{(t+1)}\).

\subsection{Conservative transport-based inflow update}
\label{app:transport_details}

After the trust update, we compute transport-based inflow scores for uncertain unlabeled examples. In this subsection, anchors are indexed by \(i\), while uncertain examples in \(\mathcal{M}^{(t)}\) are indexed by \(j\). Here, \(
\mathcal{M}^{(t)}
=
\{j\in\mathcal{U}:\tilde q_j^{(t)}\in I_{\mathrm{amb}}\}
\)
denotes the uncertain set, where \(I_{\mathrm{amb}}\subset(0,1)\) selects examples with non-decisive responsibilities.  Motivated by partial optimal transport, which allows only a budgeted portion of mass to be matched, we use a sparse, budgeted propagation rule rather than solving a full optimal transport problem~\citep{chapel2020partial,peyre2019computational}. This gives a positive inflow \(m_j^{+,(t)}\) from reliable positive anchors and a negative inflow \(m_j^{-,(t)}\) from verified negative anchors. Only \(m_j^{+,(t)}\) is carried forward as \(m_j^{(t+1)}\); \(m_j^{-,(t)}\) is used only in the current reassignment step. At iteration \(t\), the positive and negative anchor pools are
\[
\mathcal{B}_{+}^{(t)}
=
\mathcal{P}_s
\cup
\mathcal{V}^{+}
\cup
\{i\in\mathcal{U}:\tilde q_i^{(t)}\ge 1/2,\ a_i^{(t+1)}\ge\kappa_t\}, \quad \mathcal{B}_{-}^{(t)}=\mathcal{V}^{-}
\]
where \(\kappa_t\) is a confidence threshold.

For each uncertain example \(j\in\mathcal{M}^{(t)}\), we first compute a local relation-consistency score
\[
c_j^{(t)}
=
\Pi_{[0,1]}
\left(
1-
\left|
\tilde q_j^{(t)}
-
\frac{1}{|\mathcal{N}(j)|}
\sum_{r\in\mathcal{N}(j)}
\tilde q_r^{(t)}
\right|
\right).
\]
This score is large when the trust-updated responsibility of \(j\) agrees with its neighbourhood, and it is used to downweight transport into locally inconsistent examples.

We compute the inflows by propagating a bounded amount of mass from positive anchors to uncertain examples. For a positive anchor \(i\in\mathcal{B}_{+}^{(t)}\) and an uncertain example \(j\in\mathcal{M}^{(t)}\), we first compute their latent cosine similarity
\(
S_{ij}^{+,(t)}
=
\left\langle
\bar z_i^{(t)},\bar z_j^{(t)}
\right\rangle .
\)
For each anchor \(i\), we keep only the most similar uncertain examples and assign them a transport affinity
\[
\Gamma_{ij}^{+,(t)}
=
\frac{
\exp(\tau_{\mathrm{tr}}S_{ij}^{+,(t)})
\left(1-\lambda_{\mathrm{rel}}+\lambda_{\mathrm{rel}}c_j^{(t)}\right)
}{
\sum_{j'}
\exp(\tau_{\mathrm{tr}}S_{ij'}^{+,(t)})
\left(1-\lambda_{\mathrm{rel}}+\lambda_{\mathrm{rel}}c_{j'}^{(t)}\right)
+\epsilon
},
\]
where the sum is over the retained uncertain examples for anchor \(i\). Thus, \(\Gamma_{ij}^{+,(t)}\) is large when \(j\) is close to \(i\) in latent space and is locally consistent with its neighbors. The parameter \(\tau_{\mathrm{tr}}\) controls how strongly transport favors the most similar targets, while \(\lambda_{\mathrm{rel}}\) controls the contribution of local relation consistency. We assign source masses separately for positive and negative anchors:
\[
\pi_i^{+,(t)}
=
\frac{
a_i^{(t+1)}
}{
\sum_{r\in\mathcal{B}_{+}^{(t)}}a_r^{(t+1)}+\epsilon
},
\quad i\in\mathcal{B}_{+}^{(t)},
\qquad
\pi_i^{-,(t)}
=
\frac{1}{|\mathcal{B}_{-}^{(t)}|+\epsilon},
\quad i\in\mathcal{B}_{-}^{(t)}.
\]
Positive anchors are weighted by anchor confidence, while negative anchors are weighted uniformly. The negative affinity \(\Gamma_{ij}^{-,(t)}\) is defined analogously by replacing \(\mathcal{B}_{+}^{(t)}\) with \(\mathcal{B}_{-}^{(t)}\). Given budgets \(B_t^{+}\) and \(B_t^{-}\), we aggregate anchor-to-target transported mass:
\[
\begin{aligned}
m_j^{+,(t)}=
\sum_{i\in\mathcal{B}_{+}^{(t)}}
B_t^{+}\,\pi_i^{+,(t)}\Gamma_{ij}^{+,(t)},\quad
m_j^{-,(t)}=
\sum_{i\in\mathcal{B}_{-}^{(t)}}
B_t^{-}\,\pi_i^{-,(t)}\Gamma_{ij}^{-,(t)}.
\end{aligned}
\]
The positive inflow \(m_j^{+,(t)}\) supports increasing the responsibility of \(j\), whereas the negative inflow \(m_j^{-,(t)}\) supports decreasing it.
Thus, \(m_j^{+,(t)}\) pushes \(j\) toward the positive group, while \(m_j^{-,(t)}\) pushes it away from the positive group.

Finally, the responsibility is updated by applying the two inflows in opposite directions:
\[
q_{j}^{+,(t)}
=
\tilde q_j^{(t)}
+
\eta_+m_j^{+,(t)}
\left(1-\tilde q_j^{(t)}\right),
\qquad
q_j^{(t+1)}
=
q_{j}^{+,(t)}
-
\eta_-m_j^{-,(t)}q_{j}^{+,(t)}.
\]
The first update increases responsibility using positive inflow, while the second decreases it using negative inflow. For examples not in \(\mathcal{M}^{(t)}\), we keep \(q_j^{(t+1)}=\tilde q_j^{(t)}\).

Notably, only the positive inflow is carried forward to the next iteration, i.e., \(m_j^{(t+1)}=m_j^{+,(t)}\), while the negative inflow \(m_j^{-,(t)}\) is used only within the current transport update.

\subsection{Selective verification and stopping details}
\label{app:verification_stopping_details}

After transport update, the candidate set is
\[
\mathcal{C}^{(t)}
=
\mathcal{U}
\setminus
\left(
\mathcal{V}^{+}
\cup
\mathcal{V}^{-}
\cup
\mathcal{W}^{(t)}
\right),
\]
where \(\mathcal{W}^{(t)}\) contains examples already selected for human review but whose labels have not yet been incorporated.

For each candidate \(i\in\mathcal{C}^{(t)}\), we compute
\[
U_i^{(t)}=1-\left|2q_i^{(t+1)}-1\right|,
\quad
I_i^{(t)}=p_i^{(t)}(1-p_i^{(t)})(\|z_i^{(t)}\|^2+1),
\quad
D_i^{(t)}=1-\max_{r\in\mathcal{V}^{(t)}}\langle \bar z_i^{(t)},\bar z_r^{(t)}\rangle .
\]
Here \(\mathcal{V}^{(t)}=\mathcal{V}^{+}\cup\mathcal{V}^{-}\) denotes the currently verified set, and we set \(D_i^{(t)}=1\) when \(\mathcal{V}^{(t)}=\varnothing\). The final query score is
\[
H_i^{(t)}
=
\omega_U\operatorname{Scale}\!\left(U_i^{(t)}\right)
+
\omega_I\operatorname{Scale}\!\left(I_i^{(t)}\right)
+
\omega_D\operatorname{Scale}\!\left(D_i^{(t)}\right)
+
\omega_M\operatorname{Scale}\!\left(m_i^{(t+1)}\right),
\]
where \(\operatorname{Scale}(\cdot)\) is computed over \(\mathcal{C}^{(t)}\). We select the top \(b\) candidates with the largest \(H_i^{(t)}\) for human verification. If the returned human label agrees with the LLM judge, the example is added to \(\mathcal{V}^{+}\); otherwise, it is added to \(\mathcal{V}^{-}\). If the label is not yet available, the example is kept in \(\mathcal{W}^{(t+1)}\).

For stopping, we monitor
\[
\Delta_q^{(t)}
=
\frac{1}{n}\sum_{i=1}^{n}
|q_i^{(t+1)}-q_i^{(t)}|,
\]
\[
\Delta_{\mathrm{flip}}^{(t)}
=
\frac{1}{n}
\sum_{i=1}^{n}
\mathbf{1}
\left\{
\mathbf{1}(q_i^{(t+1)}\ge 1/2)
\neq
\mathbf{1}(q_i^{(t)}\ge 1/2)
\right\},
\]
\[
\Delta_m^{(t)}
=
\frac{1}{n}\sum_{i=1}^{n}
|m_i^{(t+1)}-m_i^{(t)}|,
\qquad
\Delta_{\mathrm{ver}}^{(t)}
=
|\mathcal{V}^{(t+1)}|-|\mathcal{V}^{(t)}|.
\]
The rule-based stopping criterion checks
\[
\Delta_q^{(t)}\le \epsilon_q,\qquad
\Delta_{\mathrm{flip}}^{(t)}\le \epsilon_{\mathrm{flip}},\qquad
\Delta_m^{(t)}\le \epsilon_m,\qquad
\Delta_{\mathrm{ver}}^{(t)}\le \epsilon_{\mathrm{ver}}.
\]
The algorithm stops when at least three of the four conditions hold for the required number of consecutive rounds, after a minimum number of iterations. The loop also terminates when the maximum iteration count or annotation budget is reached.

\newpage

\section{Additional theoretical results}\label{app:mass-conservation}

Possible additions include:
\begin{itemize}[leftmargin=1.5em]
    \item sensitivity analysis with respect to noisy initial PU memberships,
    \item convergence properties of the iterative membership refinement,
    \item error decomposition for fixed-group versus learned-group auditing.
\end{itemize}

\subsection{Mass conservation and bias reduction of $\TT$}

Let $m^{(t)}_+$ denote the positive inflow vector after the transport step at iteration $t$. Because of top--$K_{\text{tr}}$ truncation, each $m^{t)}_+$ is sparse and bounded in Euclidean norm.

\begin{lemma}[Mass conservation under budget]
\label{lem:mass-cons}
Define $C_B > 0$ such that $\|m^{+,(t)}\| \leq C_B$
for all $t$. Since the transported mass is controlled by the budget, this constant can be chosen uniformly as $C_B \leq \sup_t B_t^+$. Moreover, the total positive transport mass equals the budget:
\[
\sum_{i \in \mathcal{B}_+^{(t)}} \sum_{j \in \mathcal{M}^{(t)}} T_{ij}^{+,(t)}
= B_t^+ \sum_{i \in \mathcal{B}_+^{(t)}} \pi_i^{+,(t)}
\sum_{j \in \mathcal{M}^{(t)}} \Gamma_{ij}^{+,(t)}
= B_t^+,
\]
where $\pi_i^{+,(t)}$ is a probability distribution over the positive anchor pool $\mathcal{B}_+^{(t)}$, $\Gamma_{ij}^{+,(t)}$ is the retained target affinity over $\mathcal{M}^{(t)}$, and $T^{+,(t)}_{ij}$ denotes the positive transport mass from anchor $i$ to uncertain example $j$.

\end{lemma}
To relate transport to label geometry, we assume that nearby points in the learned latent space share similar human--consistency labels.

\begin{assumption}\label{ass:geo_anchor}
    Let $y_i\in\{0,1\}$ denotes the true human--consistency labels defined in Section~\ref{sec:problem_formulation}. There exists a constant $L_{\text{loc}}>0$ such that for all $i,j$,
    \begin{equation*}  
    |y_i - y_j| \leq L_{\text{loc}} \|\bar{z}_i^{(t)} - \bar{z}_j^{(t)}\|.
    \end{equation*}
    where $\bar{z}_i^{(t)}$ is the normalized latent representation at iteration $t$.
\end{assumption}
\begin{remark}
    This is a latent--space analogue of the standard cluster assumption in semi--supervised learning and manifold hypothesis~\citep{bengio2013representation}; samples that cluster together in the encoder's representation space should share same human--consistency label. It is reasonable when the encoder is sufficiently expressive to capture the semantic structure of the LLM--judge task.
\end{remark}

Under this assumption, we can show that conservative positive transport reduce a surrogate consistency error in expectation.

\begin{theorem}\label{thm:bias-reduction}
    Let $y\in\{0,1\}^n$ is the real human--consistency label and suppose Assumption~\ref{ass:geo_anchor} holds. Consider the update that applies only the positive inflow to uncertain examples, with $\eta_-=0$, so that 
    $$
    q^{(t+1)}_j=\tilde{q}_j^{(t)}+\eta_+m^{(t)}_{+,j}(1-\tilde q^{(t)}_j),
    $$
    for $j\in \mathcal{M}^{(t)}$ and $q^{(t+1)}_j=\tilde q^{(t)}_j$ otherwise. Then the expected squared error satisfies
    \begin{align*}
        &\mathbb{E}\|q^{(t+1)}-y\|\leq\mathbb{E}\|\bar{q}^{(t)}-y\|-\eta_+\cdot \mathrm{Gap}^{(t)}_++\eta^2\bar{B}^2,
    \end{align*}
    where
    $$
    \mathrm{Gap}^{(t)}_+=2\mathbb{E}\sum_{j\in\mathcal{M}^{(t)}}m_j^{+,(t)}(1-\bar{q}_j^{(t)})(y_j-\bar{q}_j^{(t)}
    $$
    and $\bar B$ is a constant depending on the transport budget and sparsity. When the affinity weights align positive anchors with high--$y$ targets (guaranteed by Assumption~\ref{ass:geo_anchor}),
    $\mathrm{Gap}_+^{(t)} > 0$ and the the positive transport step strictly reduces the surrogate squared error, up to a second--order term controlled by $\eta^2_+\bar{B}^2$.

\end{theorem}

\section{Proof of the theoretical results}

\subsection{Proof of lemma}
\begin{proof}[Proof of Lemma~\ref{lemma:trust:logit}]

The logit map $\psi(p) = \log(p/(1-p))$ has derivative
$\psi'(p) = 1/(p(1-p))$. On $[\epsilon_p, 1-\epsilon_p]$,
$\psi'(p)$ attains its maximum at the endpoints, so
$\sup\psi'(p) = 1/(\epsilon_p(1-\epsilon_p))$. By the mean value theorem (MVT),
\[
\Big|\log\frac{p_i^{(t+1)}}{1-p_i^{(t+1)}} - \log\frac{p_i^{(t)}}{1-p_i^{(t)}}\Big|
=|\psi'(p)|\cdot|p_i^{(t+1)}-p_i^{(t)}|\leq \frac{1}{\epsilon_p(1-\epsilon_p)} |\Delta p|.
\]
The full expression for the trust logit $u_i^{(t)}$ becomes
\[
u_i^{(t)}=\lambda_p\psi(p_i^{(t)})+2\lambda_{\text{loc}}\Delta r_i^{\lambda_{\text{loc}}}+2\lambda_{\text{anc}}\Delta r_i^{\lambda_{\text{anc}}}+2\lambda_{m}\Delta m_i
\]
Applying the triangle inequality and substituting the MVT bound, we get
\begin{align*}
    |u_i^{(t+1)}-u^{(t)}_i|&\leq \lambda_{p}|\psi(p_i^{(t+1)}-\psi(p_i^{(t)})|+2\lambda_{\text{loc}}|\Delta r_i^{\lambda_{\text{loc}}}|+2\lambda_{\text{anc}}|\Delta r_i^{\lambda_{\text{anc}}}|+2\lambda_{m}|\Delta m_i|\\
    &\leq \frac{\lambda_p}{\epsilon_p(1-\epsilon_p)}|\Delta p_i|+2\lambda_{\text{loc}}|\Delta r_i^{\lambda_{\text{loc}}}|+2\lambda_{\text{anc}}|\Delta r_i^{\lambda_{\text{anc}}}|+2\lambda_{m}|\Delta m_i|.
\end{align*}
This completes the proof for component-wise Lipschitz continuity of the trust logit.
\end{proof}

\begin{proof}[Proof of Lemma~\ref{lemma:suff-decrease}]
    By Assumption~\ref{ass:loss-reg} and the descent lemma,
\[
\lenc^{(t)}(\Theta^{(t+1)})
\leq
\lenc^{(t)}(\Theta^{(t)})
+ \langle \nabla_\Theta\lenc^{(t)}(\Theta^{(t)}), \Theta^{(t+1)} - \Theta^{(t)}\rangle
+ \tfrac{L_\Theta}{2}\|\Theta^{(t+1)} - \Theta^{(t)}\|^2.
\]
With $\Theta^{(t+1)} = \Theta^{(t)} - \eta g$ and
$g := \nabla_\Theta\lenc^{(t)}(\Theta^{(t)})$,
\[
\Theta^{(t+1)} - \Theta^{(t)} = -\eta g,
\quad
\|\Theta^{(t+1)} - \Theta^{(t)}\|^2 = \eta^2\|g\|^2.
\]
Substituting,
\[
\lenc^{(t)}(\Theta^{(t+1)})
\leq
\lenc^{(t)}(\Theta^{(t)}) - \eta\|g\|^2 + \tfrac{L_\Theta\eta^2}{2}\|g\|^2
= \lenc^{(t)}(\Theta^{(t)}) - \eta\big(1 - \tfrac{L_\Theta\eta}{2}\big)\|g\|^2.
\]
For $)<\eta \leq 2/L_\Theta$, the bracket is nonnegative, hence the result follows with $c_{\Theta}=\eta \left(1-\frac{L_{\Theta}\eta}{2} \right)$; for $\eta = 1/L_\Theta$,
the bracket equals $1/2$, giving $c_\Theta = 1/(2L_\Theta)$.
\end{proof}

\begin{proof}[Proof of Lemma~\ref{lem:lyapunov}]
   For $c_\Theta$, in order to fulfill the requirement of $L$--smooth, we define that for any $\theta^{(t)},\theta^{\prime}=\theta^{(t+1)}=\theta^{(t)}-\eta g$ where $g=\nabla f(\theta^{(t)})$,
   \begin{align*}
       f(\theta^{(t+1)})&\leq f(\theta^{(t)})+\langle g,\theta^{(t+1)}-\theta^{(t)}\rangle+\frac{L}{2}\|\theta^{(t+1)}-\theta^{(t)}\|^2\\
       &\leq f(\theta^{(t)})-\eta\|g\|^2+\frac{L\eta^2}{2}\|g\|^2=f(\theta^{(t)})-\eta\left(1-\frac{L\eta}{2}\right)\|g\|^2
   \end{align*}
   In order to let the gradient remain negative, stepsize $\eta$ should satisfy $1-L\eta/2\geq0 \Longleftrightarrow\eta<2/L$. We set $\eta=1/L$, the equation above can be derived as
   \begin{equation}
       f(\theta^{(t+1)})\leq f(\theta^{(t)}-\frac{1}{2L}\|g\|^2
   \end{equation}
   Thus, $c_\Theta=1/(2L_\Theta).~\qed$.

   Given $\tilde q_i^{(t)}=\sigma(u_i^{(t)})$, we obtain $\bar q_i^{(t)}$ after applying hard constraints. According to Assumption~\ref{ass:geo_anchor},
   \begin{equation*}
       \|\bar q^{(t+1)}-\bar q^{(t)}\|\leq L_{\mathcal{D},m}\|m^{(t+1)}-m^{(t)}\|
   \end{equation*}
   Given the rule of transport update,$q^{(t,+)}_j=\bar q^{(t)}_j+\eta_+m_j^{+,(t)}(1-\bar q_j^{(t)})$, we get 
   \begin{align*}
       q^{(t,+)}_j-\bar q^{(t)}_j&=\eta_+m_j^{+,(t)}(1-\bar q_j^{(t)})\\
       \Rightarrow~|q^{(t,+)}_j-\bar q^{(t)}_j|&\leq \eta_+m_j^{+,(t)}.\qquad(\because \bar q_j^{(t)}\in[0, 1]~\mathrm{and}~m_j^{+,(t)}\in[0, 1]) \\
       \Rightarrow~\|q^{(t,+)}_j-\bar q^{(t)}_j\|&\leq \eta_+\|m^{+,(t)}\|.
   \end{align*}
   Therefore,
   \begin{align*}
       \|q^{(t,+)}-q^{(t)}\|&\leq \|q^{(t+1)}-\bar q^{(t)}\|+\|\bar q^{(t)}-q^{(t)}\|\\
       &\leq L_{\mathcal{D},m}\|m^{(t)}-m^{(t-1)}\|+L_{\mathcal{D},m}\|p^{(t)}-p^{(t-1)}\|.
   \end{align*}
   where $L_{\mathcal{D},m}$ is a Lipschitz constant for $\mathcal{D}_{\text{trust}}$ vs. $p$.

   Among Lyapunov function $V^{(t)}$, the trust and transport part is 
   \begin{align*}
       V_q^{(t)}:=\lambda_q\|q^{(t)}-q^{(t-1)}\|\Rightarrow V_q^{(t+1)}-V_q^{(t)}=\lambda_q(\|q^{(t+1)}-q^{(t)}\|-\|q^{(t)}-q^{(t-1)}\|)
   \end{align*}
   In order to ensure $-c_q\|q^{(t+1)}-q^{(t)}\|$ provides the negative gradient, we use Young inequality to derive the upper bound for
   \begin{equation*}
       L_q\|q^{(t+1)}-q^{(t)}\|\leq \frac{L_q^2}{2\epsilon_q}+\frac{\epsilon_q}{2}L_q\|q^{(t+1)}-q^{(t)}\|
   \end{equation*} 
   With proper $\epsilon_q=\{\epsilon_q\mid\epsilon_q/2<\lambda_q\}$,
   \begin{equation*}
       \lambda_q\|q^{(t+1)}-q^{(t)}\|-\frac{\epsilon_q}{2}\|q^{(t+1)}-q^{(t)}\|=(\lambda_q-\frac{\epsilon_q}{2})\|q^{(t+1)}-q^{(t)}\|
   \end{equation*}
   Therefore,
   \begin{equation}
       c_q:=\lambda_q-\frac{\epsilon_q}{2}>0.~\qed
   \end{equation}

   \textbf{Transport mass.} Given the definition of transport mass and transport inflow and the mass conservation statement in Lemma~\ref{lem:mass-cons}, we know the magnitude of $m^{+,(t)}_j$ is controlled by the transport budget $B_t$. The mass will be transported to next inflow state, $m^{(t+1)}_j=m^{+,(t)}_j$, thus
   \begin{align*}
       &m^{(t+1)}-m^{(t)}=m^{+,(t)}-m^{+,(t-1)}\\
       \Rightarrow~&\|m^{(t+1)}-m^{(t)}\|\leq B_t-B_{t+1}.
   \end{align*}
   Back to Lyapunov function that 
   \begin{equation}
       V^{(t+1)}_m-V^{(t)}_m=\lambda_m\|m^{(t+1)}-m^{(t)}\|-\lambda_m\|m^{(t)}-m^{(t-1)}\|,
   \end{equation}
   similar to $c_q$ derivation, via Young inequality, we can get $c_m:=\lambda_m-\epsilon_m/2>0~\qed$.
\end{proof}

\begin{proof}[Proof of Lemma~\ref{lem:mass-cons}]
The mass identity follows by direct computation using the definition
in Section~\ref{app:transport_details}, the positive transport inflow to an uncertain target $j\in\mathcal{M}^{(t)}$ is given by:
\[
m^{+,(t)}_j=B^+_t\sum_{i\in\mathcal{B}^{(t)}_+}\pi_i^{+,(t)}\Gamma_{ij}^{+,(t)}.
\]
To verify mass conservation, we compute the $L_1$ norm of the inflow vector $m^{+,(t)}$ that 
\[
\|m^{+,(t)}\|_1=\sum_{j\in\mathcal{M}^{(t)}}m^{+,(t)}_j=B^+_t\sum_{i\in\mathcal{B}^{(t)}_+}\pi_i^{+,(t)}\left(\sum_{j\in\mathcal{M}^{(t)}}\Gamma_{ij}^{+,(t)}\right)
\]

Since the total transported mass is bounded
by $B_t^+$ and $m^{+,(t)}$ is nonnegative, we can bound the sequence by choosing $C_B=\sup_t B^+_t$
\[
\|m^{+,(t)}\| \le \|m^{+,(t)}\|_1 \leq C_B,
\]
completing the proof.

\end{proof}

\subsection{Proof of theorem}

\begin{proof}[Proof of Theorem~\ref{thm:trust-continuity}]
    {By Section~\ref{app:encoder_details}, $\tilde{q}_i^{(t)} = \sigma(u_i^{(t)})$. The sigmoid
satisfies $\sigma'(u) = \sigma(u)(1-\sigma(u)) \leq 1/4$, with equality
iff $u = 0$. Combining with Lemma~\ref{lemma:trust:logit},
\[
|\tilde{q}_i^{(t+1)} - \tilde{q}_i^{(t)}|
= |\sigma(u_i^{(t+1)}) - \sigma(u_i^{(t)})|
\leq \tfrac{1}{4}|u_i^{(t+1)} - u_i^{(t)}|.
\]
Substituting the bound from Lemma~\ref{lemma:trust:logit} yields the
result. The single-input constants follow by isolating each term.
}
\end{proof}
\begin{proof}[Proof of Theorem~\ref{thm:convergence}]
\hfill\par\noindent
    \textbf{Proof of (a).} Define $\tilde V^{(t)}:= V^{(t)}-\sum_{s<t}\mathcal{E}^{(s)}$. By Lemma~\ref{lemma:suff-decrease}, $\tilde V^{(t)}$ is monotone non--increasing. Given Assumption~\ref{ass:loss-reg} that $V^{(t)}\geq 0$, non--negativity of squared norms, and $\sum_s\mathcal{E}^{(s)}<\infty$, $\tilde V^{(t)}$ thus has a lower bounded. By Monotone Convergence Theorem~\citep{robbins1971convergence}, $\tilde V^{(t)} \to \tilde V^\star$, hence $V^{(t)} \to V^\star :=
\tilde V^\star + \sum_s \mathcal{E}^{(s)}$.

\textbf{Proof of (b).}
Telescoping~\eqref{eq:lyap-dec} from $t=0$ to $T-1$,
\[
\sum_{t=0}^{T-1}\big[c_\Theta\|\nabla_\Theta\lenc^{(t+1)}(\Theta^{(t)})\|
+ c_q\|q^{(t+1)} - q^{(t)}\|
+ c_m\|m^{(t+1)} - m^{(t)}\|\big]
\leq V^{(0)} - V^{(T)} + \sum_{t=0}^{T-1}\mathcal{E}^{(t)}.
\]
Letting $T \to \infty$, the right-hand side converges by (a). Hence
each of the three series on the left converges, which forces each
summand to vanish:
\[
\|\nabla_\Theta\lenc^{(t)}(\Theta^{(t-1)})\| \to 0,
\quad
\|q^{(t+1)} - q^{(t)}\| \to 0,
\quad
\|m^{(t+1)} - m^{(t)}\| \to 0.
\]

\textbf{Proof of (c.)} Let $(\Theta^{(t_k)}, q^{(t_k)}, a^{(t_k)}, m^{(t_k)})$ be a convergent subsequence with limit $(\Theta^{\star}, q^{\star}, a^{\star}, m^{\star})$.

By Theorem~\ref{thm:convergence}(b), $\left(\mathcal{L}_{\mathrm{enc}}^{(t)}(\Theta^{(t)})-\inf \mathcal{L}_{\mathrm{enc}}^{(t)}(\Theta^{(t)})\right)\to0$. The loss $\mathcal{L}_{\mathrm{enc}}^{(t)}(\Theta^{(t)})$ depends continuously on $(q^{(t)},a^{(t)},V^+,V^-)$. Passing to the subsequential limit, $\mathcal{L}_{\mathrm{enc}}^{(t)}(\Theta^{\star})-\inf\mathcal{L}_{\mathrm{enc}}^{(t)}(\Theta^{\star})=0$. Hence, $\Theta^\star\in\arg\min\mathcal{L}_{\mathrm{enc}}^{\star}$.

The outer loop defines that $q^{(t+1)} = \TT(\DT(p^{(t)}, z^{(t)}, m^{(t)}))$.
By Theorem~\ref{thm:convergence}(b), $\|q^{(t_k+1)} - q^{(t_k)}\| \to 0$,
so $q^{(t_k+1)} \to q^\star$. By continuity of $\DT$
(Theorem~\ref{thm:trust-continuity}) and $\TT$ (Lemma~\ref{lem:mass-cons}),
$\TT(\DT(p^{(t_k)}, z^{(t_k)}, m^{(t_k)})) \to \TT(\DT(p^\star, z^\star, m^\star))$.
Equating both limits,
\begin{equation*}
    q^\star = \TT(\DT(p^\star, z^\star, m^\star)).
\end{equation*}
\end{proof}

\begin{proof}[Proof of Theorem~\ref{thm:bias-reduction}]

Under the update rule, $q^{(t, +)}_j=\tilde{q}_j^{(t)}+\eta_+m_j^{+,(t)}(1-\tilde{q}_j^{(t)})$,  the MSE for only consider pure positive transport, $\eta_-=0$,
\begin{align*}
    (q_j^{(t+1)}-y_j)^2-(\bar{q}_j^{(t)}-y_j)^2&=(q_j^{(t+1)}-\bar{q}_j^{(t)})    (q_j^{(t+1)}-\bar{q}^{(t)}_j-2y_j)\\
    &=\eta_+m_j^{+,(t)}(1-\bar{q}_j^{(t)})(q_j^{(t+1)}+\bar{q}_j^{(t)}-2y_j)\\
    &=2\eta_+m_j^{+,(t)}(1-\bar{q}_j^{(t)})(q_j^{(t)}-y_j)+\eta^2_+[m_j^{+,(t)}(1-\bar{q}^{(t)}_j)]^2
\end{align*}
Therefore,
\begin{align*}
    &\mathbb{E}\|q^{(t+1)}-y\|-\mathbb{E}\|\bar{q}^{(t)}-y\|\\=&-2\eta_+\mathbb{E}\sum_j\left [m_j^{+,(t)}(1-\bar{q}_j^{(t)})(y_j-q_j^{(t)})\right]+\eta^2_+\mathbb{E}\left[[m_j^{+,(t)}(1-\bar{q}^{(t)}_j)^2\right]\\
    &=-\eta_+\cdot\mathrm{Gap}^{(t)}_++\eta^2_+\mathbb{E}\left[[m_j^{+,(t)}(1-\bar{q}^{(t)}_j)^2\right],
    \end{align*}
    The second term is bounded by $\eta_+^2 \bar{B}^2$ via Lemma~\ref{lem:mass-cons} and $|m_j^{+,(t)}(1-\bar q_j^{(t)})| \leq m_j^{+,(t)}$ where $\bar{B}$ is a constant,
    \begin{equation*}
        \mathbb{E}\|q^{(t+1)}-y\|-\mathbb{E}\|\bar{q}^{(t)}-y\|\leq-\eta_+\cdot\mathrm{Gap}^{(t)}_++\eta^2_+\bar{B}^2.
    \end{equation*}
    The result follows.
\end{proof}

\newpage

\section{Additional experimental details}
\label{app:experiment_details}

\subsection{Evaluation metrics}
\label{sec:exp_metrics}

We use the same set of metrics for simulation and real-data evaluation whenever applicable. 
The \emph{original accuracy} measures the agreement between the initial judge-induced consistency signal and the ground-truth consistency label. In simulation, this label is generated by the simulator; in real data, it is derived from human preference annotations. The \emph{adjusted accuracy} measures the agreement after applying \PUAuditPlus.

We define the accuracy difference as
\[
\Delta_{\mathrm{acc}}
=
\mathrm{Acc}_{\mathrm{adj}}
-
\mathrm{Acc}_{\mathrm{orig}}.
\]
A positive value indicates that the refinement improves agreement with the target human-consistency label. We also report the \emph{flip rate}, defined as the fraction of examples whose hard assignment changes after refinement:
\[
\mathrm{Flip}
=
\frac{1}{n}
\sum_{i=1}^{n}
\mathbf{1}\{\widehat y_i^{\mathrm{adj}}\neq \widehat y_i^{\mathrm{orig}}\}.
\]
This metric measures how actively the method changes the initial signal. A useful refinement method should not only flip many labels, but should flip them in a way that improves adjusted accuracy.

For real-data experiments, we additionally report the final waiting-pool size, runtime, and peak memory usage. The waiting-pool size records the number of examples selected for verification but not yet incorporated into the refinement state at termination. Runtime and memory are used to quantify the computational cost of applying \PUAuditPlus to real LLM-as-a-judge data.

\subsection{Simulation details}
\label{app:simulation_details}

We construct a generative simulation that mirrors the refinement problem studied in Sections~\ref{sec:problem_formulation}--\ref{sec:method}. Specifically, the simulator generates examples from two latent consistency groups with a adjustable separation, and then produces noisy comparison features and corrupted initial consistency signals. \PUAuditPlus only observes these noisy features and corrupted signals. This setup allows us to evaluate whether the proposed trust and transport updates can recover the underlying consistency structure from the information available to the refinement procedure.

For each simulated example, we first sample a binary consistency label \(y_i\in\{0,1\}\), following the definition in Section~\ref{sec:problem_formulation}. We then generate a simulator-side latent coordinate \(\bm{h}_i\in\mathbb{R}^{d}\), which is used only for data generation and is not observed by \PUAuditPlus. Conditional on \(y_i=k\), we sample
\[
\bm{h}_i \mid y_i=k \sim \mathcal{N}_d(\bm{\mu}_k,\bm{\Sigma}),
\qquad k\in\{0,1\},
\]
where \(\bm{\mu}_k\in\mathbb{R}^{d}\) is the group mean vector for label \(k\), and \(\bm{\Sigma}\in\mathbb{R}^{d\times d}\) is the covariance matrix. The distance between \(\bm{\mu}_0\) and \(\bm{\mu}_1\) controls how separable the human-consistent and human-inconsistent groups are.

The observed comparison feature is generated from this oracle coordinate as
\[
x_i = A h_i + \epsilon_i,
\qquad
\epsilon_i\sim \mathcal{N}(0,\sigma_x^2 I),
\]
where \(A\in\mathbb{R}^{d_x\times d}\) is a linear transformation matrix, \(\epsilon_i\in\mathbb{R}^{d_x}\) is Gaussian observation noise, \(\sigma_x>0\) controls the noise level, and \(I_{d_x}\) is the \(d_x\)-dimensional identity matrix. Thus, \(x_i\) plays the same role as the fixed comparison feature in the real pipeline: it contains useful geometric information about consistency, but the conditional distributions of \(x_i\mid y_i=1\) and \(x_i\mid y_i=0\) remain noisy and overlapping rather than perfectly separable.

We then corrupt \(y_i\) to obtain an initial noisy consistency signal \(\tilde y_i\in\{0,1\}\), which plays the role of the raw LLM-induced signal before refinement. We consider two corruption mechanisms, following the common distinction between class-conditional and instance-dependent label-noise models~\citep{xia2020part,berthon2021confidence,cheng2022instance}. In the \emph{class-dependent} (CD) setting, the corruption probability depends only on the true consistency label:
\[
\mathbb{P}(\tilde y_i\neq y_i\mid y_i,x_i)
=
\mathbb{P}(\tilde y_i\neq y_i\mid y_i).
\]
Thus, all examples with the same value of \(y_i\) share the same corruption rate, regardless of their location in the feature space. This provides a controlled baseline for testing whether \PUAuditPlus can correct label-level corruption.

In the \emph{distribution-dependent} (DD) setting, the corruption probability is allowed to depend on the sample location:
\[
\mathbb{P}(\tilde y_i\neq y_i\mid y_i,x_i).
\]
This setting is more complex because the initial signal is more likely to be wrong for examples in uncertain or overlapping regions of the feature space. It better reflects LLM-as-a-judge behavior, where errors are often concentrated on difficult, subjective, or geometrically ambiguous comparisons rather than being uniformly distributed within each class.

Finally, we construct the verified sets \(\mathcal{V}^{+}\) and \(\mathcal{V}^{-}\) by revealing only a subset of labels to the refinement procedure. We consider two verification mechanisms. The first is \emph{selected completely at random} (SCAR), under which the probability that an example is human-verified is constant within the relevant latent label class and does not depend on its features. Thus, conditional on the latent label status, the verified examples form an unbiased random subset. The second is \emph{selected at random} (SAR), under which the verification probability is allowed to depend on the feature vector \(x_i\). Under SAR, the verified set may therefore be systematically biased toward particular regions of the feature space, such as examples with higher uncertainty, greater difficulty, or more salient content. This mechanism better reflects selective human review, where verified examples are rarely sampled uniformly at random.

For the simulation setting, unless otherwise specified, we set the latent dimension to \(d=2\) and the class-centre distance to \(\|\mu_1-\mu_0\|=4\). Each noise--verification configuration is repeated over 5 independent runs. We report four metrics: original accuracy, the accuracy of the corrupted initial signal before refinement; adjusted accuracy, the accuracy after applying \PUAuditPlus; accuracy difference, the adjusted accuracy minus the original accuracy; and flip rate, the fraction of examples whose hard assignment changes after refinement.

\subsection{Simulation sensitivity to latent geometry}
\label{app:simulation_geometry}
We further examine how the latent geometry affects refinement under the DD+SAR setting. Our goal is to assess the sensitivity of \PUAuditPlus to two geometric factors: the class-center distance \(\|\bm{\mu}_1-\bm{\mu}_0\|\), which controls group separability, and the latent dimension \(d\), which affects the reliability of local-neighbourhood and distance-based evidence. We first vary the class-center distance while fixing the latent dimension. We consider \(d=2\), \(d=5\), and \(d=20\), and vary \(\|\bm{\mu}_1-\bm{\mu}_0\|\) from \(1\) to \(10\). Across all three dimensions, small distances make the two consistency groups highly overlapping, so the observed geometry provides limited information for refinement. As the distance increases, the adjusted accuracy improves and then stabilizes. These results suggest that \PUAuditPlus benefits from separable but still noisy latent structure. The detailed results for the distance analyses are shown in Figures~\ref{fig:app_sim_distance_sweep}, \ref{fig:app_sim_distance_sweep_dim5}, and~\ref{fig:app_sim_distance_sweep_dim20}.

We then vary the latent dimension while fixing the class-center distance. We consider two fixed distances, \(\|\bm{\mu}_1-\bm{\mu}_0\|=2\) and \(\|\bm{\mu}_1-\bm{\mu}_0\|=4\), and vary \(d\) from \(2\) to \(100\). When the distance is relatively large, \(\|\bm{\mu}_1-\bm{\mu}_0\|=4\), \PUAuditPlus remains effective across dimensions, although the gain gradually decreases as \(d\) grows. When the distance is smaller, \(\|\bm{\mu}_1-\bm{\mu}_0\|=2\), the method is more sensitive to dimensionality: the adjusted accuracy falls below the original accuracy when \(d\) becomes large. This shows that high dimensionality is especially harmful when the two groups are not well separated. The detailed results for the dimension analyses are shown in Figures~\ref{fig:app_sim_dimension_sweep} and~\ref{fig:app_sim_dimension_sweep_dis2}. These analyses suggest that \PUAuditPlus performs best when the latent consistency groups are sufficiently separated and the feature dimension is not too large. In this regime, local neighbourhoods and anchor support provide informative evidence for the trust and transport updates. When the groups overlap heavily, the refinement signal becomes unreliable; when the dimension is high, distance-based neighbourhoods become less stable, weakening both transport and anchor evidence.


\subsection{Real-data details}
\label{app:real_data_details}

We evaluate \PUAuditPlus on real LLM-as-a-judge data constructed from MT-Bench and Chatbot Arena~\citep{zheng2023judging,chiang2024chatbot}. For the main objective-data experiments, we sample pairwise comparisons from three question types: coding, math reasoning, and structured factual questions. We additionally include one MT-Bench reference setting, where examples are sampled from MT-Bench without question-type stratification.

We consider five judge models, including three API-based models and two locally deployed open-weight models. The API judges are \textsc{GPT-5.4}, \textsc{GPT-5.4-mini}, and \textsc{Gemini-2.5-Flash}~\citep{gemini25flash}. The local judges are \textsc{Qwen2.5-7B-Instruct}~\citep{qwen25_technical_report} and \textsc{Mistral-7B-Instruct-v0.3}~\citep{jiang2023mistral7b}. For each judge, we use its preference prediction as the initial LLM-induced signal and evaluate the refined output against available human preference labels.

The main real data evaluation consists of \(5\) judge models \(\times\) \(3\) question types, together with the additional MT-Bench reference setting. Each configuration is repeated over five independent runs. Comparison features are extracted using \textsc{Skywork-Reward-Llama-3.1-8B-v0.2}~\citep{skywork_reward_2024}, which is kept fixed throughout refinement. We evaluate two encoder configurations, using latent dimensions \(16\) and \(128\), to examine whether the refinement behaviour is sensitive to representation capacity. We provide a representative hyperparameter configuration for a standard MT-Bench run in Table~\ref{tab:standard-mtbench-heavy-oldstyle-hparams}, corresponding to the actual setup used in that experiment.

We also compare against four lightweight baselines: logistic regression, MLP classifier, random forest~\citep{breiman2001random}, and label propagation~\citep{zhu2002labelprop}. These baselines use the same comparison features and verified examples, but do not use the full trust-state update or conservative transport mechanism of \PUAuditPlus.

\begin{table}[t]
\centering
\caption{Hyperparameters for the standard heavy old-style real-data runs. Values are taken from the selected best-combo configuration used for the 16-dataset, 5-seed experiment.}
\small
\setlength{\tabcolsep}{4pt}
\renewcommand{\arraystretch}{1.08}
\begin{tabular}{lll}
\toprule
Component & Hyperparameters & Values \\
\midrule
Encoder & backbone; loss & \texttt{skywork}; \texttt{cls\_geo\_share} \\
Encoder size & hidden dim; latent dim & 256; 16 \\
Optimization & batch size; lr; weight decay & 32; \(10^{-4}\); \(5\times 10^{-4}\) \\
Training schedule & max epochs; patience; warmup & 40; 8; 8 \\
Curriculum & alt. rounds; epochs/round; pseudo weight & 3; 2; 0.35 \\
Initialization & library frac.; seed frac.; min pos/neg & 0.20; 0.25; 24/16 \\
Risk estimation & Beta prior; max iter.; tol & \((1,19)\); 100; \(10^{-5}\) \\
Neighborhood risk & \(k\); kernel scale & 30; 3.0 \\
Trusted partition & threshold; budget scale; max frac. & 0.20; 0.20; 0.20 \\
Wash step & latent dim; kNN \(k\); posterior threshold & 32; 25; 0.55 \\
Wash risk gate & pre-wash risk threshold; gate alpha & 0.60; 1.5 \\
Transport & frac.; min/max; relabel threshold & 0.95; 0.40/0.95; 0.45 \\
Query selection & top-\(k\); Fisher \(L_2\); max iter. & 50; \(10^{-3}\); 200 \\
Waiting pool & EMA decay; min hits & 0.7; 2 \\
Stopping & max iter.; stable rounds & 24; 2 \\
Stopping thresholds & gap; change; size; \(\Delta\)gap & 0.35; 0.02; 0.015; 0.05 \\
\bottomrule
\end{tabular}
\label{tab:standard-mtbench-heavy-oldstyle-hparams}
\end{table}

\clearpage
\subsection{Additional simulation result}
\label{app:simulation_figures}

\begin{figure}[t]
    \centering
    \includegraphics[width=0.92\linewidth]{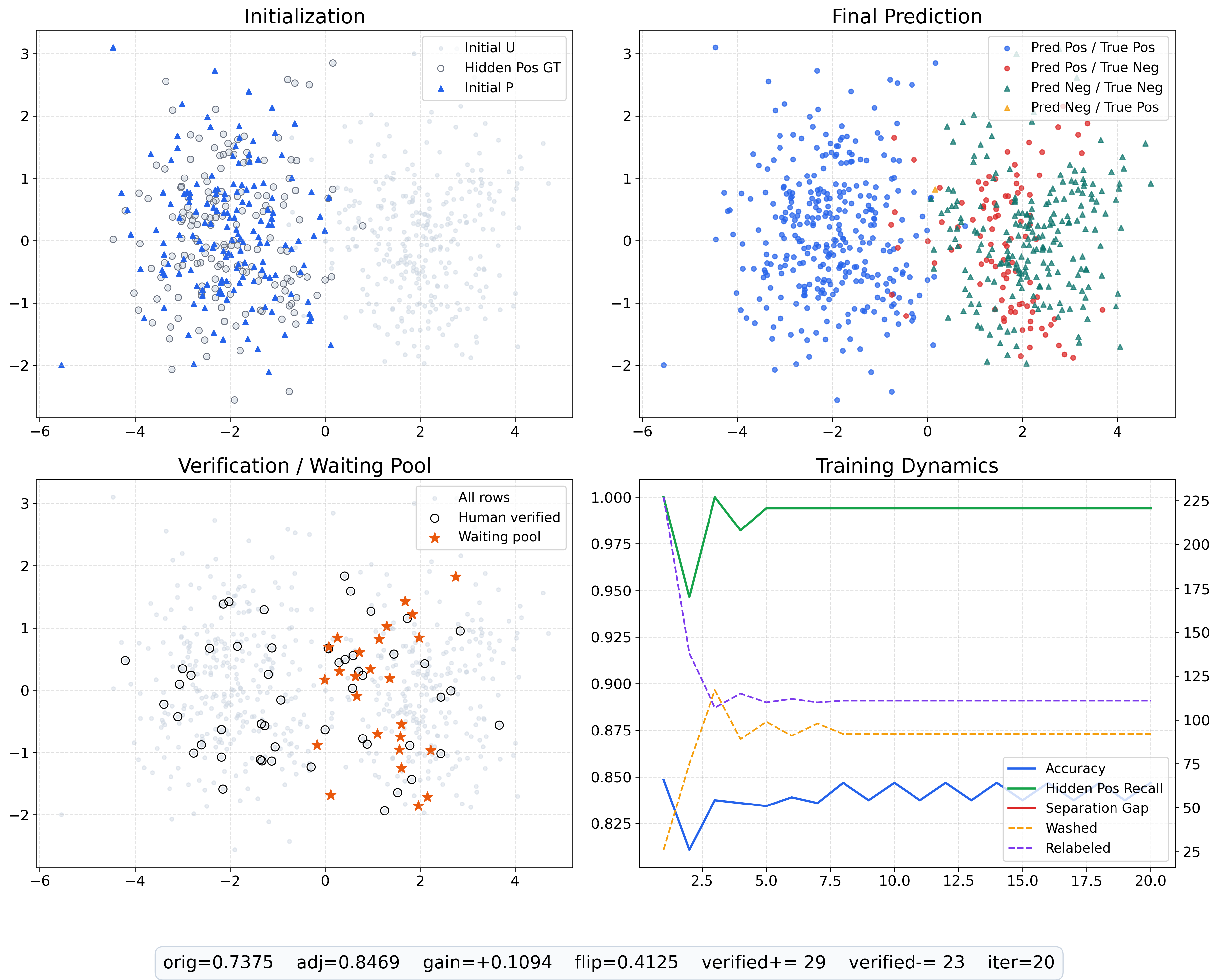}
    \caption{
    Simulation embedding result under class-dependent (CD) noise with SCAR verification.
    }
    \label{fig:app_sim_cd_scar_embedding}
\end{figure}

\begin{figure}[t]
    \centering
    \includegraphics[width=0.92\linewidth]{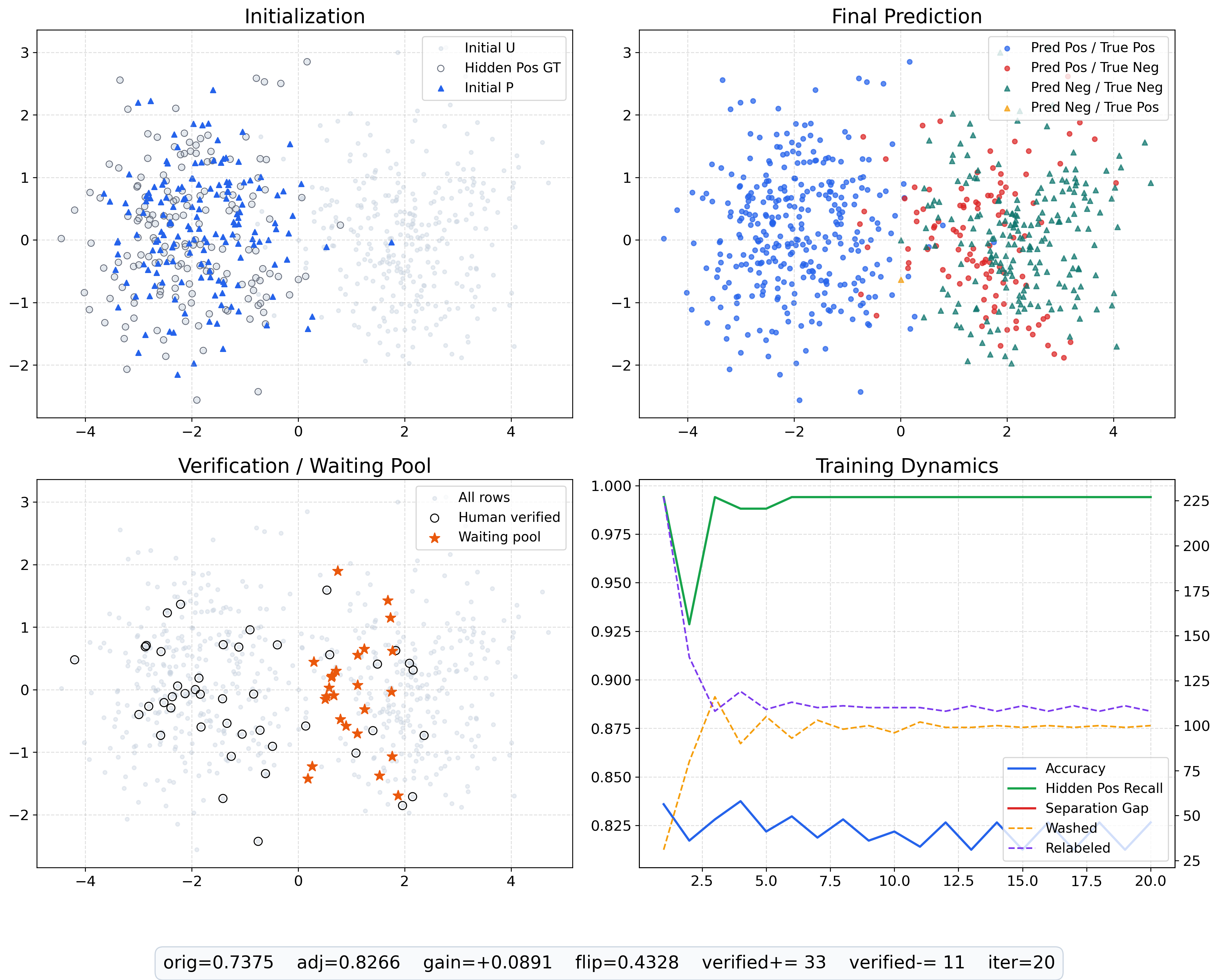}
    \caption{
    Simulation embedding result under class-dependent (CD) noise with SAR verification.
    }
    \label{fig:app_sim_cd_sar_embedding}
\end{figure}

\begin{figure}[t]
    \centering
    \includegraphics[width=0.92\linewidth]{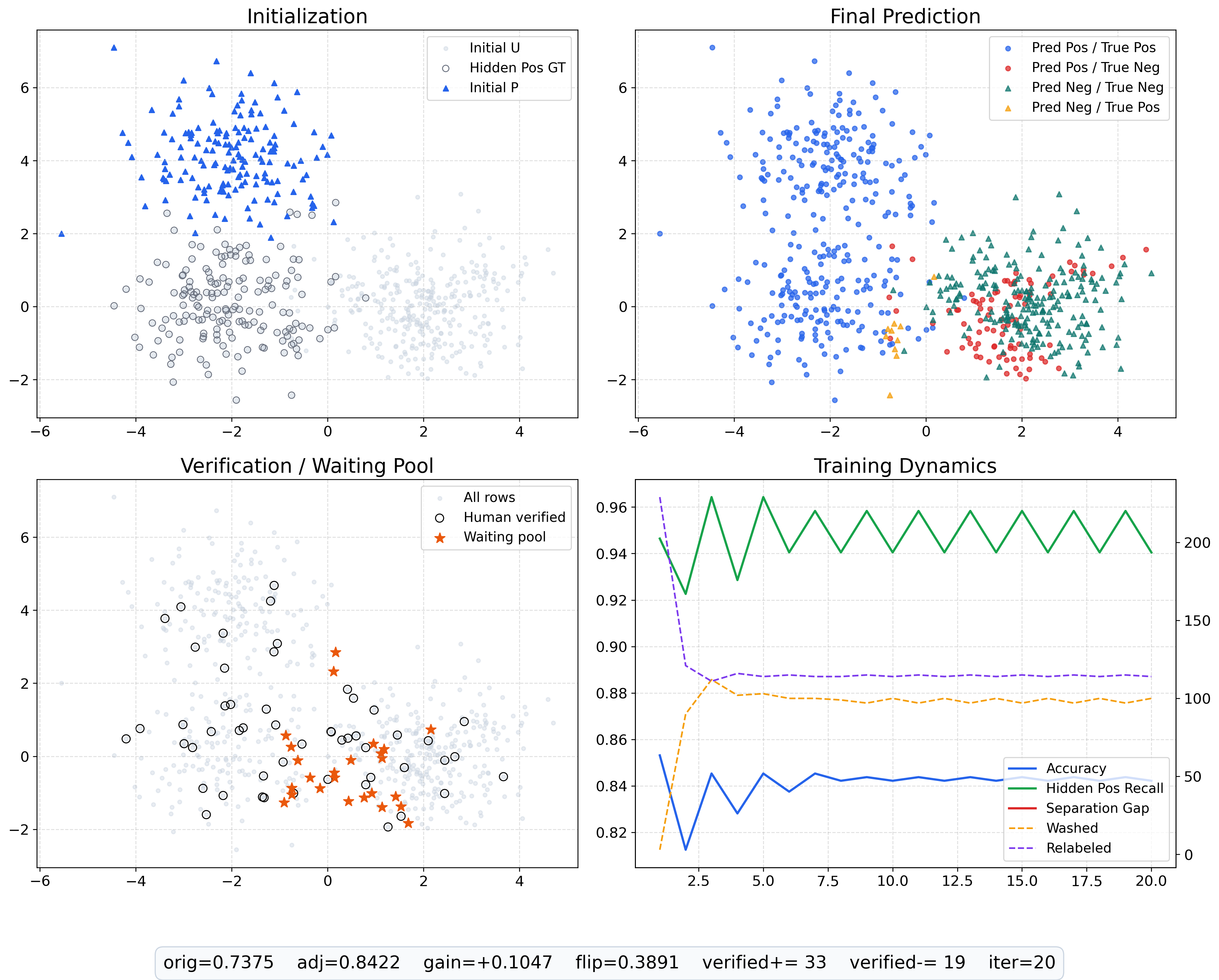}
    \caption{
    Simulation embedding result under distribution-dependent (DD) noise with SCAR verification.
    }
    \label{fig:app_sim_dd_scar_embedding}
\end{figure}

\begin{figure}[t]
    \centering
    \includegraphics[width=0.92\linewidth]{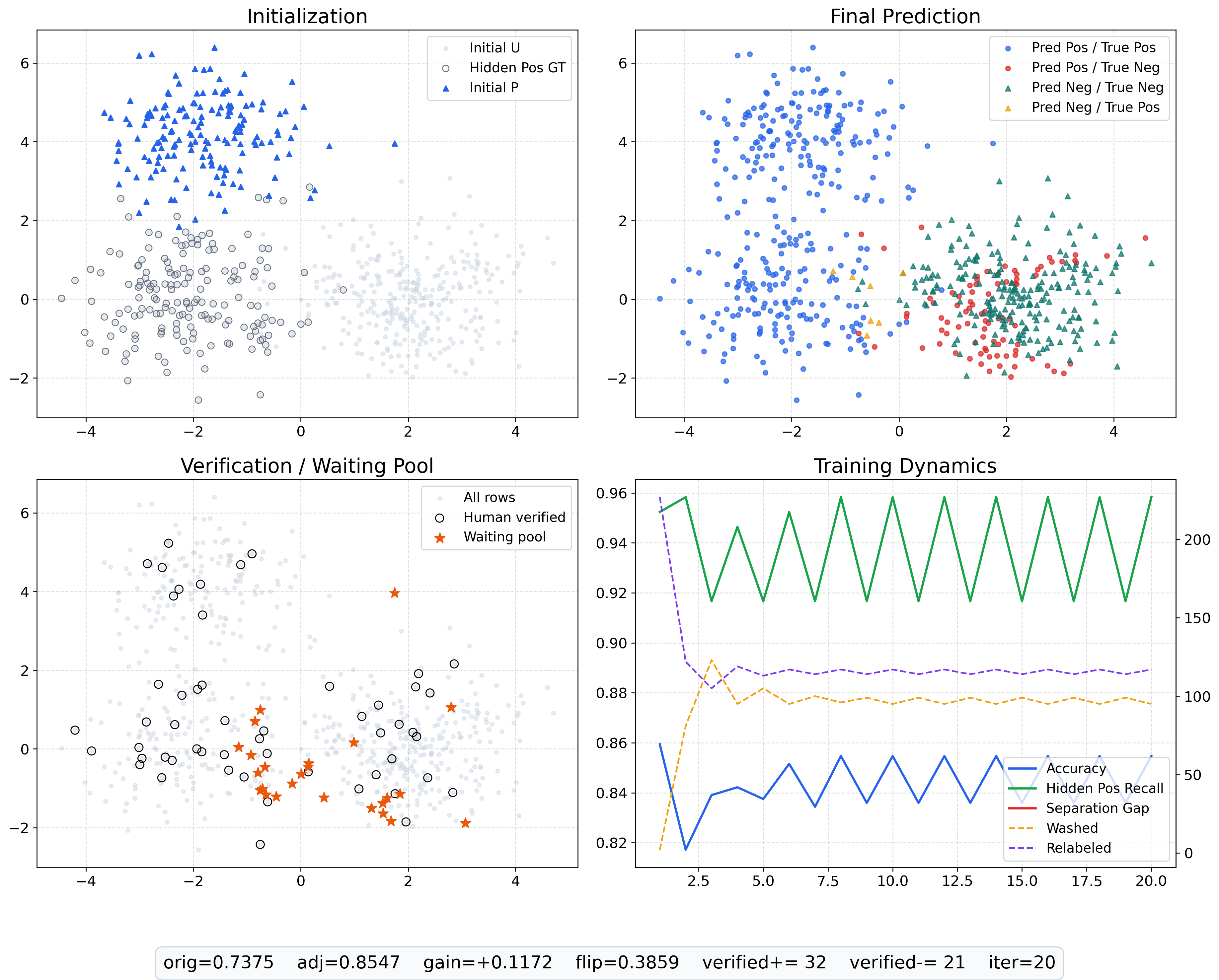}
    \caption{
    Simulation embedding result under distribution-dependent (DD) noise with SAR verification.
    }
    \label{fig:app_sim_dd_sar_embedding}
\end{figure}

\begin{figure}[t]
    \centering
    \includegraphics[width=1\linewidth]{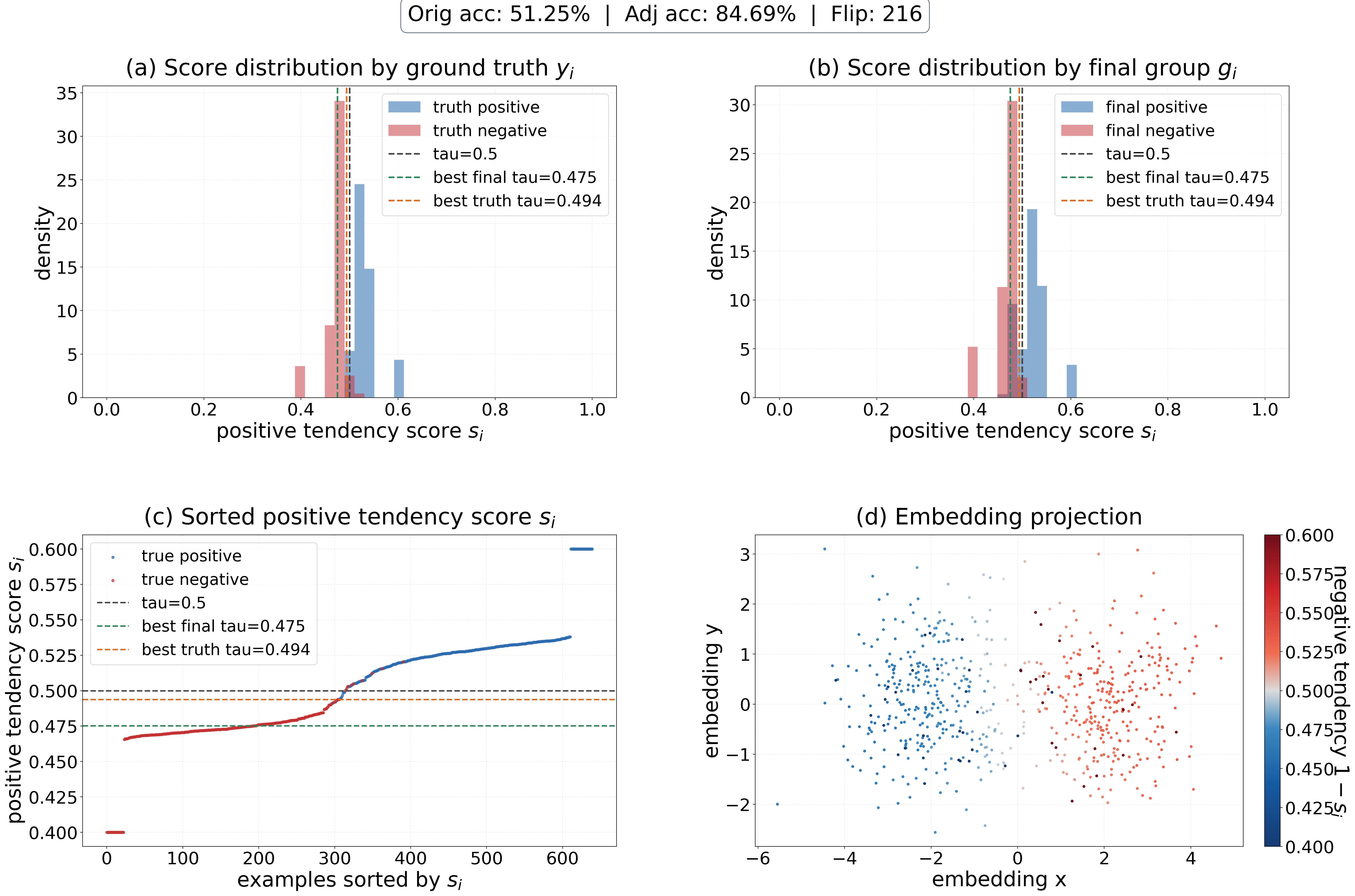}
    \caption{
    Threshold diagnostic under class-dependent (CD) noise with SCAR verification.
    Panels (a) and (b) compare score distributions by ground-truth consistency label \(y_i\) and final group assignment \(g_i\), respectively. Panel (c) sorts examples by the diagnostic score \(s_i=0.6(1-r_i)+0.4(1-d_i)\), with reference thresholds for the default rule, final groups, and ground truth. Panel (d) shows a two-dimensional projection of the reward-model embeddings colored by negative tendency \(1-s_i\). The summary box reports original accuracy, adjusted accuracy, and the number of flipped examples.
    }
    \label{fig:app_sim_cd_scar_threshold}
\end{figure}

\begin{figure}[t]
    \centering
    \includegraphics[width=1\linewidth]{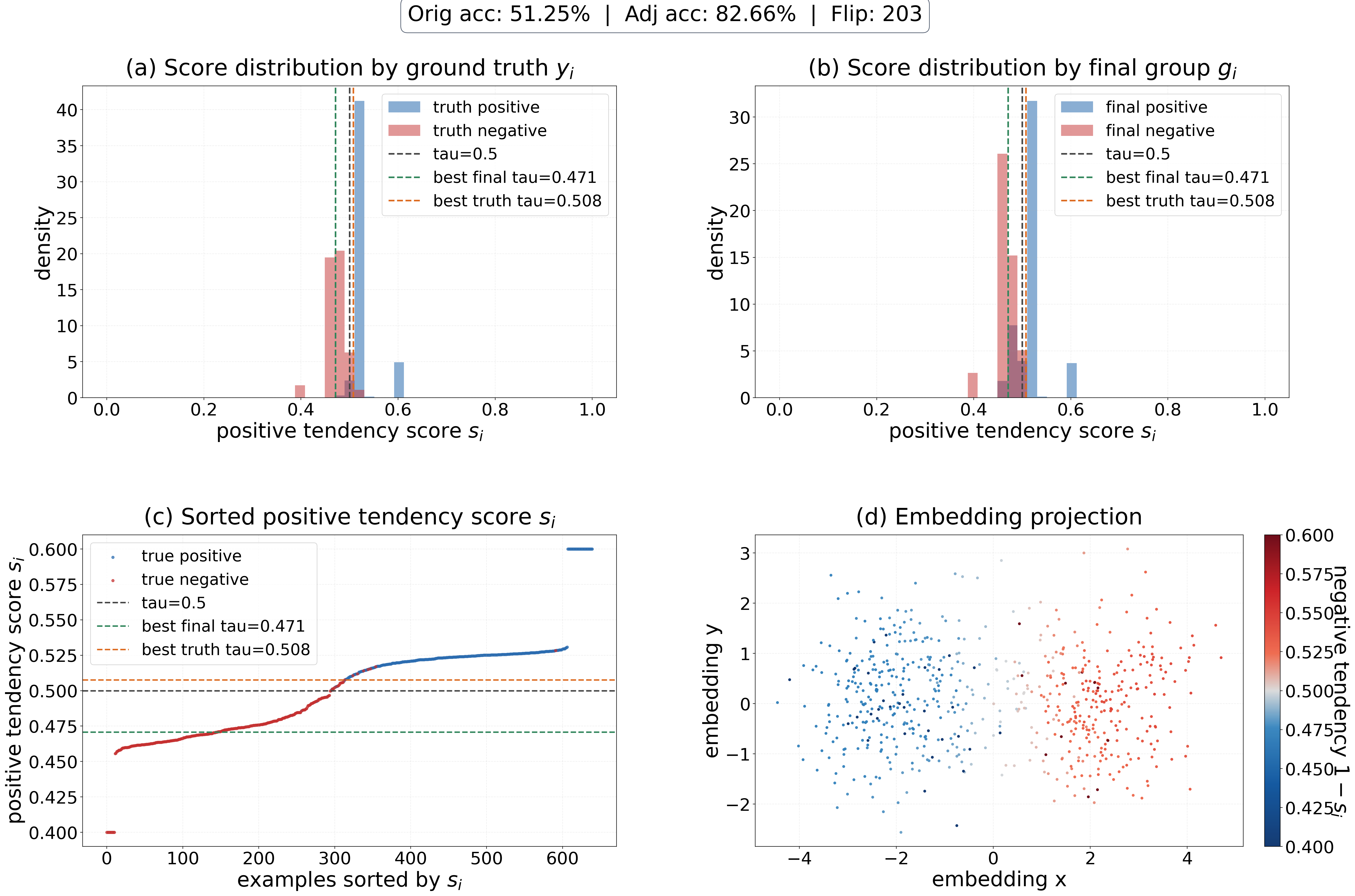}
    \caption{
    Threshold diagnostic under class-dependent (CD) noise with SAR verification. Details follow the same reporting as Figure~\ref{fig:app_sim_cd_scar_threshold}.
    }
    \label{fig:app_sim_cd_sar_threshold}
\end{figure}

\begin{figure}[t]
    \centering
    \includegraphics[width=1\linewidth]{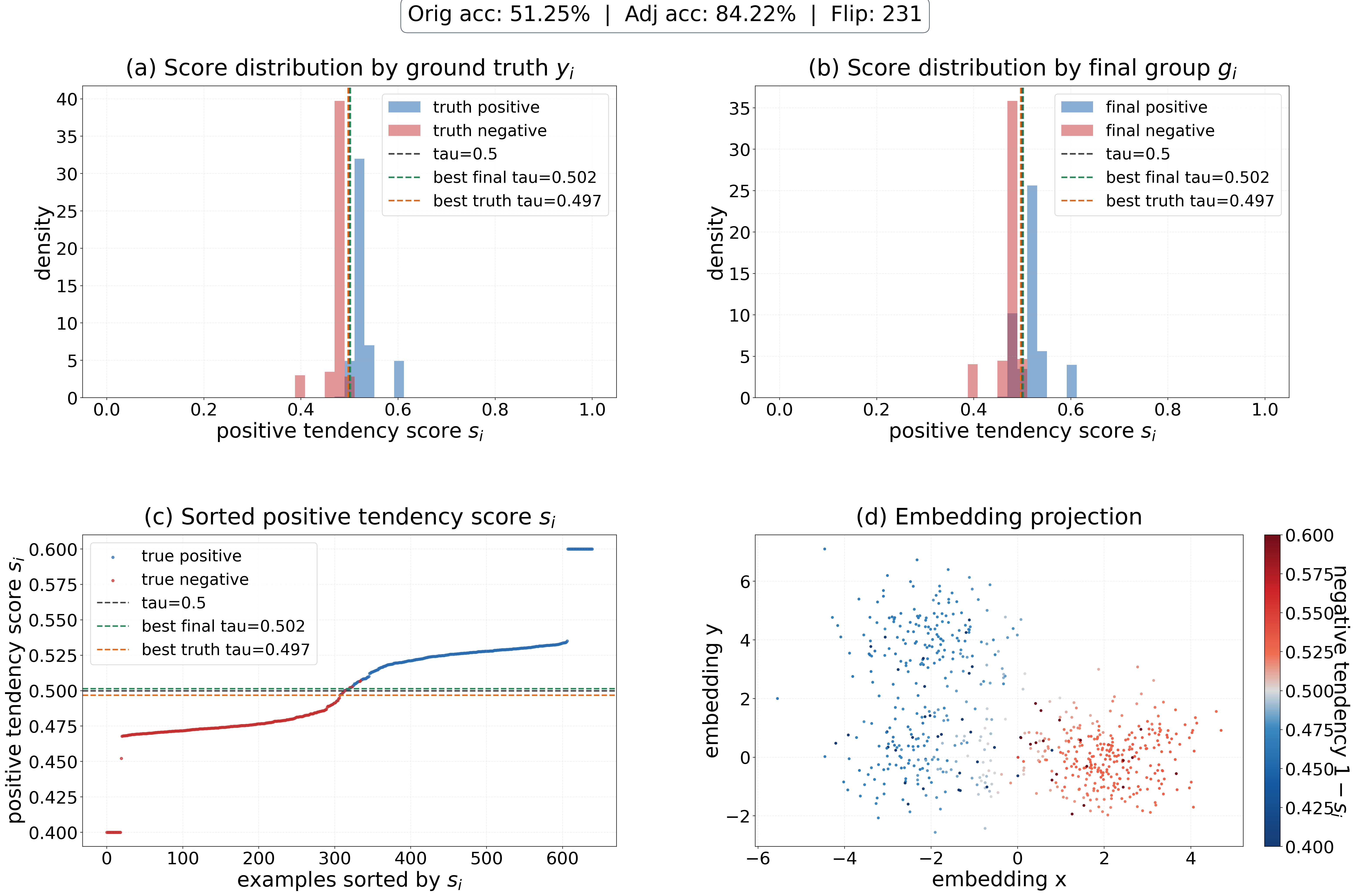}
    \caption{
    Threshold diagnostic under distribution-dependent (DD) noise with SCAR verification.
    Details follow the same reporting as Figure~\ref{fig:app_sim_cd_scar_threshold}.
    }
    \label{fig:app_sim_dd_scar_threshold}
\end{figure}

\begin{figure}[t]
    \centering
    \includegraphics[width=1\linewidth]{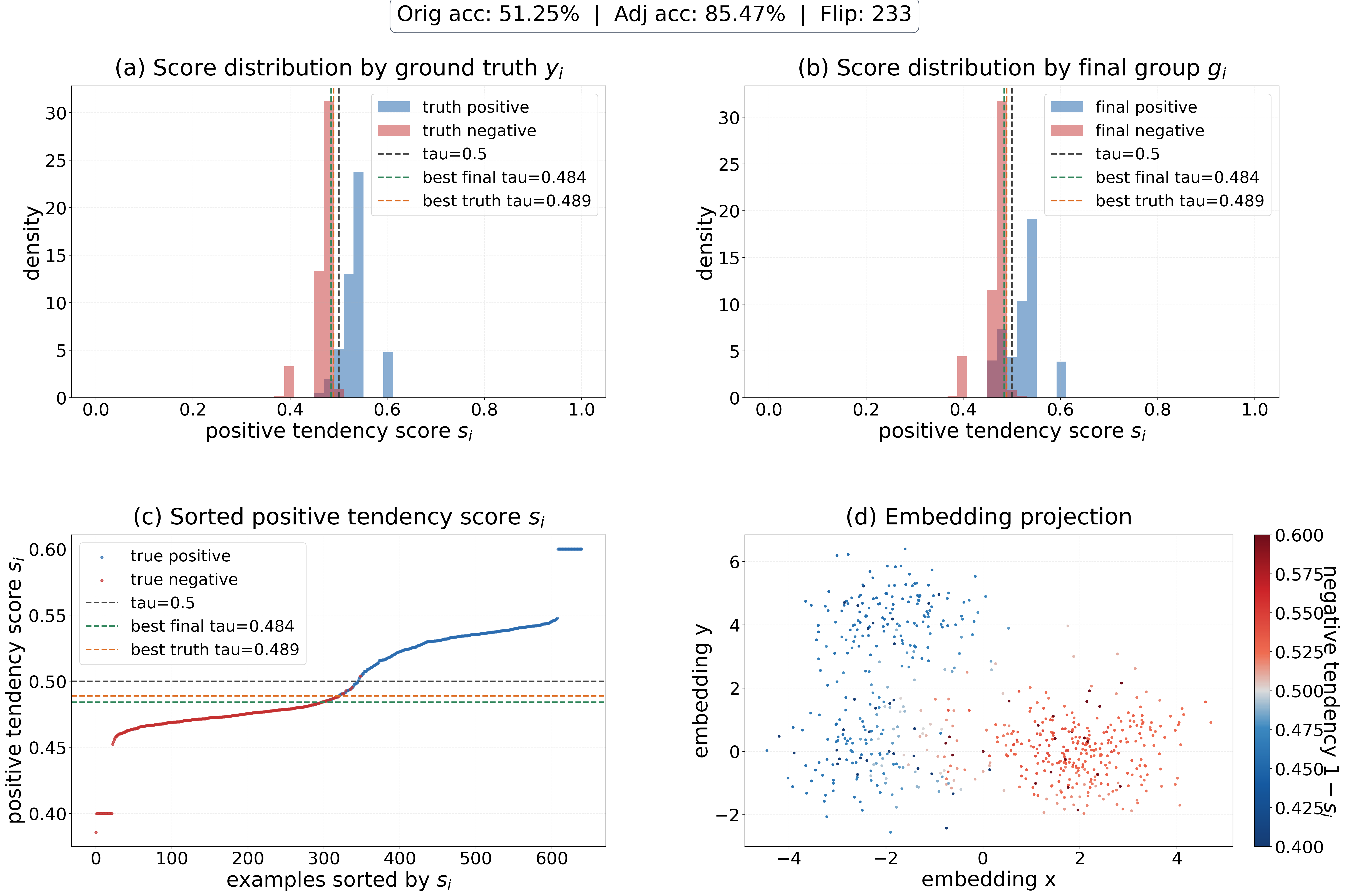}
    \caption{
    Threshold diagnostic under distribution-dependent (DD) noise with SAR verification.
    Details follow the same reporting as Figure~\ref{fig:app_sim_cd_scar_threshold}.
    }
    \label{fig:app_sim_dd_sar_threshold}
\end{figure}

\begin{figure}[t]
    \centering
    \includegraphics[width=1\linewidth]{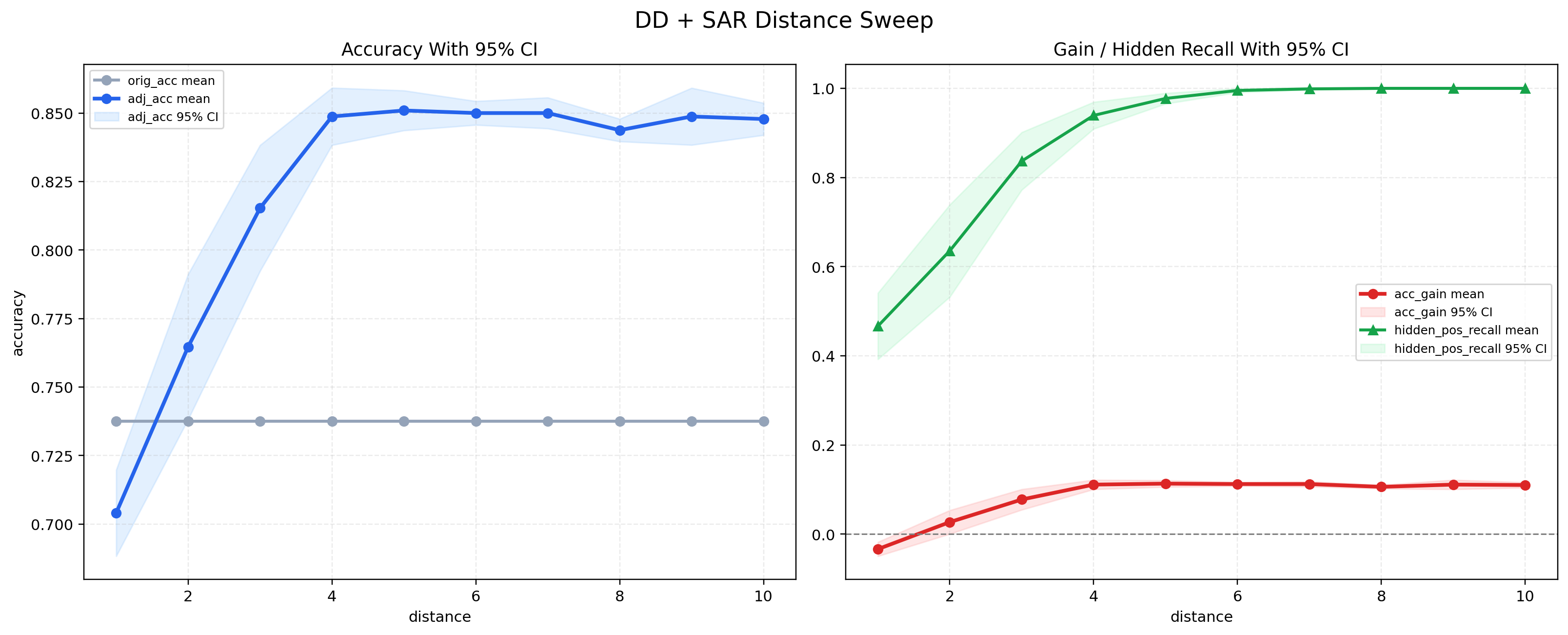}
    \caption{
    Distance analysis under the DD+SAR simulation setting with latent dimension fixed at \(d=2\). The class-separation distance controls how distinguishable the latent human-consistent and human-inconsistent groups are.
    \PUAuditPlus performs best when the latent groups are sufficiently separated, while performance weakens when the two groups heavily overlap.
    }
    \label{fig:app_sim_distance_sweep}
\end{figure}

\begin{figure}[t]
    \centering
    \includegraphics[width=1\linewidth]{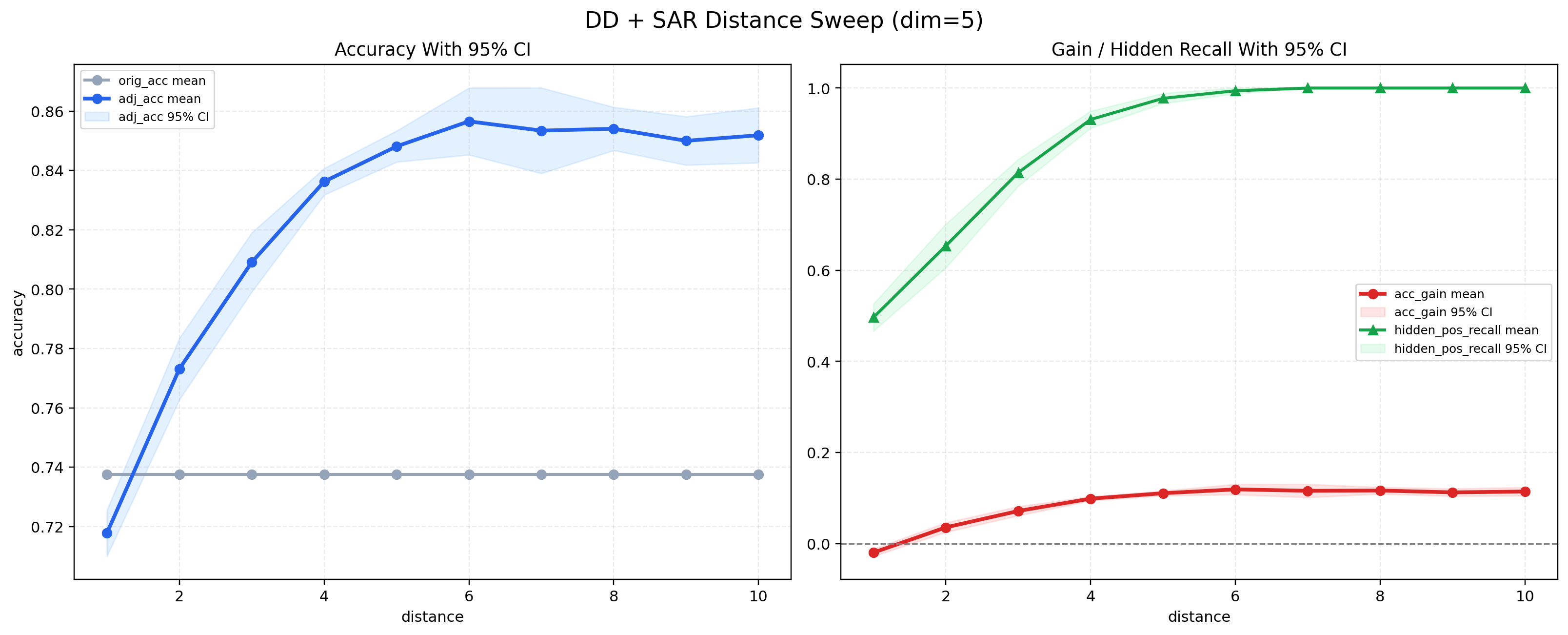}
    \caption{
    Distance analysis under the DD+SAR simulation setting with latent dimension fixed at \(d=5\).
    }
    \label{fig:app_sim_distance_sweep_dim5}
\end{figure}

\begin{figure}[t]
    \centering
    \includegraphics[width=1\linewidth]{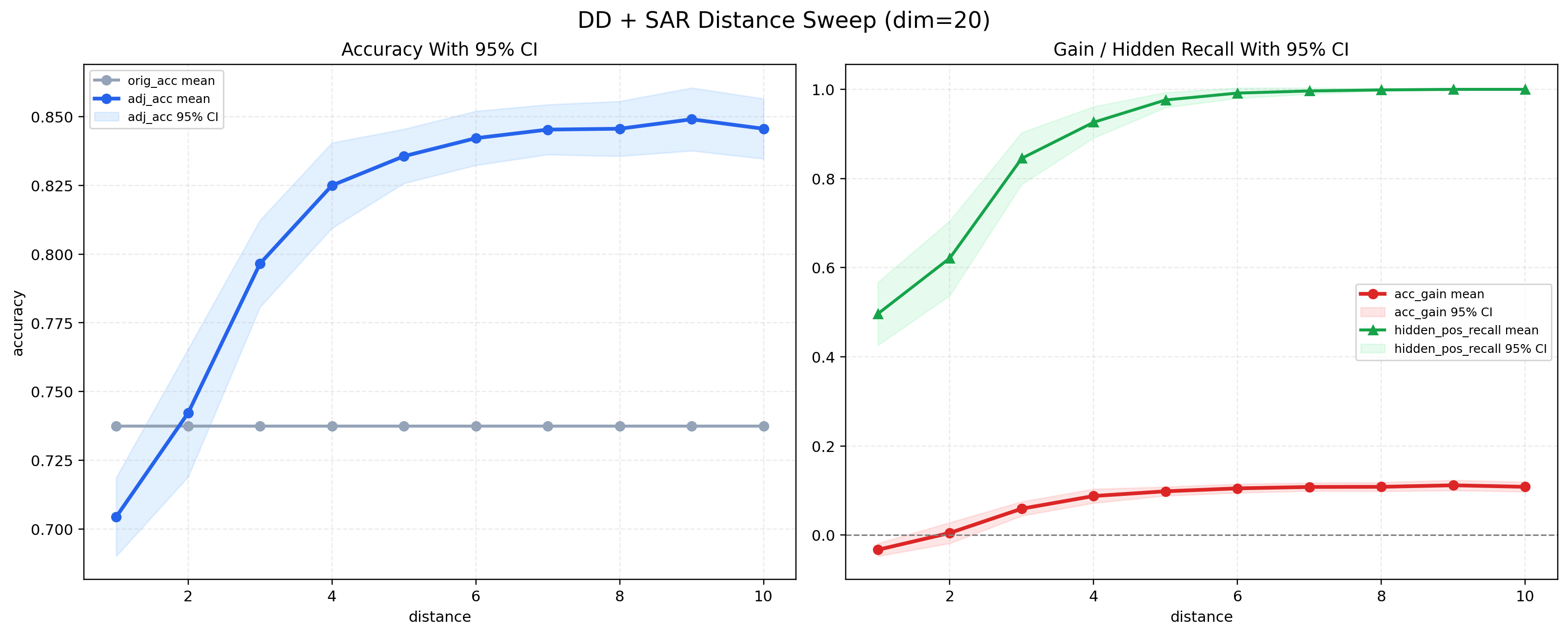}
    \caption{
    Distance analysis under the DD+SAR simulation setting with latent dimension fixed at \(d=20\).
    }
    \label{fig:app_sim_distance_sweep_dim20}
\end{figure}

\begin{figure}[t]
    \centering
    \includegraphics[width=1\linewidth]{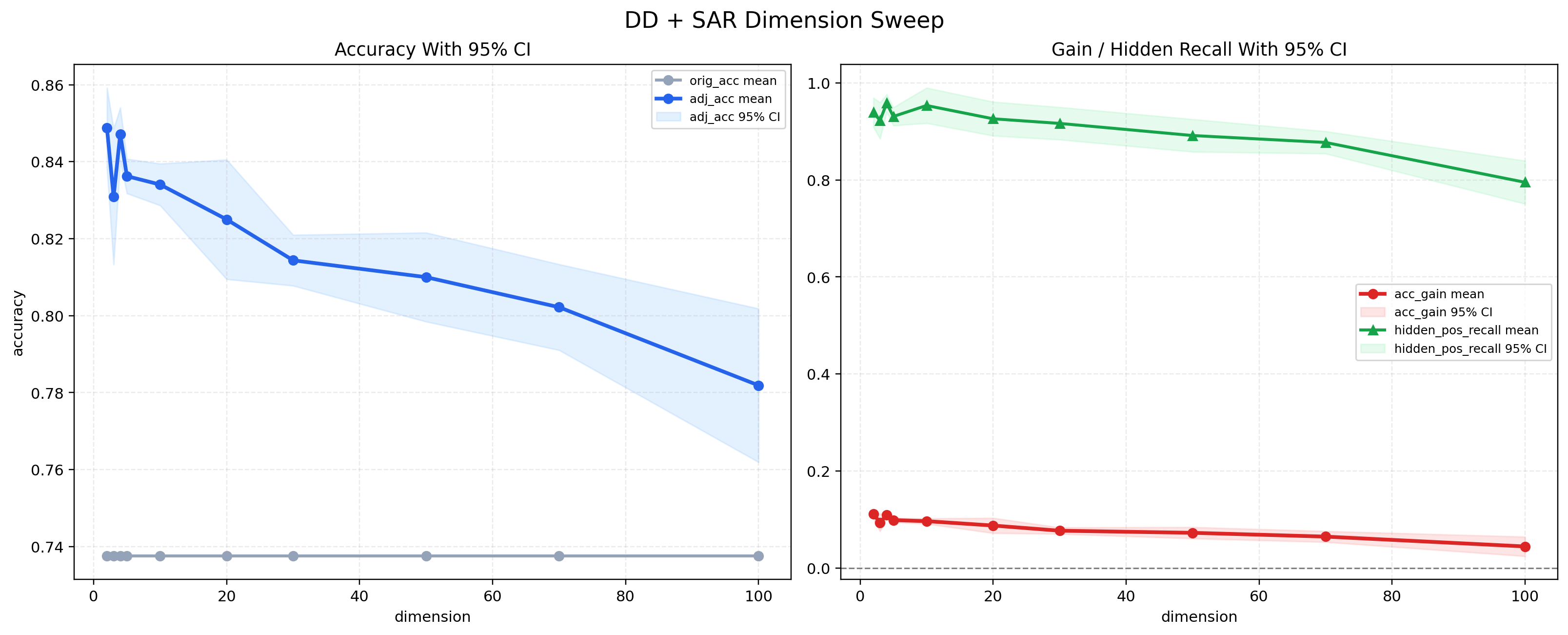}
    \caption{
    Dimension analysis under the DD+SAR simulation setting with class-center distance fixed at \(\|\bm{\mu}_1-\bm{\mu}_0\|=4\). The latent dimension controls the difficulty of estimating reliable local neighborhoods and anchor support.
    }
    \label{fig:app_sim_dimension_sweep}
\end{figure}

\begin{figure}[t]
    \centering
    \includegraphics[width=1\linewidth]{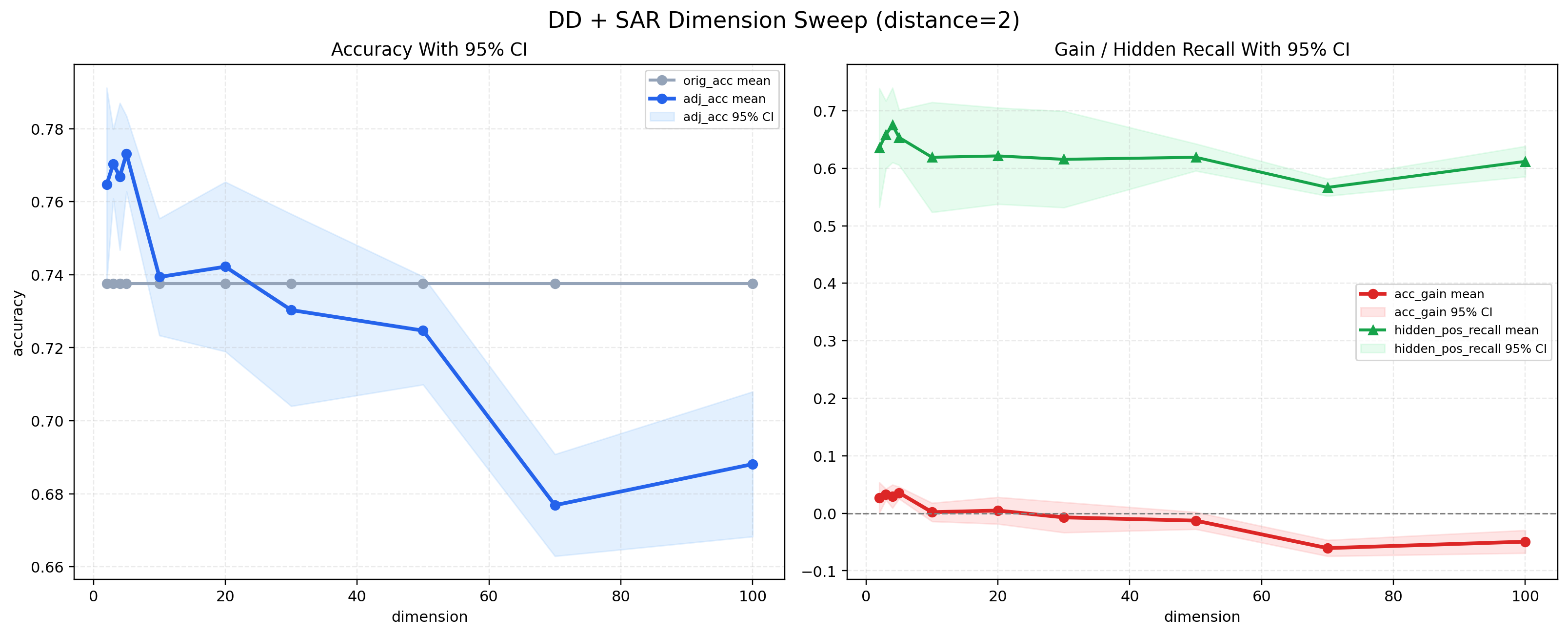}
    \caption{
    Dimension analysis under the DD+SAR simulation setting with class-center distance fixed at \(\|\bm{\mu}_1-\bm{\mu}_0\|=2\).
    }
    \label{fig:app_sim_dimension_sweep_dis2}
\end{figure}

\clearpage
\subsection{Additional real-data results}
\label{app:additional_real_data_results}

\begin{table*}[t]
\centering
\caption{Real-data results with 3\% verified data for baseline methods (Coding and MT-Bench). Accuracy and gain values are reported as mean values with range in parentheses, where the range is computed as maximum minus minimum across five random seeds. Data is the number of pairwise comparisons. Human Verif. Req. denotes the number of human-verified examples required by each method, averaged over five seeds with range in parentheses. Orig. Acc. and Adj. Acc. denote the original LLM-judge accuracy and the adjusted accuracy after refinement, respectively. Bold values indicate \PUAuditPlus results. Green cells highlight the highest Adj. Acc. or Gain between the best baseline and \PUAuditPlus within each model and question type.}
\resizebox{\textwidth}{!}{%
\begin{tabular}{cc|ccccc}
\toprule
\midrule
\multicolumn{7}{c}{\textbf{Real-Data Results with 3\% Verified Baseline Data (Coding and MT-Bench)}}\\
\midrule
\midrule
Question Type & Method & Data & Orig. Acc. & Adj. Acc. & Gain & Human Verif. Req. \\
\midrule
\multicolumn{7}{c}{\textbf{Model: GPT-5.4}}\\
\midrule
Coding & Logistic Regression & 2949 & 63.92\% & 57.23\% (3.57) & -6.69\% (3.57) & 89 (0) \\
Coding & Multilayer Perceptron & 2949 & 63.92\% & 57.96\% (9.13) & -5.96\% (9.13) & 89 (0) \\
Coding & Random Forest & 2949 & 63.92\% & 61.82\% (4.09) & -2.10\% (4.09) & 89 (0) \\
Coding & Label Propagation & 2949 & 63.92\% & 56.34\% (18.15) & -7.58\% (18.15) & 89 (0) \\
Coding & nnPU & 2949 & 63.92\% & 62.45\% (3.25) & -1.47\% (3.25) & 89 (0) \\
Coding & \PUAuditPlus & 2949 & 63.92\% & \hlgreen{\textbf{67.20\% (2.00)}} & \hlgreen{\textbf{+3.28\% (2.00)}} & 100.6 (5) \\
\midrule
\multicolumn{7}{c}{\textbf{Model: GPT-5.4-mini}}\\
\midrule
Coding & Logistic Regression & 2971 & 63.24\% & 56.34\% (3.78) & -6.91\% (3.78) & 89 (0) \\
Coding & Multilayer Perceptron & 2971 & 63.24\% & 59.22\% (9.58) & -4.02\% (9.58) & 89 (0) \\
Coding & Random Forest & 2971 & 63.24\% & 61.47\% (1.77) & -1.77\% (1.77) & 89 (0) \\
Coding & Label Propagation & 2971 & 63.24\% & 54.55\% (19.81) & -8.70\% (19.81) & 89 (0) \\
Coding & nnPU & 2971 & 63.24\% & 63.03\% (0.32) & -0.21\% (0.32) & 89 (0) \\
Coding & \PUAuditPlus & 2971 & 63.24\% & \hlgreen{\textbf{66.27\% (1.55)}} & \hlgreen{\textbf{+3.03\% (1.55)}} & 102.2 (9) \\
\midrule
MT-Bench & Logistic Regression & 2565 & 77.82\% & 74.27\% (4.34) & -3.55\% (4.34) & 77 (0) \\
MT-Bench & Multilayer Perceptron & 2565 & 77.82\% & 74.12\% (14.39) & -3.70\% (14.39) & 77 (0) \\
MT-Bench & Random Forest & 2565 & 77.82\% & 77.46\% (1.25) & -0.36\% (1.25) & 77 (0) \\
MT-Bench & Label Propagation & 2565 & 77.82\% & 64.91\% (51.69) & -12.91\% (51.69) & 77 (0) \\
MT-Bench & nnPU & 2565 & 77.82\% & 77.81\% (0.00) & -0.01\% (0.00) & 77 (0) \\
MT-Bench & \PUAuditPlus & 2575 & 77.82\% & \hlgreen{\textbf{81.01\% (0.94)}} & \hlgreen{\textbf{+3.19\% (0.94)}} & 87.4 (7) \\
\midrule
\multicolumn{7}{c}{\textbf{Model: Gemini}}\\
\midrule
Coding & Logistic Regression & 2820 & 63.55\% & 56.37\% (3.03) & -7.18\% (3.03) & 85 (0) \\
Coding & Multilayer Perceptron & 2820 & 63.55\% & 59.33\% (5.96) & -4.21\% (5.96) & 85 (0) \\
Coding & Random Forest & 2820 & 63.55\% & 61.63\% (4.17) & -1.92\% (4.17) & 85 (0) \\
Coding & Label Propagation & 2820 & 63.55\% & 46.87\% (26.29) & -16.68\% (26.29) & 85 (0) \\
Coding & nnPU & 2820 & 63.55\% & 63.43\% (0.14) & -0.12\% (0.14) & 85 (0) \\
Coding & \PUAuditPlus & 2820 & 63.55\% & \hlgreen{\textbf{67.17\% (2.84)}} & \hlgreen{\textbf{+3.62\% (2.84)}} & 96.6 (8) \\
\midrule
\multicolumn{7}{c}{\textbf{Model: Qwen}}\\
\midrule
Coding & Logistic Regression & 8884 & 53.73\% & 53.99\% (1.90) & +0.27\% (1.90) & 266 (0) \\
Coding & Multilayer Perceptron & 8884 & 53.73\% & 53.25\% (6.15) & -0.48\% (6.15) & 266 (0) \\
Coding & Random Forest & 8884 & 53.73\% & 56.70\% (2.23) & +2.98\% (2.23) & 266 (0) \\
Coding & Label Propagation & 8884 & 53.73\% & 50.87\% (4.55) & -2.86\% (4.55) & 266 (0) \\
Coding & nnPU & 8884 & 53.73\% & 55.74\% (1.65) & +2.01\% (1.65) & 266 (0) \\
Coding & \PUAuditPlus & 8884 & 53.73\% & \hlgreen{\textbf{77.74\% (9.78)}} & \hlgreen{\textbf{+24.02\% (9.78)}} & 278.8 (14) \\
\midrule
\multicolumn{7}{c}{\textbf{Model: Mistral}}\\
\midrule
Coding & Logistic Regression & 8240 & 50.29\% & 53.24\% (3.77) & +2.95\% (3.77) & 247 (0) \\
Coding & Multilayer Perceptron & 8240 & 50.29\% & 54.12\% (6.04) & +3.83\% (6.04) & 247 (0) \\
Coding & Random Forest & 8240 & 50.29\% & 57.20\% (4.63) & +6.90\% (4.63) & 247 (0) \\
Coding & Label Propagation & 8240 & 50.29\% & 50.89\% (1.83) & +0.60\% (1.83) & 247 (0) \\
Coding & nnPU & 8240 & 50.29\% & 55.38\% (3.31) & +5.09\% (3.31) & 247 (0) \\
Coding & \PUAuditPlus & 8240 & 50.29\% & \hlgreen{\textbf{72.91\% (0.61)}} & \hlgreen{\textbf{+22.62\% (0.61)}} & 261.8 (10) \\
\bottomrule
\end{tabular}}
\label{tab:real_data_verified_3_coding_mtbench}
\end{table*}

\begin{table*}[t]
\centering
\caption{Real-data results with 3\% verified data for baseline methods (Math). Values follow the same reporting and highlighting conventions as Table~\ref{tab:real_data_verified_3_coding_mtbench}.}
\resizebox{\textwidth}{!}{%
\begin{tabular}{cc|ccccc}
\toprule
\midrule
\multicolumn{7}{c}{\textbf{Real-Data Results with 3\% Verified Baseline Data (Math)}}\\
\midrule
\midrule
Question Type & Method & Data & Orig. Acc. & Adj. Acc. & Gain & Human Verif. Req. \\
\midrule
\multicolumn{7}{c}{\textbf{Model: GPT-5.4}}\\
\midrule
Math & Logistic Regression & 2817 & 67.20\% & 61.05\% (4.43) & -6.15\% (4.43) & 85 (0) \\
Math & Multilayer Perceptron & 2817 & 67.20\% & 62.00\% (8.09) & -5.20\% (8.09) & 85 (0) \\
Math & Random Forest & 2817 & 67.20\% & 65.92\% (2.93) & -1.28\% (2.93) & 85 (0) \\
Math & Label Propagation & 2817 & 67.20\% & 67.07\% (0.48) & -0.13\% (0.48) & 85 (0) \\
Math & nnPU & 2817 & 67.20\% & 67.23\% (0.05) & +0.03\% (0.05) & 85 (0) \\
Math & \PUAuditPlus & 2817 & 67.20\% & \hlgreen{\textbf{70.05\% (1.17)}} & \hlgreen{\textbf{+2.85\% (1.17)}} & 95.8 (5) \\
\midrule
\multicolumn{7}{c}{\textbf{Model: GPT-5.4-mini}}\\
\midrule
Math & Logistic Regression & 2911 & 62.35\% & 58.55\% (2.12) & -3.80\% (2.12) & 87 (0) \\
Math & Multilayer Perceptron & 2911 & 62.35\% & 56.88\% (12.64) & -5.47\% (12.64) & 87 (0) \\
Math & Random Forest & 2911 & 62.35\% & 63.68\% (2.41) & +1.33\% (2.41) & 87 (0) \\
Math & Label Propagation & 2911 & 62.35\% & 59.89\% (12.82) & -2.46\% (12.82) & 87 (0) \\
Math & nnPU & 2911 & 62.35\% & 62.53\% (0.65) & +0.18\% (0.65) & 87 (0) \\
Math & \PUAuditPlus & 2911 & 62.35\% & \hlgreen{\textbf{65.62\% (1.79)}} & \hlgreen{\textbf{+3.27\% (1.79)}} & 99 (5) \\
\midrule
\multicolumn{7}{c}{\textbf{Model: Gemini}}\\
\midrule
Math & Logistic Regression & 2770 & 66.35\% & 59.71\% (6.44) & -6.64\% (6.44) & 83 (0) \\
Math & Multilayer Perceptron & 2770 & 66.35\% & 59.86\% (11.46) & -6.50\% (11.46) & 83 (0) \\
Math & Random Forest & 2770 & 66.35\% & 65.63\% (3.13) & -0.72\% (3.13) & 83 (0) \\
Math & Label Propagation & 2770 & 66.35\% & 56.17\% (26.65) & -10.19\% (26.65) & 83 (0) \\
Math & nnPU & 2770 & 66.35\% & 66.19\% (0.26) & -0.16\% (0.26) & 83 (0) \\
Math & \PUAuditPlus & 2770 & 66.35\% & \hlgreen{\textbf{68.79\% (1.52)}} & \hlgreen{\textbf{+2.43\% (1.52)}} & 96.2 (4) \\
\midrule
\multicolumn{7}{c}{\textbf{Model: Qwen}}\\
\midrule
Math & Logistic Regression & 9145 & 55.10\% & 55.71\% (1.54) & +0.61\% (1.54) & 274 (0) \\
Math & Multilayer Perceptron & 9145 & 55.10\% & 54.28\% (10.39) & -0.82\% (10.39) & 274 (0) \\
Math & Random Forest & 9145 & 55.10\% & 59.70\% (2.19) & +4.60\% (2.19) & 274 (0) \\
Math & Label Propagation & 9145 & 55.10\% & 53.82\% (2.15) & -1.28\% (2.15) & 274 (0) \\
Math & nnPU & 9145 & 55.10\% & 58.88\% (0.90) & +3.78\% (0.90) & 274 (0) \\
Math & \PUAuditPlus & 9145 & 55.10\% & \hlgreen{\textbf{79.27\% (5.47)}} & \hlgreen{\textbf{+24.17\% (5.47)}} & 291.2 (10) \\
\midrule
\multicolumn{7}{c}{\textbf{Model: Mistral}}\\
\midrule
Math & Logistic Regression & 8361 & 50.46\% & 56.73\% (2.75) & +6.27\% (2.75) & 251 (0) \\
Math & Multilayer Perceptron & 8361 & 50.46\% & 54.30\% (5.60) & +3.84\% (5.60) & 251 (0) \\
Math & Random Forest & 8361 & 50.46\% & 60.53\% (1.94) & +10.07\% (1.94) & 251 (0) \\
Math & Label Propagation & 8361 & 50.46\% & 51.55\% (3.11) & +1.09\% (3.11) & 251 (0) \\
Math & nnPU & 8361 & 50.46\% & 55.24\% (4.09) & +4.78\% (4.09) & 251 (0) \\
Math & \PUAuditPlus & 8361 & 50.46\% & \hlgreen{\textbf{71.75\% (6.54)}} & \hlgreen{\textbf{+21.29\% (6.54)}} & 266 (23) \\
\bottomrule
\end{tabular}}
\label{tab:real_data_verified_3_math}
\end{table*}

\begin{table*}[t]
\centering
\caption{Real-data results with 3\% verified data for baseline methods (Factual). Values follow the same reporting and highlighting conventions as Table~\ref{tab:real_data_verified_3_coding_mtbench}.}
\resizebox{\textwidth}{!}{%
\begin{tabular}{cc|ccccc}
\toprule
\midrule
\multicolumn{7}{c}{\textbf{Real-Data Results with 3\% Verified Baseline Data (Factual)}}\\
\midrule
\midrule
Question Type & Method & Data & Orig. Acc. & Adj. Acc. & Gain & Human Verif. Req. \\
\midrule
\multicolumn{7}{c}{\textbf{Model: GPT-5.4}}\\
\midrule
Factual & Logistic Regression & 2842 & 64.39\% & 58.72\% (7.65) & -5.68\% (7.65) & 85 (0) \\
Factual & Multilayer Perceptron & 2842 & 64.39\% & 59.01\% (7.83) & -5.38\% (7.83) & 85 (0) \\
Factual & Random Forest & 2842 & 64.39\% & 63.40\% (2.97) & -0.99\% (2.97) & 85 (0) \\
Factual & Label Propagation & 2842 & 64.39\% & 54.87\% (23.10) & -9.52\% (23.10) & 85 (0) \\
Factual & nnPU & 2842 & 64.39\% & 64.39\% (0.02) & -0.00\% (0.02) & 85 (0) \\
Factual & \PUAuditPlus & 2842 & 64.39\% & \hlgreen{\textbf{67.66\% (2.22)}} & \hlgreen{\textbf{+3.27\% (2.22)}} & 95.6 (9) \\
\midrule
\multicolumn{7}{c}{\textbf{Model: GPT-5.4-mini}}\\
\midrule
Factual & Logistic Regression & 2920 & 62.47\% & 57.20\% (4.31) & -5.26\% (4.31) & 88 (0) \\
Factual & Multilayer Perceptron & 2920 & 62.47\% & 56.79\% (5.19) & -5.68\% (5.19) & 88 (0) \\
Factual & Random Forest & 2920 & 62.47\% & 61.52\% (4.03) & -0.95\% (4.03) & 88 (0) \\
Factual & Label Propagation & 2920 & 62.47\% & 53.23\% (18.50) & -9.24\% (18.50) & 88 (0) \\
Factual & nnPU & 2920 & 62.47\% & 62.44\% (0.11) & -0.03\% (0.11) & 88 (0) \\
Factual & \PUAuditPlus & 2920 & 62.47\% & \hlgreen{\textbf{66.07\% (2.09)}} & \hlgreen{\textbf{+3.60\% (2.09)}} & 99 (17) \\
\midrule
\multicolumn{7}{c}{\textbf{Model: Gemini}}\\
\midrule
Factual & Logistic Regression & 2768 & 66.37\% & 59.14\% (4.13) & -7.22\% (4.13) & 83 (0) \\
Factual & Multilayer Perceptron & 2768 & 66.37\% & 62.22\% (7.82) & -4.15\% (7.82) & 83 (0) \\
Factual & Random Forest & 2768 & 66.37\% & 65.29\% (2.20) & -1.08\% (2.20) & 83 (0) \\
Factual & Label Propagation & 2768 & 66.37\% & 54.91\% (26.63) & -11.45\% (26.63) & 83 (0) \\
Factual & nnPU & 2768 & 66.37\% & 66.03\% (0.75) & -0.34\% (0.75) & 83 (0) \\
Factual & \PUAuditPlus & 2768 & 66.37\% & \hlgreen{\textbf{69.21\% (1.73)}} & \hlgreen{\textbf{+2.85\% (1.73)}} & 94 (12) \\
\midrule
\multicolumn{7}{c}{\textbf{Model: Qwen}}\\
\midrule
Factual & Logistic Regression & 9333 & 55.94\% & 56.15\% (2.88) & +0.21\% (2.88) & 280 (0) \\
Factual & Multilayer Perceptron & 9333 & 55.94\% & 56.51\% (5.39) & +0.57\% (5.39) & 280 (0) \\
Factual & Random Forest & 9333 & 55.94\% & 60.15\% (2.14) & +4.21\% (2.14) & 280 (0) \\
Factual & Label Propagation & 9333 & 55.94\% & 51.89\% (6.01) & -4.05\% (6.01) & 280 (0) \\
Factual & nnPU & 9333 & 55.94\% & 59.23\% (1.63) & +3.29\% (1.63) & 280 (0) \\
Factual & \PUAuditPlus & 9333 & 55.94\% & \hlgreen{\textbf{80.31\% (4.76)}} & \hlgreen{\textbf{+24.37\% (4.76)}} & 292.8 (11) \\
\midrule
\multicolumn{7}{c}{\textbf{Model: Mistral}}\\
\midrule
Factual & Logistic Regression & 8260 & 51.05\% & 56.11\% (2.91) & +5.06\% (2.91) & 248 (0) \\
Factual & Multilayer Perceptron & 8260 & 51.05\% & 53.90\% (7.81) & +2.84\% (7.81) & 248 (0) \\
Factual & Random Forest & 8260 & 51.05\% & 59.96\% (2.97) & +8.91\% (2.97) & 248 (0) \\
Factual & Label Propagation & 8260 & 51.05\% & 50.87\% (2.60) & -0.19\% (2.60) & 248 (0) \\
Factual & nnPU & 8260 & 51.05\% & 56.32\% (3.73) & +5.27\% (3.73) & 248 (0) \\
Factual & \PUAuditPlus & 8260 & 51.05\% & \hlgreen{\textbf{72.37\% (6.49)}} & \hlgreen{\textbf{+21.31\% (6.49)}} & 259 (7) \\
\bottomrule
\end{tabular}}
\label{tab:real_data_verified_3_factual}
\end{table*}

\begin{table*}[t]
\centering
\caption{Real-data results with 20\% verified data for baseline methods (Coding and MT-Bench). Values follow the same reporting and highlighting conventions as Table~\ref{tab:real_data_verified_3_coding_mtbench}.}
\resizebox{\textwidth}{!}{%
\begin{tabular}{cc|ccccc}
\toprule
\midrule
\multicolumn{7}{c}{\textbf{Real-Data Results with 20\% Verified Baseline Data (Coding and MT-Bench)}}\\
\midrule
\midrule
Question Type & Method & Data & Orig. Acc. & Adj. Acc. & Gain & Human Verif. Req. \\
\midrule
\multicolumn{7}{c}{\textbf{Model: GPT-5.4}}\\
\midrule
Coding & Logistic Regression & 2949 & 63.92\% & 59.54\% (3.56) & -4.38\% (3.56) & 590 (0) \\
Coding & Multilayer Perceptron & 2949 & 63.92\% & 60.83\% (2.80) & -3.09\% (2.80) & 590 (0) \\
Coding & Random Forest & 2949 & 63.92\% & 63.23\% (1.87) & -0.69\% (1.87) & 590 (0) \\
Coding & Label Propagation & 2949 & 63.92\% & 60.50\% (9.92) & -3.42\% (9.92) & 590 (0) \\
Coding & nnPU & 2949 & 63.92\% & 63.86\% (0.10) & -0.06\% (0.10) & 590 (0) \\
Coding & \PUAuditPlus & 2949 & 63.92\% & \hlgreen{\textbf{67.20\% (2.00)}} & \hlgreen{\textbf{+3.28\% (2.00)}} & 100.6 (5) \\
\midrule
\multicolumn{7}{c}{\textbf{Model: GPT-5.4-mini}}\\
\midrule
Coding & Logistic Regression & 2971 & 63.24\% & 59.20\% (0.84) & -4.04\% (0.84) & 594 (0) \\
Coding & Multilayer Perceptron & 2971 & 63.24\% & 61.40\% (4.29) & -1.85\% (4.29) & 594 (0) \\
Coding & Random Forest & 2971 & 63.24\% & 63.27\% (1.22) & +0.03\% (1.22) & 594 (0) \\
Coding & Label Propagation & 2971 & 63.24\% & 54.38\% (11.11) & -8.87\% (11.11) & 594 (0) \\
Coding & nnPU & 2971 & 63.24\% & 63.43\% (0.73) & +0.19\% (0.73) & 594 (0) \\
Coding & \PUAuditPlus & 2971 & 63.24\% & \hlgreen{\textbf{66.27\% (1.55)}} & \hlgreen{\textbf{+3.03\% (1.55)}} & 102.2 (9) \\
\midrule
MT-Bench & Logistic Regression & 2565 & 77.82\% & 75.39\% (3.61) & -2.43\% (3.61) & 513 (0) \\
MT-Bench & Multilayer Perceptron & 2565 & 77.82\% & 77.66\% (0.54) & -0.16\% (0.54) & 513 (0) \\
MT-Bench & Random Forest & 2565 & 77.82\% & 77.28\% (1.80) & -0.54\% (1.80) & 513 (0) \\
MT-Bench & Label Propagation & 2565 & 77.82\% & 76.74\% (3.61) & -1.07\% (3.61) & 513 (0) \\
MT-Bench & nnPU & 2565 & 77.82\% & 77.83\% (0.00) & +0.01\% (0.00) & 513 (0) \\
MT-Bench & \PUAuditPlus & 2575 & 77.82\% & \hlgreen{\textbf{81.01\% (0.94)}} & \hlgreen{\textbf{+3.19\% (0.94)}} & 87.4 (7) \\
\midrule
\multicolumn{7}{c}{\textbf{Model: Gemini}}\\
\midrule
Coding & Logistic Regression & 2820 & 63.55\% & 59.72\% (2.57) & -3.83\% (2.57) & 564 (0) \\
Coding & Multilayer Perceptron & 2820 & 63.55\% & 61.76\% (6.25) & -1.79\% (6.25) & 564 (0) \\
Coding & Random Forest & 2820 & 63.55\% & 63.39\% (1.68) & -0.16\% (1.68) & 564 (0) \\
Coding & Label Propagation & 2820 & 63.55\% & 58.65\% (6.12) & -4.89\% (6.12) & 564 (0) \\
Coding & nnPU & 2820 & 63.55\% & 63.51\% (0.12) & -0.04\% (0.12) & 564 (0) \\
Coding & \PUAuditPlus & 2820 & 63.55\% & \hlgreen{\textbf{67.17\% (2.84)}} & \hlgreen{\textbf{+3.62\% (2.84)}} & 96.6 (8) \\
\midrule
\multicolumn{7}{c}{\textbf{Model: Qwen}}\\
\midrule
Coding & Logistic Regression & 8884 & 53.73\% & 58.47\% (1.59) & +4.75\% (1.59) & 1777 (0) \\
Coding & Multilayer Perceptron & 8884 & 53.73\% & 57.44\% (1.83) & +3.72\% (1.83) & 1777 (0) \\
Coding & Random Forest & 8884 & 53.73\% & 59.20\% (0.96) & +5.47\% (0.96) & 1777 (0) \\
Coding & Label Propagation & 8884 & 53.73\% & 53.77\% (2.11) & +0.04\% (2.11) & 1777 (0) \\
Coding & nnPU & 8884 & 53.73\% & 58.54\% (0.48) & +4.81\% (0.48) & 1777 (0) \\
Coding & \PUAuditPlus & 8884 & 53.73\% & \hlgreen{\textbf{77.74\% (9.78)}} & \hlgreen{\textbf{+24.02\% (9.78)}} & 278.8 (14) \\
\midrule
\multicolumn{7}{c}{\textbf{Model: Mistral}}\\
\midrule
Coding & Logistic Regression & 8240 & 50.29\% & 57.52\% (1.18) & +7.22\% (1.18) & 1648 (0) \\
Coding & Multilayer Perceptron & 8240 & 50.29\% & 56.82\% (2.64) & +6.53\% (2.64) & 1648 (0) \\
Coding & Random Forest & 8240 & 50.29\% & 58.88\% (1.27) & +8.59\% (1.27) & 1648 (0) \\
Coding & Label Propagation & 8240 & 50.29\% & 53.83\% (1.94) & +3.53\% (1.94) & 1648 (0) \\
Coding & nnPU & 8240 & 50.29\% & 57.97\% (0.59) & +7.68\% (0.59) & 1648 (0) \\
Coding & \PUAuditPlus & 8240 & 50.29\% & \hlgreen{\textbf{72.91\% (0.61)}} & \hlgreen{\textbf{+22.62\% (0.61)}} & 261.8 (10) \\
\bottomrule
\end{tabular}}
\label{tab:real_data_verified_20_coding_mtbench}
\end{table*}

\begin{table*}[t]
\centering
\caption{Real-data results with 20\% verified data for baseline methods (Math). Values follow the same reporting and highlighting conventions as Table~\ref{tab:real_data_verified_3_coding_mtbench}.}
\resizebox{\textwidth}{!}{%
\begin{tabular}{cc|ccccc}
\toprule
\midrule
\multicolumn{7}{c}{\textbf{Real-Data Results with 20\% Verified Baseline Data (Math)}}\\
\midrule
\midrule
Question Type & Method & Data & Orig. Acc. & Adj. Acc. & Gain & Human Verif. Req. \\
\midrule
\multicolumn{7}{c}{\textbf{Model: GPT-5.4}}\\
\midrule
Math & Logistic Regression & 2817 & 67.20\% & 62.92\% (3.77) & -4.28\% (3.77) & 564 (0) \\
Math & Multilayer Perceptron & 2817 & 67.20\% & 65.45\% (2.75) & -1.75\% (2.75) & 564 (0) \\
Math & Random Forest & 2817 & 67.20\% & 67.98\% (1.38) & +0.78\% (1.38) & 564 (0) \\
Math & Label Propagation & 2817 & 67.20\% & 65.28\% (2.66) & -1.92\% (2.66) & 564 (0) \\
Math & nnPU & 2817 & 67.20\% & 67.32\% (0.19) & +0.12\% (0.19) & 564 (0) \\
Math & \PUAuditPlus & 2817 & 67.20\% & \hlgreen{\textbf{70.05\% (1.17)}} & \hlgreen{\textbf{+2.85\% (1.17)}} & 95.8 (5) \\
\midrule
\multicolumn{7}{c}{\textbf{Model: GPT-5.4-mini}}\\
\midrule
Math & Logistic Regression & 2911 & 62.35\% & 62.18\% (2.23) & -0.17\% (2.23) & 582 (0) \\
Math & Multilayer Perceptron & 2911 & 62.35\% & 63.31\% (1.50) & +0.97\% (1.50) & 582 (0) \\
Math & Random Forest & 2911 & 62.35\% & 65.47\% (2.49) & +3.12\% (2.49) & 582 (0) \\
Math & Label Propagation & 2911 & 62.35\% & 62.46\% (1.80) & +0.11\% (1.80) & 582 (0) \\
Math & nnPU & 2911 & 62.35\% & 63.31\% (0.77) & +0.96\% (0.77) & 582 (0) \\
Math & \PUAuditPlus & 2911 & 62.35\% & \hlgreen{\textbf{65.62\% (1.79)}} & \hlgreen{\textbf{+3.27\% (1.79)}} & 99 (5) \\
\midrule
\multicolumn{7}{c}{\textbf{Model: Gemini}}\\
\midrule
Math & Logistic Regression & 2770 & 66.35\% & 63.61\% (2.44) & -2.74\% (2.44) & 554 (0) \\
Math & Multilayer Perceptron & 2770 & 66.35\% & 64.50\% (2.21) & -1.85\% (2.21) & 554 (0) \\
Math & Random Forest & 2770 & 66.35\% & 67.45\% (1.44) & +1.09\% (1.44) & 554 (0) \\
Math & Label Propagation & 2770 & 66.35\% & 65.05\% (3.70) & -1.31\% (3.70) & 554 (0) \\
Math & nnPU & 2770 & 66.35\% & 66.31\% (0.13) & -0.04\% (0.13) & 554 (0) \\
Math & \PUAuditPlus & 2770 & 66.35\% & \hlgreen{\textbf{68.79\% (1.52)}} & \hlgreen{\textbf{+2.43\% (1.52)}} & 96.2 (4) \\
\midrule
\multicolumn{7}{c}{\textbf{Model: Qwen}}\\
\midrule
Math & Logistic Regression & 9145 & 55.10\% & 61.13\% (1.22) & +6.03\% (1.22) & 1829 (0) \\
Math & Multilayer Perceptron & 9145 & 55.10\% & 59.60\% (3.35) & +4.50\% (3.35) & 1829 (0) \\
Math & Random Forest & 9145 & 55.10\% & 62.85\% (1.56) & +7.75\% (1.56) & 1829 (0) \\
Math & Label Propagation & 9145 & 55.10\% & 55.35\% (1.18) & +0.25\% (1.18) & 1829 (0) \\
Math & nnPU & 9145 & 55.10\% & 60.27\% (1.58) & +5.17\% (1.58) & 1829 (0) \\
Math & \PUAuditPlus & 9145 & 55.10\% & \hlgreen{\textbf{79.27\% (5.47)}} & \hlgreen{\textbf{+24.17\% (5.47)}} & 291.2 (10) \\
\midrule
\multicolumn{7}{c}{\textbf{Model: Mistral}}\\
\midrule
Math & Logistic Regression & 8361 & 50.46\% & 61.63\% (0.76) & +11.17\% (0.76) & 1672 (0) \\
Math & Multilayer Perceptron & 8361 & 50.46\% & 60.61\% (1.97) & +10.15\% (1.97) & 1672 (0) \\
Math & Random Forest & 8361 & 50.46\% & 62.63\% (1.38) & +12.17\% (1.38) & 1672 (0) \\
Math & Label Propagation & 8361 & 50.46\% & 55.64\% (0.78) & +5.18\% (0.78) & 1672 (0) \\
Math & nnPU & 8361 & 50.46\% & 60.68\% (0.62) & +10.22\% (0.62) & 1672 (0) \\
Math & \PUAuditPlus & 8361 & 50.46\% & \hlgreen{\textbf{71.75\% (6.54)}} & \hlgreen{\textbf{+21.29\% (6.54)}} & 266 (23) \\
\bottomrule
\end{tabular}}
\label{tab:real_data_verified_20_math}
\end{table*}

\begin{table*}[t]
\centering
\caption{Real-data results with 20\% verified data for baseline methods (Factual). Values follow the same reporting and highlighting conventions as Table~\ref{tab:real_data_verified_3_coding_mtbench}.}
\resizebox{\textwidth}{!}{%
\begin{tabular}{cc|ccccc}
\toprule
\midrule
\multicolumn{7}{c}{\textbf{Real-Data Results with 20\% Verified Baseline Data (Factual)}}\\
\midrule
\midrule
Question Type & Method & Data & Orig. Acc. & Adj. Acc. & Gain & Human Verif. Req. \\
\midrule
\multicolumn{7}{c}{\textbf{Model: GPT-5.4}}\\
\midrule
Factual & Logistic Regression & 2842 & 64.39\% & 62.19\% (2.02) & -2.20\% (2.02) & 568 (0) \\
Factual & Multilayer Perceptron & 2842 & 64.39\% & 64.12\% (2.07) & -0.27\% (2.07) & 568 (0) \\
Factual & Random Forest & 2842 & 64.39\% & 65.82\% (1.72) & +1.43\% (1.72) & 568 (0) \\
Factual & Label Propagation & 2842 & 64.39\% & 61.84\% (4.31) & -2.55\% (4.31) & 568 (0) \\
Factual & nnPU & 2842 & 64.39\% & 65.03\% (0.75) & +0.64\% (0.75) & 568 (0) \\
Factual & \PUAuditPlus & 2842 & 64.39\% & \hlgreen{\textbf{67.66\% (2.22)}} & \hlgreen{\textbf{+3.27\% (2.22)}} & 95.6 (9) \\
\midrule
\multicolumn{7}{c}{\textbf{Model: GPT-5.4-mini}}\\
\midrule
Factual & Logistic Regression & 2920 & 62.47\% & 61.38\% (3.51) & -1.09\% (3.51) & 584 (0) \\
Factual & Multilayer Perceptron & 2920 & 62.47\% & 61.44\% (3.98) & -1.03\% (3.98) & 584 (0) \\
Factual & Random Forest & 2920 & 62.47\% & 64.77\% (1.97) & +2.30\% (1.97) & 584 (0) \\
Factual & Label Propagation & 2920 & 62.47\% & 59.41\% (3.17) & -3.06\% (3.17) & 584 (0) \\
Factual & nnPU & 2920 & 62.47\% & 63.16\% (0.57) & +0.69\% (0.57) & 584 (0) \\
Factual & \PUAuditPlus & 2920 & 62.47\% & \hlgreen{\textbf{66.07\% (2.09)}} & \hlgreen{\textbf{+3.60\% (2.09)}} & 99 (17) \\
\midrule
\multicolumn{7}{c}{\textbf{Model: Gemini}}\\
\midrule
Factual & Logistic Regression & 2768 & 66.37\% & 61.16\% (2.71) & -5.21\% (2.71) & 553 (0) \\
Factual & Multilayer Perceptron & 2768 & 66.37\% & 64.81\% (6.41) & -1.55\% (6.41) & 553 (0) \\
Factual & Random Forest & 2768 & 66.37\% & 66.58\% (0.99) & +0.22\% (0.99) & 553 (0) \\
Factual & Label Propagation & 2768 & 66.37\% & 61.95\% (4.38) & -4.42\% (4.38) & 553 (0) \\
Factual & nnPU & 2768 & 66.37\% & 66.33\% (0.08) & -0.04\% (0.08) & 553 (0) \\
Factual & \PUAuditPlus & 2768 & 66.37\% & \hlgreen{\textbf{69.21\% (1.73)}} & \hlgreen{\textbf{+2.85\% (1.73)}} & 94 (12) \\
\midrule
\multicolumn{7}{c}{\textbf{Model: Qwen}}\\
\midrule
Factual & Logistic Regression & 9333 & 55.94\% & 61.22\% (0.87) & +5.27\% (0.87) & 1866 (0) \\
Factual & Multilayer Perceptron & 9333 & 55.94\% & 60.10\% (2.68) & +4.16\% (2.68) & 1866 (0) \\
Factual & Random Forest & 9333 & 55.94\% & 63.01\% (1.18) & +7.07\% (1.18) & 1866 (0) \\
Factual & Label Propagation & 9333 & 55.94\% & 55.41\% (3.48) & -0.53\% (3.48) & 1866 (0) \\
Factual & nnPU & 9333 & 55.94\% & 60.63\% (1.42) & +4.69\% (1.42) & 1866 (0) \\
Factual & \PUAuditPlus & 9333 & 55.94\% & \hlgreen{\textbf{80.31\% (4.76)}} & \hlgreen{\textbf{+24.37\% (4.76)}} & 292.8 (11) \\
\midrule
\multicolumn{7}{c}{\textbf{Model: Mistral}}\\
\midrule
Factual & Logistic Regression & 8260 & 51.05\% & 61.67\% (0.86) & +10.62\% (0.86) & 1652 (0) \\
Factual & Multilayer Perceptron & 8260 & 51.05\% & 58.73\% (3.65) & +7.68\% (3.65) & 1652 (0) \\
Factual & Random Forest & 8260 & 51.05\% & 62.78\% (1.07) & +11.73\% (1.07) & 1652 (0) \\
Factual & Label Propagation & 8260 & 51.05\% & 54.42\% (1.88) & +3.37\% (1.88) & 1652 (0) \\
Factual & nnPU & 8260 & 51.05\% & 60.66\% (0.60) & +9.61\% (0.60) & 1652 (0) \\
Factual & \PUAuditPlus & 8260 & 51.05\% & \hlgreen{\textbf{72.37\% (6.49)}} & \hlgreen{\textbf{+21.31\% (6.49)}} & 259 (7) \\
\bottomrule
\end{tabular}}
\label{tab:real_data_verified_20_factual}
\end{table*}

\begin{table*}[t]
\centering
\caption{Real-data results with 80\% verified data for baseline methods (Coding and MT-Bench). Values follow the same reporting and highlighting conventions as Table~\ref{tab:real_data_verified_3_coding_mtbench}.}
\resizebox{\textwidth}{!}{%
\begin{tabular}{cc|ccccc}
\toprule
\midrule
\multicolumn{7}{c}{\textbf{Real-Data Results with 80\% Verified Baseline Data (Coding and MT-Bench)}}\\
\midrule
\midrule
Question Type & Method & Data & Orig. Acc. & Adj. Acc. & Gain & Human Verif. Req. \\
\midrule
\multicolumn{7}{c}{\textbf{Model: GPT-5.4}}\\
\midrule
Coding & Logistic Regression & 2949 & 63.92\% & 62.17\% (1.86) & -1.75\% (1.86) & 2359 (0) \\
Coding & Multilayer Perceptron & 2949 & 63.92\% & 61.93\% (3.73) & -1.99\% (3.73) & 2359 (0) \\
Coding & Random Forest & 2949 & 63.92\% & 64.10\% (2.37) & +0.18\% (2.37) & 2359 (0) \\
Coding & Label Propagation & 2949 & 63.92\% & 60.07\% (3.22) & -3.85\% (3.22) & 2359 (0) \\
Coding & nnPU & 2949 & 63.92\% & 63.97\% (0.26) & +0.05\% (0.26) & 2359 (0) \\
Coding & \PUAuditPlus & 2949 & 63.92\% & \hlgreen{\textbf{67.20\% (2.00)}} & \hlgreen{\textbf{+3.28\% (2.00)}} & 100.6 (5) \\
\midrule
\multicolumn{7}{c}{\textbf{Model: GPT-5.4-mini}}\\
\midrule
Coding & Logistic Regression & 2971 & 63.24\% & 62.26\% (4.88) & -0.99\% (4.88) & 2377 (0) \\
Coding & Multilayer Perceptron & 2971 & 63.24\% & 61.48\% (3.70) & -1.76\% (3.70) & 2377 (0) \\
Coding & Random Forest & 2971 & 63.24\% & 63.20\% (3.03) & -0.05\% (3.03) & 2377 (0) \\
Coding & Label Propagation & 2971 & 63.24\% & 58.18\% (3.54) & -5.06\% (3.54) & 2377 (0) \\
Coding & nnPU & 2971 & 63.24\% & 63.77\% (0.32) & +0.53\% (0.32) & 2377 (0) \\
Coding & \PUAuditPlus & 2971 & 63.24\% & \hlgreen{\textbf{66.27\% (1.55)}} & \hlgreen{\textbf{+3.03\% (1.55)}} & 102.2 (9) \\
\midrule
MT-Bench & Logistic Regression & 2565 & 77.82\% & 78.67\% (3.90) & +0.86\% (3.90) & 2052 (0) \\
MT-Bench & Multilayer Perceptron & 2565 & 77.82\% & 78.91\% (1.36) & +1.09\% (1.36) & 2052 (0) \\
MT-Bench & Random Forest & 2565 & 77.82\% & 77.66\% (6.04) & -0.16\% (6.04) & 2052 (0) \\
MT-Bench & Label Propagation & 2565 & 77.82\% & 78.75\% (2.92) & +0.94\% (2.92) & 2052 (0) \\
MT-Bench & nnPU & 2565 & 77.82\% & 77.78\% (0.00) & -0.04\% (0.00) & 2052 (0) \\
MT-Bench & \PUAuditPlus & 2575 & 77.82\% & \hlgreen{\textbf{81.01\% (0.94)}} & \hlgreen{\textbf{+3.19\% (0.94)}} & 87.4 (7) \\
\midrule
\multicolumn{7}{c}{\textbf{Model: Gemini}}\\
\midrule
Coding & Logistic Regression & 2820 & 63.55\% & 61.67\% (2.48) & -1.88\% (2.48) & 2256 (0) \\
Coding & Multilayer Perceptron & 2820 & 63.55\% & 62.52\% (4.26) & -1.03\% (4.26) & 2256 (0) \\
Coding & Random Forest & 2820 & 63.55\% & 63.76\% (3.19) & +0.21\% (3.19) & 2256 (0) \\
Coding & Label Propagation & 2820 & 63.55\% & 58.72\% (4.79) & -4.82\% (4.79) & 2256 (0) \\
Coding & nnPU & 2820 & 63.55\% & 63.33\% (0.23) & -0.22\% (0.23) & 2256 (0) \\
Coding & \PUAuditPlus & 2820 & 63.55\% & \hlgreen{\textbf{67.17\% (2.84)}} & \hlgreen{\textbf{+3.62\% (2.84)}} & 96.6 (8) \\
\midrule
\multicolumn{7}{c}{\textbf{Model: Qwen}}\\
\midrule
Coding & Logistic Regression & 8884 & 53.73\% & 60.34\% (1.80) & +6.61\% (1.80) & 7107 (0) \\
Coding & Multilayer Perceptron & 8884 & 53.73\% & 58.74\% (4.05) & +5.01\% (4.05) & 7107 (0) \\
Coding & Random Forest & 8884 & 53.73\% & 60.30\% (1.63) & +6.58\% (1.63) & 7107 (0) \\
Coding & Label Propagation & 8884 & 53.73\% & 54.58\% (1.91) & +0.85\% (1.91) & 7107 (0) \\
Coding & nnPU & 8884 & 53.73\% & 58.91\% (0.50) & +5.18\% (0.50) & 7107 (0) \\
Coding & \PUAuditPlus & 8884 & 53.73\% & \hlgreen{\textbf{77.74\% (9.78)}} & \hlgreen{\textbf{+24.02\% (9.78)}} & 278.8 (14) \\
\midrule
\multicolumn{7}{c}{\textbf{Model: Mistral}}\\
\midrule
Coding & Logistic Regression & 8240 & 50.29\% & 59.14\% (1.94) & +8.85\% (1.94) & 6592 (0) \\
Coding & Multilayer Perceptron & 8240 & 50.29\% & 58.40\% (2.06) & +8.11\% (2.06) & 6592 (0) \\
Coding & Random Forest & 8240 & 50.29\% & 60.15\% (1.40) & +9.85\% (1.40) & 6592 (0) \\
Coding & Label Propagation & 8240 & 50.29\% & 56.36\% (3.16) & +6.07\% (3.16) & 6592 (0) \\
Coding & nnPU & 8240 & 50.29\% & 59.11\% (0.99) & +8.82\% (0.99) & 6592 (0) \\
Coding & \PUAuditPlus & 8240 & 50.29\% & \hlgreen{\textbf{72.91\% (0.61)}} & \hlgreen{\textbf{+22.62\% (0.61)}} & 261.8 (10) \\
\bottomrule
\end{tabular}}
\label{tab:real_data_verified_80_coding_mtbench}
\end{table*}

\begin{table*}[t]
\centering
\caption{Real-data results with 80\% verified data for baseline methods (Math). Values follow the same reporting and highlighting conventions as Table~\ref{tab:real_data_verified_3_coding_mtbench}.}
\resizebox{\textwidth}{!}{%
\begin{tabular}{cc|ccccc}
\toprule
\midrule
\multicolumn{7}{c}{\textbf{Real-Data Results with 80\% Verified Baseline Data (Math)}}\\
\midrule
\midrule
Question Type & Method & Data & Orig. Acc. & Adj. Acc. & Gain & Human Verif. Req. \\
\midrule
\multicolumn{7}{c}{\textbf{Model: GPT-5.4}}\\
\midrule
Math & Logistic Regression & 2817 & 67.20\% & 67.16\% (1.06) & -0.04\% (1.06) & 2253 (0) \\
Math & Multilayer Perceptron & 2817 & 67.20\% & 66.06\% (3.72) & -1.14\% (3.72) & 2253 (0) \\
Math & Random Forest & 2817 & 67.20\% & 67.84\% (1.95) & +0.64\% (1.95) & 2253 (0) \\
Math & Label Propagation & 2817 & 67.20\% & 65.64\% (1.42) & -1.56\% (1.42) & 2253 (0) \\
Math & nnPU & 2817 & 67.20\% & 67.45\% (0.32) & +0.25\% (0.32) & 2253 (0) \\
Math & \PUAuditPlus & 2817 & 67.20\% & \hlgreen{\textbf{70.05\% (1.17)}} & \hlgreen{\textbf{+2.85\% (1.17)}} & 95.8 (5) \\
\midrule
\multicolumn{7}{c}{\textbf{Model: GPT-5.4-mini}}\\
\midrule
Math & Logistic Regression & 2911 & 62.35\% & 66.01\% (2.23) & +3.66\% (2.23) & 2329 (0) \\
Math & Multilayer Perceptron & 2911 & 62.35\% & 65.15\% (3.61) & +2.80\% (3.61) & 2329 (0) \\
Math & Random Forest & 2911 & 62.35\% & \hlgreen{66.84\% (3.78)} & \hlgreen{+4.49\% (3.78)} & 2329 (0) \\
Math & Label Propagation & 2911 & 62.35\% & 61.79\% (3.95) & -0.56\% (3.95) & 2329 (0) \\
Math & nnPU & 2911 & 62.35\% & 63.85\% (1.46) & +1.50\% (1.46) & 2329 (0) \\
Math & \PUAuditPlus & 2911 & 62.35\% & \textbf{65.62\% (1.79)} & \textbf{+3.27\% (1.79)} & 99 (5) \\
\midrule
\multicolumn{7}{c}{\textbf{Model: Gemini}}\\
\midrule
Math & Logistic Regression & 2770 & 66.35\% & 67.58\% (3.07) & +1.23\% (3.07) & 2216 (0) \\
Math & Multilayer Perceptron & 2770 & 66.35\% & 67.73\% (2.89) & +1.37\% (2.89) & 2216 (0) \\
Math & Random Forest & 2770 & 66.35\% & 68.19\% (1.81) & +1.84\% (1.81) & 2216 (0) \\
Math & Label Propagation & 2770 & 66.35\% & 66.03\% (1.62) & -0.32\% (1.62) & 2216 (0) \\
Math & nnPU & 2770 & 66.35\% & 67.15\% (0.61) & +0.80\% (0.61) & 2216 (0) \\
Math & \PUAuditPlus & 2770 & 66.35\% & \hlgreen{\textbf{68.79\% (1.52)}} & \hlgreen{\textbf{+2.43\% (1.52)}} & 96.2 (4) \\
\midrule
\multicolumn{7}{c}{\textbf{Model: Qwen}}\\
\midrule
Math & Logistic Regression & 9145 & 55.10\% & 61.90\% (2.08) & +6.80\% (2.08) & 7316 (0) \\
Math & Multilayer Perceptron & 9145 & 55.10\% & 62.07\% (2.35) & +6.97\% (2.35) & 7316 (0) \\
Math & Random Forest & 9145 & 55.10\% & 63.59\% (1.97) & +8.49\% (1.97) & 7316 (0) \\
Math & Label Propagation & 9145 & 55.10\% & 57.89\% (2.95) & +2.79\% (2.95) & 7316 (0) \\
Math & nnPU & 9145 & 55.10\% & 61.12\% (1.67) & +6.02\% (1.67) & 7316 (0) \\
Math & \PUAuditPlus & 9145 & 55.10\% & \hlgreen{\textbf{79.27\% (5.47)}} & \hlgreen{\textbf{+24.17\% (5.47)}} & 291.2 (10) \\
\midrule
\multicolumn{7}{c}{\textbf{Model: Mistral}}\\
\midrule
Math & Logistic Regression & 8361 & 50.46\% & 62.91\% (2.39) & +12.45\% (2.39) & 6689 (0) \\
Math & Multilayer Perceptron & 8361 & 50.46\% & 61.77\% (4.31) & +11.31\% (4.31) & 6689 (0) \\
Math & Random Forest & 8361 & 50.46\% & 64.10\% (2.21) & +13.64\% (2.21) & 6689 (0) \\
Math & Label Propagation & 8361 & 50.46\% & 57.97\% (2.69) & +7.51\% (2.69) & 6689 (0) \\
Math & nnPU & 8361 & 50.46\% & 62.26\% (1.60) & +11.80\% (1.60) & 6689 (0) \\
Math & \PUAuditPlus & 8361 & 50.46\% & \hlgreen{\textbf{71.75\% (6.54)}} & \hlgreen{\textbf{+21.29\% (6.54)}} & 266 (23) \\
\bottomrule
\end{tabular}}
\label{tab:real_data_verified_80_math}
\end{table*}

\begin{table*}[t]
\centering
\caption{Real-data results with 80\% verified data for baseline methods (Factual). Values follow the same reporting and highlighting conventions as Table~\ref{tab:real_data_verified_3_coding_mtbench}.}
\resizebox{\textwidth}{!}{%
\begin{tabular}{cc|ccccc}
\toprule
\midrule
\multicolumn{7}{c}{\textbf{Real-Data Results with 80\% Verified Baseline Data (Factual)}}\\
\midrule
\midrule
Question Type & Method & Data & Orig. Acc. & Adj. Acc. & Gain & Human Verif. Req. \\
\midrule
\multicolumn{7}{c}{\textbf{Model: GPT-5.4}}\\
\midrule
Factual & Logistic Regression & 2842 & 64.39\% & 65.53\% (2.11) & +1.14\% (2.11) & 2274 (0) \\
Factual & Multilayer Perceptron & 2842 & 64.39\% & 65.07\% (5.63) & +0.68\% (5.63) & 2274 (0) \\
Factual & Random Forest & 2842 & 64.39\% & 65.67\% (2.99) & +1.28\% (2.99) & 2274 (0) \\
Factual & Label Propagation & 2842 & 64.39\% & 62.04\% (3.35) & -2.35\% (3.35) & 2274 (0) \\
Factual & nnPU & 2842 & 64.39\% & 65.11\% (0.79) & +0.72\% (0.79) & 2274 (0) \\
Factual & \PUAuditPlus & 2842 & 64.39\% & \hlgreen{\textbf{67.66\% (2.22)}} & \hlgreen{\textbf{+3.27\% (2.22)}} & 95.6 (9) \\
\midrule
\multicolumn{7}{c}{\textbf{Model: GPT-5.4-mini}}\\
\midrule
Factual & Logistic Regression & 2920 & 62.47\% & 65.17\% (5.82) & +2.71\% (5.82) & 2336 (0) \\
Factual & Multilayer Perceptron & 2920 & 62.47\% & 64.55\% (3.77) & +2.09\% (3.77) & 2336 (0) \\
Factual & Random Forest & 2920 & 62.47\% & \hlgreen{66.16\% (2.40)} & \hlgreen{+3.70\% (2.40)} & 2336 (0) \\
Factual & Label Propagation & 2920 & 62.47\% & 61.37\% (2.91) & -1.10\% (2.91) & 2336 (0) \\
Factual & nnPU & 2920 & 62.47\% & 64.42\% (1.24) & +1.95\% (1.24) & 2336 (0) \\
Factual & \PUAuditPlus & 2920 & 62.47\% & \textbf{66.07\% (2.09)} & \textbf{+3.60\% (2.09)} & 99 (17) \\
\midrule
\multicolumn{7}{c}{\textbf{Model: Gemini}}\\
\midrule
Factual & Logistic Regression & 2768 & 66.37\% & 65.28\% (5.42) & -1.09\% (5.42) & 2215 (0) \\
Factual & Multilayer Perceptron & 2768 & 66.37\% & 65.03\% (4.16) & -1.34\% (4.16) & 2215 (0) \\
Factual & Random Forest & 2768 & 66.37\% & 66.26\% (2.53) & -0.11\% (2.53) & 2215 (0) \\
Factual & Label Propagation & 2768 & 66.37\% & 63.94\% (1.27) & -2.42\% (1.27) & 2215 (0) \\
Factual & nnPU & 2768 & 66.37\% & 66.37\% (0.00) & -0.00\% (0.00) & 2215 (0) \\
Factual & \PUAuditPlus & 2768 & 66.37\% & \hlgreen{\textbf{69.21\% (1.73)}} & \hlgreen{\textbf{+2.85\% (1.73)}} & 94 (12) \\
\midrule
\multicolumn{7}{c}{\textbf{Model: Qwen}}\\
\midrule
Factual & Logistic Regression & 9333 & 55.94\% & 64.18\% (1.77) & +8.24\% (1.77) & 7467 (0) \\
Factual & Multilayer Perceptron & 9333 & 55.94\% & 63.12\% (3.11) & +7.18\% (3.11) & 7467 (0) \\
Factual & Random Forest & 9333 & 55.94\% & 64.24\% (2.47) & +8.30\% (2.47) & 7467 (0) \\
Factual & Label Propagation & 9333 & 55.94\% & 59.37\% (2.04) & +3.43\% (2.04) & 7467 (0) \\
Factual & nnPU & 9333 & 55.94\% & 62.60\% (0.94) & +6.66\% (0.94) & 7467 (0) \\
Factual & \PUAuditPlus & 9333 & 55.94\% & \hlgreen{\textbf{80.31\% (4.76)}} & \hlgreen{\textbf{+24.37\% (4.76)}} & 292.8 (11) \\
\midrule
\multicolumn{7}{c}{\textbf{Model: Mistral}}\\
\midrule
Factual & Logistic Regression & 8260 & 51.05\% & 62.60\% (3.93) & +11.55\% (3.93) & 6608 (0) \\
Factual & Multilayer Perceptron & 8260 & 51.05\% & 61.74\% (4.00) & +10.69\% (4.00) & 6608 (0) \\
Factual & Random Forest & 8260 & 51.05\% & 62.65\% (2.54) & +11.60\% (2.54) & 6608 (0) \\
Factual & Label Propagation & 8260 & 51.05\% & 57.99\% (2.24) & +6.94\% (2.24) & 6608 (0) \\
Factual & nnPU & 8260 & 51.05\% & 60.65\% (0.84) & +9.60\% (0.84) & 6608 (0) \\
Factual & \PUAuditPlus & 8260 & 51.05\% & \hlgreen{\textbf{72.37\% (6.49)}} & \hlgreen{\textbf{+21.31\% (6.49)}} & 259 (7) \\
\bottomrule
\end{tabular}}
\label{tab:real_data_verified_80_factual}
\end{table*}

\clearpage
\subsection{Runtime and memory usage}
\label{app:runtime_memory}

\begin{table*}[t]
\centering
\caption{Runtime and memory usage for the real-data experiments. Time and memory are reported as mean values, with ranges in parentheses.}
\resizebox{\textwidth}{!}{%
\begin{tabular}{cc|ccc}
\toprule
\midrule
\multicolumn{5}{c}{\textbf{Runtime and Memory Usage for Real-Data Experiments}}\\
\midrule
\midrule
\textbf{Question Type}
& \textbf{Model}
& \textbf{\# Data}
& \textbf{Time (min)}
& \textbf{Memory (GB)} \\
\midrule
Coding & GPT-5.4       & 2949 & 2.38 (1.75--4.12) & 1.63 (1.62--1.65) \\
Coding & GPT-5.4-mini  & 2971 & 1.89 (1.77--2.03) & 1.63 (1.62--1.64) \\
Coding & Gemini        & 2820 & 2.28 (1.73--4.33) & 1.65 (1.62--1.74) \\
Coding & Qwen          & 8884 & 72.99 (61.55--81.50) & 3.50 (3.24--3.58) \\
Coding & Mistral       & 8240 & 57.16 (48.38--72.73) & 3.07 (2.89--3.29) \\
\midrule
Math & GPT-5.4       & 2817 & 1.75 (1.60--2.12) & 1.63 (1.62--1.64) \\
Math & GPT-5.4-mini  & 2911 & 1.75 (1.62--1.93) & 1.63 (1.62--1.64) \\
Math & Gemini        & 2770 & 1.69 (1.53--1.77) & 1.63 (1.62--1.64) \\
Math & Qwen          & 9145 & 83.56 (69.88--94.73) & 4.35 (4.12--4.47) \\
Math & Mistral       & 8361 & 55.92 (46.23--67.93) & 3.00 (2.87--3.38) \\
\midrule
Factual & GPT-5.4       & 2842 & 1.84 (1.68--2.10) & 1.65 (1.64--1.66) \\
Factual & GPT-5.4-mini  & 2920 & 1.81 (1.55--2.05) & 1.62 (1.61--1.63) \\
Factual & Gemini        & 2768 & 1.72 (1.60--1.80) & 1.64 (1.63--1.65) \\
Factual & Qwen          & 9333 & 96.24 (87.27--101.28) & 4.63 (4.52--4.98) \\
Factual & Mistral       & 8260 & 53.31 (43.03--65.42) & 2.95 (2.83--3.38) \\
\midrule
MT-Bench & GPT-5.4-mini & 2575 & 1.58 (1.47--1.63) & 1.65 (1.64--1.66) \\
\bottomrule
\end{tabular}}
\label{tab:real_data_runtime_memory}
\end{table*}


\end{document}